\documentclass[11pt]{article}

\usepackage{natbib}
\usepackage{amsmath}
\usepackage{amssymb}
\usepackage{amsthm}
\usepackage{url}

\usepackage{hyperref}

\hypersetup{colorlinks,
            linkcolor=blue,
            citecolor=blue,
            urlcolor=magenta,
            linktocpage,
            plainpages=false}
\usepackage{algorithm}
\usepackage{algpseudocode}
\usepackage{mathtools}
\usepackage{ifthen}
\usepackage{xcolor}
\usepackage{booktabs}
\usepackage{makecell}
\usepackage{pifont}
\usepackage{fullpage}

\newenvironment{keywords}
{\bgroup\leftskip 27pt\rightskip 27pt \small\noindent{\bfseries
Keywords:} \ignorespaces}%
{\par\egroup\vskip 0.25ex}

\reversemarginpar

\newcommand{\bs}[1]{\boldsymbol{\mathbf{#1}}}
\newcommand{\hbs}[1]{\hat{\boldsymbol{\mathbf{#1}}}}
\newcommand{\tbs}[1]{\tilde{\boldsymbol{\mathbf{#1}}}}

\newcommand{\bbC}{\mathbb{C}}

\newcommand{\bbE}{\mathbb{E}}

\newcommand{\bbI}{\mathbb{I}}

\newcommand{\bbN}{\mathbb{N}}

\newcommand{\bbR}{\mathbb{R}}

\newcommand{\cB}{\mathcal{B}}

\newcommand{\cD}{\mathcal{D}}

\newcommand{\cF}{\mathcal{F}}

\newcommand{\cM}{\mathcal{M}}
\newcommand{\cN}{\mathcal{N}}
\newcommand{\cO}{\mathcal{O}}
\newcommand{\cP}{\mathcal{P}}

\newcommand{\cR}{\mathcal{R}}

\newcommand{\rmmin}{\mathrm{min}}
\newcommand{\rmmax}{\mathrm{max}}

\newcommand{\rmbias}{\mathrm{bias}}
\newcommand{\rmvar}{\mathrm{var}}
\newcommand{\rmeff}{\mathrm{eff}}

\newcommand{\defeq}{\stackrel{\smash{\mathrm{\scriptscriptstyle def}}}{=}}
\newcommand{\phrel}[1][=]{\mathrel{\phantom{#1}}}
\renewcommand{\leq}{\leqslant}
\renewcommand{\geq}{\geqslant}

\DeclareMathOperator{\tr}{tr}

\DeclareMathOperator*{\argmin}{arg\,min}

\newcommand{\subtitle}[2][]{%
  \par
  \ifthenelse{\equal{#1}{}}%
    {\noindent\emph{#2}}%
    {\noindent\textbf{#1} \emph{#2}}%
  \par
}

\theoremstyle{definition}
\newtheorem{definition}{Definition}

\theoremstyle{plain}
\newtheorem{theorem}{Theorem}

\theoremstyle{plain}
\newtheorem{assumption}{Assumption}

\theoremstyle{plain}
\newtheorem{lemma}{Lemma}

\theoremstyle{plain}
\newtheorem{corollary}{Corollary}

\theoremstyle{remark}
\newtheorem{remark}{Remark}

\theoremstyle{plain}

\title{Accelerating Single-Pass SGD for Generalized Linear Prediction}

\author{Qian Chen$^{\dagger}$ \quad Shihong Ding$^{\dagger}$ \quad Cong Fang$^{\dagger}$ \\
\\
        \small $^{\dagger}$Peking University
}
\date{}

\begin{document}

\maketitle

\begin{abstract}

We study generalized linear prediction under a streaming setting, where each iteration uses only one fresh data point for a gradient-level update. While momentum is well-established in deterministic optimization, a fundamental open question is whether it can \emph{accelerate} such single-pass non-quadratic stochastic optimization. We propose the \emph{first} algorithm that successfully incorporates momentum via a novel data-dependent proximal method, achieving dual-momentum acceleration. Our derived excess risk bound decomposes into three components: an \emph{improved} optimization error, a minimax \emph{optimal} statistical error, and a \emph{higher-order} model-misspecification error. 
The proof handles mis-specification via a fine-grained stationary analysis of inner updates, while localizing statistical error through a two-phase outer-loop analysis.
As a result, we resolve the open problem posed by \citet{jain2018accelerating} and demonstrate that momentum acceleration is more effective than variance reduction for generalized linear prediction in the streaming setting.
\end{abstract}


\section{Introduction}
This paper considers Generalized Linear Prediction in the form of :
\begin{equation}
    \min_{\bs x\in\bbR^d} F(\bs x)= \bbE_{(\bs a,b)\sim\cD}~\ell(\bs a^\top\bs x,b),\tag{GLP}\label{eq:glp}
\end{equation}
where the objective $F$ minimizes the expected convex loss $\ell$ over linearly projected data $\bs a$ with $(\bs a, b)$ drawn from an underlying distirbution $\cD$. 
This problem is ubiquitous in machine learning and finds broad application across domains. Formally, it encapsulates maximum likelihood estimation (MLE) for generalized linear models (GLMs), where specific choices of the loss $\ell$ recover fundamental learning problems including linear and logistic regression. We study this problem in a large-scale streaming setting, where an algorithm is permitted only a gradient-level $\cO(d)$ computation from a \emph{fresh} data point per-itertion. The central interest is to improve the data/runtime complexity required to approximately solve \eqref{eq:glp}. We primarily focus on loss functions $\ell$ that are strongly convex and gradient Lipschitz continuous for the first argument, and we denote  $\alpha$ as the condition number of the loss (see Section~\ref{sec:assumptions} for the definition). 

A direct comparison baseline of our constrained streaming algorithms is those free of computational and memory constraints.  Among them, a commonly used approach is to compute the Empirical Risk Minimization of the form:
\begin{equation}
   \hat{\bs x} \in  \argmin_{\bs x} \frac{1}{N}\sum_{i=1}^N \ell(\bs a_i^\top\bs x,b).\label{eq:glperm}
\end{equation}
For \eqref{eq:glperm}, a non-asymptotic sample complexity of $\frac{\alpha \sigma_0^2}{\varepsilon}$ for the excess risk $F(\hat{\bs x})-F(\bs x^*)$ can be established via the celebrated localization Rademacher technique~\citep{bartlett2005local,wainwright2019high}, where $\bs x^*$ denotes the population minimizer of \eqref{eq:glp}, $\sigma_0^2$ relates to the Lipschitzness of $\ell$. The sample complexity is at least $\frac{\alpha\sigma_*^2}{\varepsilon}$, where $\sigma_*^2=\tr(\bs H^{-1}\bs Q)$, $\bs Q$ is the stochastic gradient covariance at $\bs x^*$, and $\bs H$ is an upper bound of $\nabla^2 F$ (see Section~\ref{sec:quantity} for definition). We show that this sample complexity is optimal when only strong convexity and gradient Lipschitz continuity are imposed on $\ell$ in Appendix~\ref{sec:statistical-optimality}.

Turning to the online setting, a natural question arises: how much additional data is required under our streaming constraint? It should be noted that standard Stochastic Gradient Descent (SGD) is inefficient for solving \eqref{eq:glp}, as it fails to exploit the problem structure. For linear regression (where $\alpha = 1$), it is known that certain variants of SGD can indeed improve the sample complexity. For example, under a well-specified model, a sample complexity of $\kappa \log(1/\varepsilon) + \frac{\sigma_*^2}{\varepsilon}$ for SGD with tail-averaging~\citep{jain2018parallelizing,zou2023benign} and exponential step size decay~\citep{ge2019step,wu2022last}, respectively. Interestingly, the complexity decomposes into the optimal statistical term and an optimization term that depends only logarithmically on the target accuracy, scaled by $\kappa$. Here, $\kappa$ denotes the condition number of the data distribution (see Section~\ref{sec:assumptions} for details). It can be bounded as $\kappa \leq L/\mu$, provided that the data satisfy $\|\bs a\|^2\leq L$ and $\mu$ is the minimum eigenvalue of the data covariance matrix $\bs\Sigma\defeq \bbE_{\bs a\sim\cD} [\bs a \bs a^\top]$.

In recent years, several works~\citep{frostig2015competing,li2022root} have attempted to solve \eqref{eq:glp} in the general setting, obtaining complexity bounds that decompose into a statistical term and an optimization term imposed by the streaming constraint. 
The core idea of these algorithms is to use \emph{variance-reduction} techniques that separate the gradient variance into the variance at the optimum and an optimization-controllable error component, thereby localizing the statistical error. However, a common limitation of these results is  the optimization complexity dependence  on the problem condition number $\alpha^2\kappa$. The optimization complexity characterizes the time required to converge from an initial error. Specifically, the excess risk does not reduce to a non-trivial level (i.e., becomes $o(1)$) unless the sample size exceeds $\alpha^2\kappa$.

In deterministic optimization, momentum techniques such as Nesterov's acceleration~\citep{nesterov1983method} and the heavy-ball method~\citep{polyak1964some} are well-known to accelerate convergence across different settings. Yet, how to effectively incorporate momentum into our streaming optimization problem remains largely open.

For general stochastic convex problems, there is a prevailing view that momentum may not offer efficient acceleration. For instance, \citet{agarwal2012informationtheoretic} establishes the optimality of SGD under general strongly convex objectives with gradient variance. In contrast, for well-specified linear regression—as a structured estimation problem—\citet{jain2018accelerating} shows that the optimization term can indeed be accelerated while preserving the optimal statistical complexity. Specifically, \citet{jain2018accelerating} introduces a statistical condition number of the data, denoted $\tilde{\kappa}$ (detailed in Section~\ref{sec:assumptions}), which is strictly smaller than $\kappa$, and improves the optimization complexity to $\sqrt{\kappa\tilde{\kappa}}$. Their work leaves two open problems: (i) extending the analysis to allow for model misspecification, and (ii) exploring more general estimation settings where momentum remains efficient. This paper addresses both challenges by proposing the \emph{first} algorithm that successfully incorporates momentum for \emph{generalized linear prediction} without relying on fixed Hessian structure and model specification.

\subsection{Review: Previous Results}
\paragraph{Well-specified Linear Regression.}
The central idea enabling SGD-type algorithms to attain optimal statistical complexity—along with an additional optimization term—is to model stochastic gradient descent as a stochastic process and to establish convergence of the variable in distribution~\citep{dieuleveut2016nonparametric,jain2018parallelizing,ge2019step,zou2023benign,wu2022last,li2024risk}. This convergence analysis naturally yields the optimization complexity, while a variance computation for the stationary distribution captures the statistical error.
\citet{jain2018accelerating} improves the optimization term  and attains the complexity of
\begin{equation*}
    \sqrt{\kappa\tilde{\kappa}}+\frac{\sigma_*^2}{\varepsilon},
\end{equation*}
ignoring logarithmic factors, by employing momentum acceleration with a finer-grained analysis. This result, however, is restricted to well-specified linear regression, since the stationary-distribution analysis critically relies on the quadratic form of the objective and the commutativity of the gradient noise covariance and the problem Hessian.

\paragraph{Variance Reduction for Generalized Linear Prediction.} \citet{frostig2015competing} and \citet{li2022root} study a slightly more general stochastic convex optimization problem that does not assume a linear structure in the individual functions. They require the objective $F$ to be $\mu_h$-strongly convex and each component to be $L_h$-gradient Lipschitz continuous, and further crucially impose a Hessian dominance condition $\nabla^2 F(\bs x^*)\preceq\alpha_h\nabla^2 F(\bs x)$ tailored to generalized linear prediction. Adapting their results to our setting, \citet{frostig2015competing} attains a complexity\footnote{\citet{frostig2015competing} obtain a sample complexity of $\alpha_h\kappa_h+\frac{\alpha_h\tr((\nabla F(\bs x^*))^{-1}\bs Q)}{\varepsilon}$ in Corollary~4, where $\kappa_h=L_h/\mu_h$. In the worst case, one has $\alpha_h=\alpha$, $\kappa_h=\alpha\kappa$ and $\tr((\nabla F(\bs x^*))^{-1}\bs Q)=\alpha\sigma_*^2$, which is the complexity in \eqref{eq:svrg-complexity}.} of
\begin{equation}    
    \alpha^2\kappa+\frac{\alpha^2\sigma_*^2}{\varepsilon},\label{eq:svrg-complexity}
\end{equation}
where the term $\alpha^2\kappa$ corresponds to the optimization cost and $\frac{\alpha^2\sigma_*^2}{\varepsilon}$ is the statistical sample complexity. 
\citet{li2022root} obtains the sample complexity\footnote{\citet{li2022root} obtain a sample complexity of $\kappa_h+\frac{\ell_\Xi^2}{\mu_h^2}+\frac{\tr((\nabla F(\bs x^*))^{-1}\bs Q)}{\varepsilon}+\cO(\varepsilon^{-2/3})$ in Corollary~3. In the worst case, one has $\kappa_h=\alpha\kappa$, $\ell_\Xi^2/\mu_h^2=\alpha^2\kappa\cdot\frac{\lambda_\rmmax(\bs\Sigma)}{\lambda_\rmmin(\bs\Sigma)}$ and $\tr((\nabla F(\bs x^*))^{-1}\bs Q)=\alpha\sigma_*^2$, which is the complexity in \eqref{eq:root-sgd-complexity}.} of
\begin{equation}
    \alpha^2\kappa\cdot\frac{\lambda_\rmmax(\bs\Sigma)}{\lambda_\rmmin(\bs\Sigma)}+ \frac{\alpha\sigma_*^2}{\varepsilon}+\cO\left(\frac{1}{\varepsilon^{2/3}}\right).\label{eq:root-sgd-complexity}
\end{equation}
This work requires an additional Hessian smoothness assumption to further reduce the statistical complexity. Moreover, the last term above hides problem-dependent factors, including the Hessian smoothness parameter.  It remains open whether variance reduction is fundamentally necessary for \eqref{eq:glp} and how to improve the dependence on the optimization complexity.

\subsection{Our Results and Implications}
This paper proposes momentum acceleration for \eqref{eq:glp} in the streaming setting.  Our algorithm's main idea lies in a data-dependent proximal scheme, with the proximal operator explicitly constructed from the expected data covariance $\bs \Sigma$.  This makes double momentum acceleration applicable to both the inner and outer loops, and we obtain a complexity of
\begin{equation*}
    \left(\sqrt{\alpha\kappa\tilde{\kappa}}+\alpha^2\tilde{\kappa}\right)+\frac{\alpha\sigma_*^2}{\varepsilon}+\left(\frac{\alpha^2\tilde{\kappa}^2\tr\bs Q}{L_\ell\mu\varepsilon}\right)^{1/3}.
\end{equation*}

The above excess risk bound is composed of three interpretable components: 
(i) an optimization term with \emph{improved} dependence on the problem and data condition numbers, where the first term from the doubly accelerated scheme, and the second term is the effect of noise; (ii) the estimation term matching the \emph{optimal} statistical risk, and (iii) a \emph{higher-order} mis-specification error term capturing the coupling effect of constrained computation with model mis-specification, which we provide refined characterization. Our framework can be extended to weakly convex objectives and naturally accommodates the use of unlabeled data, mini-batching, and parallel implementations.

The analysis of our algorithm is composed of two parts: (i) In the inner loop, the sub-problem resembles linear regression but with model mis-specification. We provide a fine-grained characterization of the effect of model mis-specification; (ii) The outer-loop requires us to localize the statistical error to the minimum points. We propose a two-phase analysis showing that variance reduction is not needed in our problem.

As a result, we resolve the open problems posed by \citet{jain2018accelerating}. We also demonstrate that momentum acceleration is more efficient than variance reduction for streaming generalized linear prediction. This finding is fundamentally different from known results for non-convex optimization in the streaming setting. There, the goal is to find an approximate stationary point, and
momentum fails to improve the worst-case rate over plain SGD even with Hessian smoothness assumptions \citep{fang2019sharp,jin2018accelerated}. However, variance-reduction methods such as SVRG improve $\varepsilon^{-3.5}$ to $\varepsilon^{-10/3}$~\citep{allen2016variance,reddi2016stochastic}, with SPIDER~\citep{fang2018spider} and SARAH~\citep{nguyen2017sarah} attaining the optimal $\varepsilon^{-3}$. We hope our work offers insights toward designing accelerated streaming methods for broader classes of convex and non-convex problems.

\paragraph{Notations.}
For a vector $\bs x$, let $\|\bs x\|$ denote $2$-norm, and for a symmetric positive semidefinite matrix $\bs A$, $\|\bs x\|_{\bs A}=\sqrt{\bs x^\top\bs A\bs x}$ denotes the induced norm. Let $\|\bs A\|$ denote the spectral norm of $\bs A$. For positive semi-definite matrix $\bs\Sigma$, we denote the maximum and minimum eigenvalue of $\bs\Sigma$ by $\lambda_\rmmax(\bs\Sigma)$ and $\lambda_\rmmin(\bs\Sigma)$, respectively. Define $\kappa(\bs\Sigma)=\lambda_\rmmax(\bs\Sigma)/\lambda_\rmmin(\bs\Sigma)$ as the condition number of $\bs\Sigma$.
We denote the Gaussian distribution with mean $\bs\mu$ and covariance $\bs\Sigma$ by $\cN(\bs\mu,\bs\Sigma)$.
For two nonnegative sequences $\{a_n\}$ and $\{b_n\}$, we write $a_n = \cO(b_n)$ (resp.\ $a_n = \Omega(b_n)$) if there exists a universal constant $C>0$ such that $a_n \le C b_n$ (resp.\ $a_n \ge C b_n$) for all sufficiently large $n$, and we say $a_n = \Theta(b_n)$ if both hold. We use $\tilde{\cO}(\cdot)$, $\tilde{\Omega}(\cdot)$, and $\tilde{\Theta}(\cdot)$ to hide logarithmic factors. We write $a_n \lesssim b_n$ (resp.\ $a_n \gtrsim b_n$) if $a_n=\cO(b_n)$ (resp.\ $a_n=\Omega(b_n)$).

\section{Related Work}
\paragraph{Stochastic Approximation.}
Stochastic Approximation methods date back to the seminal work of \citet{robbins1951stochastic}. A line of work has developed non-asymptotic convergence rate for stochastic gradient descent in convex and strongly convex settings, and analyzed the effect of step size schedules, and averaging schemes~\citep{moulines2011non,bach2013non,defossez2015averaged,dieuleveut2016nonparametric,ge2019step,zhang2025learning,sheshukova2025nonasymptotic}. In the linear regression setting, refined non-asymptotic analyses show that iterate-averaged SGD attains the optimal statistical rate $\sigma^2d/n$~\citep{bach2013non,jain2018parallelizing,dieuleveut2017harder,jain2018accelerating}. However, such optimal rates do not generally extend beyond quadratic objectives. Recent works developed algorithms to achieve such an optimal statistical rate~\citep{frostig2015competing,li2022root}. Specifically, Streaming SVRG~\citep{frostig2015competing} adopts the idea of variance reduction to the streaming setting and achieves the statistical performance of ERM. ROOT-SGD~\citep{li2022root} achieves a nonasymptotic rate, and asymptotic rate matches the Cram\'er-Rao lower bound. These analyses rely on third-order continuous conditions such as self-concordance or Hessian smoothness, while we need only second-order conditions.

\paragraph{Momentum Acceleration and Variance Reduction.}
Momentum methods originate from the heavy-ball method~\citep{polyak1964some} and Nesterov’s accelerated gradient descent~\citep{nesterov1983method}, which achieve accelerated and optimal convergence rates. In the streaming setting, the effect of momentum is more subtle, as it does not improve the worst-case statistical rate compared to SGD in general~\citep{agarwal2012informationtheoretic}. However, in the linear regression, several works establish accelerated convergence of the optimization error and optimal statistical error by exploiting the quadratic structure~\citep{jain2018accelerating,dieuleveut2017harder,pan2024accelerated,li2024risk,zhang2024optimality,liu2025optimal}. However, their analyses rely heavily on matrix calculations, which are non-trivial to extend to general objectives. For non-convex problems, \citet{fang2019sharp} proved the complexity for SGD to find a second-order stationary point is $\cO(\varepsilon^{-3.5})$, matches that of momentum-based methods~\citep{jin2018accelerated}. Variance reduction techniques accelerate stochastic optimization by exploiting problem structure to reduce stochastic noise. For convex problem, in the finite-sum setting, SDCA~\citep{shalev2013stochastic} considers the same problem as \eqref{eq:glp}. A line of work developed various algorithms with momentum acceleration~\citep{shalev2014accelerated,lin2018catalyst,allen2018katyusha}. In the streaming setting, SVRG~\citep{frostig2015competing} and ROOT-SGD~\citep{li2022root} improve the statistical term but offer no acceleration to the optimization term. For non-convex problems, variance reduction techniques have provably shown benefits in both finite-sum and streaming settings. Specifically, in the streaming setting, SVRG achieves a complexity of $\varepsilon^{-10/3}$~\citep{allen2016variance,reddi2016stochastic}, while recursive variance reduction methods, such as SARAH~\citep{nguyen2017sarah}, attain the optimal rate of $\varepsilon^{-3}$.

\section{Problem Setup}\label{sec:setup}
We restate Generalized Linear Prediction as follows:

\begin{equation*}
    \min_{\bs x\in\bbR^d} F(\bs x)\defeq\bbE_{\bs a,b\sim\cD}\ \ell(\bs a^\top\bs x,b).\tag{GLP}
\end{equation*}

\subsection{Assumptions}\label{sec:assumptions}

We make the following assumptions on the objective~\eqref{eq:glp} and the data distribution $\cD$.
\begin{assumption}[Condition Number of Loss Function]\label{assumption:l-condition}
    Let $\ell''(\cdot,\cdot)$ denote the second derivative with respect to its first argument. Assume there exists $L_\ell,\mu_\ell>0$ such that $\mu_\ell\leq\ell''(\bs a^\top\bs x,b)\leq L_\ell$. Let $\alpha\defeq L_\ell/\mu_\ell$ be the condition number of the loss function $\ell$.
\end{assumption}

Assumption~\ref{assumption:l-condition} implies that $\ell$ is $L_\ell$-smooth and $\mu_\ell$-strongly convex, which is widely used in optimization literature~\citep{nesterov2018lectures}. This assumption is also commonly adopted in the analysis of generalized linear models, such through tools like local Rademacher complexity~\citep{bartlett2005local,wainwright2019high}.

\begin{assumption}[Second Moment Condition on Data]\label{assumption:regularity}
    We assume that the second-order moment matrix $\bs\Sigma\defeq\bbE\bs a\bs a^\top$ exists and is finite. Moreover, we assume that $\bs\Sigma$ is positive definite, with minimum eigenvalue $\mu\defeq\lambda_{\rmmin}(\bs\Sigma)>0$.
\end{assumption}

Assumption~\ref{assumption:regularity} is a mild regularity condition. Combined with Assumption~\ref{assumption:l-condition}, it implies that the objective $F$ is $L_\ell\lambda_\rmmax(\bs\Sigma)$-smooth and $\mu_\ell\mu$-strongly convex. We will extend
our result to the weakly convex case in Section \ref{sec:extensions}.

\begin{assumption}[Fourth Moment Condition on Data]\label{assumption:fourth-moment}
    We assume the following boundness condition of the fourth moment:
    \begin{equation*}
        \bbE_{\bs a\sim\cD}\left(\|\bs a\|^2\bs a\bs a^\top\right)\preceq R^2\bs\Sigma,\quad
        \bbE_{\bs a,b\sim\cD}\left(\|\bs a\|_{\bs\Sigma^{-1}}^2\bs a\bs a^\top\right)\preceq\tilde{\kappa}\bs\Sigma.
    \end{equation*}
\end{assumption}

Assumption~\ref{assumption:fourth-moment} is commonly used in the analysis of SGD for linear regression~\citep{jain2018accelerating,zou2023benign,wu2022last,li2024risk}. 
The quantity $R$ can be viewed as the effective radius of the data $\bs a\sim\cD$. Specifically, if $\| \bs a \| \leq L$ almost surely, then $R=L$. The quantity $\tilde{\kappa}$ can be understood similarly. In our setting, $R^2$ and $\tilde{\kappa}$ characterize the $2$-norm and $\bs\Sigma^{-1}$-norm of the gradient noise, respectively.
We refer to $\kappa=R^2/\mu$ as the condition number of $\cD$ and $\tilde{\kappa}$ as the statistical condition number following \citet{defossez2015averaged,jain2018accelerating}.
Note we always have $\tilde{\kappa}\leq\kappa $ since $\bbE(\|\bs a\|_{\bs\Sigma^{-1}}^2\bs a\bs a^\top)\preceq\frac{1}{\mu}\bbE_{\bs a\sim\cD}(\|\bs a\|^2\bs a\bs a^\top)\preceq\kappa\bs\Sigma$~\citep{jain2018accelerating}.

In the linear regression setting, \citet{jain2018accelerating} improve the optimization complexity from $\kappa$ to $\sqrt{\kappa\tilde{\kappa}}$, where $\tilde{\kappa}\leq\kappa$ holds for all distribtution. The improvement is significant when $\bs\Sigma$ is poor-conditioned, i.e. $\tilde{\kappa}\ll\kappa$. The statistical condition $\tilde{\kappa}$ is necessary since $\tilde{\kappa}$ governs the concentration of the empirical covariance matrix to $\bs\Sigma$ and in some sense, it captures the statistical complexity under the noiseless-label setting~\citep{jain2018accelerating}. We refer the reader to \citet{jain2018accelerating} for more discussion on the necessity of $\tilde{\kappa}$.

\begin{remark}[Effect of Acceleration]\label{remark:examples}
     In the Gaussian design $\bs a \sim \cN(\bs 0, \bs\Sigma)$, one may take $\kappa=3\tr\bs\Sigma/\mu$ and $\tilde{\kappa} = 3d$. So $\kappa \geq \tilde{\kappa}$. More generally, if $\bs\Sigma^{-1/2} \bs a$ is $\sigma_a^2$-sub-Gaussian, then the condition holds with $\kappa = 16 \sigma_a^4\tr\bs\Sigma/\mu$ and $\tilde{\kappa} = 16 \sigma_a^4 d$~\citep{zou2023benign}. In the worst case, $\tr\bs\Sigma/\mu =  \kappa(\bs\Sigma)\tilde{\kappa}$.

\end{remark}

\subsection{Summary of Problem-Dependent Quantities}\label{sec:quantity}
We collect the quantities and present the goal for the overall complexity. We denote $\bs H=L_\ell\bs\Sigma$, which is an upper bound of the objective Hessian $\nabla^2 F$. Let $\bs x^*$ denote the minimizer of $F$, we define the second moment matrix of the gradient noise at minimizer $\bs x^*$ as
\begin{equation*}
    \bs Q=\bbE_{\bs a,b\sim\cD}\left(\ell'(\bs a^\top\bs x^*,b)\right)^2\bs a\bs a^\top.
\end{equation*}

\begin{itemize}
    \item For optimization complexity, the goal is to \emph{improve} the dependence on the loss condition number $\alpha$ and data condition number $\kappa$, $\tilde{\kappa}$. In particular, we are interested in improving the dependence on $\kappa$.  As shown in Remark \ref{remark:examples}, it can be $\kappa(\bs \Sigma)$ times larger than $\tilde{\kappa}$.
    \item 
    For the statistical complexity, our objective is to \emph{preserve} the optimal rate $\frac{\alpha\tr(\bs H^{-1}\bs Q)}{\varepsilon}$.
 In the well-specified linear regression model $b = \bs a^\top \bs x^*+\varepsilon_\mathrm{noise}$, this expression reduces to  $\frac{\sigma^2 d}{\varepsilon}$, where $\sigma^2$
  is the variance of the noise $\varepsilon_\mathrm{noise}$.
  \item For additional complexity, it may emerge from the coupling between the optimization process and model misspecification. We expect this complexity to exhibit a lower-order dependence on $\varepsilon^{-1}$ and $\kappa$ (i.e., to appear as a higher-order term in the risk rate).

\end{itemize}

\begin{algorithm}[t]
    \caption{\textbf{S}tochastic \textbf{A}ccelerated \textbf{D}ata-Dependent \textbf{A}lgorithm (SADA)}\label{alg:sada}
    \begin{algorithmic}
    \Require {Initialization $\tbs x_0$, regularization parameters $\{h_k\}_{k=1}^K$, step sizes $\eta$, $\gamma$, and momentum parameters $\{\beta_k\}_{k=1}^K$, $\theta$, $\tbs x_{-1}=\tbs x_0$}
    \For{$k=1,2,\ldots,K$}
        \State $\tilde{\bs y}_{k-1}\gets\tbs x_{k-1}+\beta_k(\tbs x_{k-1}-\tbs x_{k-2})$ \Comment{Extraplotation}
        \State $\bs x_0\gets\tilde{\bs y}_{k-1}$, $\bs z_0\gets\tilde{\bs y}_{k-1}$
        \For{$t=1,2,\ldots,T$} \Comment{Inner loop for solving subproblem~\eqref{eq:subproblem}}
            \State Sample fresh data $(\bs a_t,b_t)\sim\cD$
            \State $\bs y_{t-1}\gets\frac{1}{1+\theta}\bs x_{t-1}+\frac{\theta}{1+\theta}\bs z_{t-1}$
            \State $\bs{\hat g}_t\gets h_k\ell'(\bs a_t^\top\tbs y_{k-1},b_t)\bs a_t+\left[\bs a_t^\top(\bs y_{t-1}-\tbs y_{k-1})\right] \bs a_t$ \Comment{Compute $\hat{\bs g}_t$ by streaming data}
            \State $\bs x_t\gets\bs y_{t-1}-\eta\hat{\bs g}_t$
            \State $\bs z_t\gets\theta\bs y_{t-1}+(1-\theta)\bs z_{t-1}-\gamma\hat{\bs g}_t$
        \EndFor
        \State $\tbs x_k\gets\frac{2}{T}\sum_{t=T/2+1}^T\bs x_t$ \Comment{Tail-averaging scheme}
    \EndFor

    \State\Return $\tbs x_K$
    
    \end{algorithmic}
\end{algorithm}

\section{Stochastic Accelerated Data-Dependent Algorithm}

We propose the \textbf{Stochastic Accelerated Data-Dependent Algorithm (SADA)} to solve \eqref{eq:glp}, as summarized in Algorithm~\ref{alg:sada}. SADA combines momentum and data-dependent proximal methods to achieve acceleration in the streaming setting. At a high level, the outer loop iteratively constructs data-dependent proximal subproblems. The proximal term is induced by the data covariance and cannot be accessed explicitly. The inner loop approximately solves the subproblem using streaming data to approximate the proximal term, and returns the solution obtained by by tail-averaging the last half iterations.

\subsection{Inner Loop: Accelerated Solver with Tail-Averaging}
The inner loop adopts momentum to accelerate the convergence of optimization error and a tail-averaging scheme to reduce variance. At the $k$-th outer iteration, let $\tilde{\bs y}_{k-1}$ denote the extrapolated variable. The inner loop solves the following data-dependent proximal subproblem:
\begin{equation}
    \min_{\bs x\in\bbR^d}\;\bbE_{\bs a,b\sim \cD} h_k\langle\ell'(\bs a^\top\tbs y_{k-1},b)\bs a,\bs x-\tbs y_{k-1}\rangle+\frac{1}{2}\|\bs x-\tbs y_{k-1}\|_{\bs \Sigma}^2,\label{eq:subproblem}
\end{equation}
where $\ell'(\cdot,\cdot)$ denotes the first-order partial derivative of $\ell$ with respect to its first argument.

Since population covariance $\bs\Sigma$ is not directly accessible, the inner loop uses fresh samples $(\bs a,b)\sim\cD$ to approximate population covariance $\bs\Sigma$ by $\bs a\bs a^\top$. This resembles linear regression with data pair $(\bs a, -h_k\ell'(\bs a^\top\tbs y_{k-1},b))$. However, model mis-specification will occur, and its analysis is non-trivial. Previous studies have only focused on the well-specified setting. To address this gap, we propose a Layer-Peeled Decomposition method to study its stationary distribution. The details are presented in Section~\ref{sec:sketch-inner-loop}.

\subsection{Outer Loop: Data-Dependent Proximal Method with Acceleration}

The outer loop iteratively constructs the proximal subproblem~\eqref{eq:subproblem} and employs momentum to accelerate convergence. Since acceleration often amplifies the error, we need to carefully control the inner loop optimization accuracy.
SADA uses a two-phase step size to achieve an accelerated optimization convergence rate and control the stochastic noise. The first phase uses a large constant step size to rapidly reduce the optimization error, and the second phase uses a gradually decaying step size to control the stochastic noise.

\section{Convergence Result and Algorithmic Complexity}
In this section, we present the convergence rate and complexity analysis of SADA.

\paragraph{Hyperparameter Choice.} We choose inner loop hyperparameters as follows:
\begin{equation}
    \eta\leq\frac{1}{16 R^2},\quad\gamma=\frac{1}{4}\sqrt{\frac{\eta}{\tilde{\kappa}\mu}},\quad\theta=\frac{1}{4}\sqrt{\frac{\mu\eta}{\tilde{\kappa}}},\quad T\geq\tilde{\Omega}\left(\sqrt{\frac{\tilde{\kappa}}{\mu\eta}}\right).\label{eq:param-choice-inner}
\end{equation}
For the hyperparameters of the outer loop, we set step size $h_k=2\alpha\tilde{\theta}_k^2/L_\ell$, and momentum parameters $\beta_k=(1-\tilde{\theta}_k)/(1+\tilde{\theta}_k)$, where
\begin{equation}
    \tilde{\theta}_\rmmax\asymp\min\left\{\frac{1}{\sqrt{\alpha}},\frac{T}{\alpha^2\tilde{\kappa}}\right\},\quad\tilde{\theta}_k=\begin{cases}
        \tilde{\theta}_\rmmax,&k\leq K/2,\\
        \dfrac{4}{4/\tilde{\theta}_\rmmax+k-K/2},&k>K/2.
    \end{cases}\label{eq:param-choice-outer}
\end{equation}
\begin{remark}
    
    In the above setting, the regularization parameter $h_k$ can be viewed as the step size of the outer loop; we will refer to $h_k$ as the step size in the following. $\eta$ is required to be smaller than $1/R^2$ to guarantee convergence, and $\gamma$ and $\theta$ are acceleration hyperparameters, which are proposed in stochastic coordinate descent~\citep{nesterov2012efficiency} and linear regression \citep{jain2018accelerating,li2024risk}. We require the number of inner loop iterations to exceed $T\geq\tilde{\Omega}(\sqrt{\tilde{\kappa}/(\mu\eta)})$ to guarantee the inner loop is sufficiently converged. In the outer loop, by substituting $\tilde{\theta}_k$ in \eqref{eq:param-choice-outer} into $h_k=2\alpha\tilde{\theta}_k^2/L_\ell$, one has the step size $h_k$ is bounded by the inverse smoothness parameter $1/L_\ell$ and an additional term $\frac{T^2}{L_\ell\alpha^3\tilde{\kappa}^2}$ imposed by additional noise. In the first phase ($k\leq K/2$), we choose the maximum step size to accelerate the convergence in distribution; in the second phase ($t>K/2$), we carefully decay the step size to control noise.
\end{remark}

\begin{remark}
    For the Gaussian and sub-Gaussian setting in Remark~\ref{remark:examples}, the inner loop hyperparameters are $\eta\lesssim 1/\tr\bs\Sigma$, $\gamma\asymp\sqrt{\eta/(\mu d)}$, $\theta\asymp\sqrt{\mu\eta/d}$ and $T\gtrsim\tilde{\Omega}(1/(\mu\gamma))$. For the worst case where $\tr\bs\Sigma/\mu=\kappa(\bs\Sigma)d$, the step size $\gamma$ is larger than $\eta$ by a factor of $\sqrt{\kappa(\bs\Sigma)}$, where $\eta$ is also the step size of plain SGD~\citep{jain2018parallelizing,zou2023benign}. This suggests an acceleration by a factor of $\sqrt{\kappa(\bs\Sigma)}$ in the convergence rate.
\end{remark}

\begin{theorem}\label{thm:glm}
    Suppose Assumptions~\ref{assumption:l-condition}, \ref{assumption:regularity} and \ref{assumption:fourth-moment} hold. Let Algorithm~\ref{alg:sada} start from $\tilde{\bs x}_0$, and choose the hyperparameters as specified above in \eqref{eq:param-choice-inner} and \eqref{eq:param-choice-outer}. Let $\tbs x_K$ denote the output of Algorithm~\ref{alg:sada} after $K$ outer iterations, each consisting of $T\geq\tilde{\Omega}(\sqrt{\tilde{\kappa}/(\mu\eta)})$ inner updates. Then we have
    \begin{equation*}
        \bbE F(\tbs x_K)-F(\bs x^*)\lesssim\exp\left(-\frac{c_0 K}{\sqrt{\alpha}+\alpha^2\tilde{\kappa}/T}\right)\bigl(F(\tbs x_0)-F(\bs x^*)\bigr)+\frac{\alpha\tr(\bs H^{-1}\bs Q)}{n}+\frac{\alpha\eta\tilde{\kappa}\tr\bs Q}{L_\ell n},
    \end{equation*}
    where $n=KT$ is the sample size and $c_0$ is a universal constant.
\end{theorem}
\begin{corollary}[Sample Complexity]\label{corollary:sample-complexity}
    Under the setting of Theorem~\ref{thm:glm}, to obtain a solution $\tilde{\bs x}_K$ with excess risk at most $\varepsilon$, the required sample size $n=KT$ is
    
    \begin{equation*}
        \tilde{\cO}\Bigg( \underbrace{\left(\sqrt{\alpha\kappa\tilde{\kappa}}+\alpha^2\tilde{\kappa}\right)\vphantom{\biggl(\biggr)}}_\text{Optimization term}+\underbrace{\frac{\alpha\tr(\bs H^{-1}\bs Q)}{\varepsilon}\vphantom{\biggl(\biggr)}}_\text{Statistical term}+\underbrace{\left(\frac{\alpha^2\tilde{\kappa}^2\tr\bs Q}{L_\ell\mu\varepsilon}\right)^{1/3}\vphantom{\biggl(\biggr)}}_\text{Mis-specification term} \Bigg).
    \end{equation*}
    We choose the hyperparameter as $\tilde{\Theta}\left(\frac{\alpha\tilde{\kappa}}{\mu n^2}\right)\leq\eta\leq\tilde{\Theta}\left(\min\left\{\frac{1}{R^2},\frac{n\varepsilon L_\ell}{\alpha\tilde{\kappa}\tr\bs Q}\right\}\right)$, $T=\tilde{\Theta}(\sqrt{\tilde{\kappa}/(\mu\eta)})$, $K=n/T$ and other hyperparameters as specified in \eqref{eq:param-choice-inner} and \eqref{eq:param-choice-outer}.
    
\end{corollary}
\begin{remark}
    The sample complexity requirement is sufficient to guarantee the existence of a step size $\eta$ satisfying the above condition (see Appendix~\ref{sec:proof-sample-complexity}). Note that this $\eta$ is chosen for a finite-horizon setting, as it depends on the pre-known sample size $n$. This choice can be extended to an infinite-horizon setting (where the algorithm may be stopped at any time) using the doubling trick~\citep{hazan2014beyond}, incurring only logarithmic overhead. Moreover, the choice of hyperparameters requires pre-knowing some problem-dependent parameters, which is common in optimization methods~\citep{nesterov2018lectures}. Such dependence may be removed by restarting~\citep{o2015adaptive,necoara2019linear} or line-search procedures~\citep{bubeck2015convex}. We leave further investigation for future work.
    
\end{remark}

\paragraph{Optimization Term.}

The optimization complexity consists of two terms. The first term, $\sqrt{\alpha \kappa \tilde{\kappa}}$, corresponds to an accelerated decay of the initial error. This rate reflects a double acceleration effect, induced by momentum mechanisms at both the outer and inner loops of Algorithm~\ref{alg:sada}. The second term $\alpha^2\tilde{\kappa}$ arises from the additional noise caused by localization with magnitude proportional to the distance $\|\tilde{\bs y}_k-\bs x^*\|_{\bs\Sigma}^2$. We improve the complexity  from  previous $\alpha^2 \kappa$ in VR methods to $\alpha^2 \tilde{\kappa}$.

As a representative example, consider the Gaussian case in Remark~\ref{remark:examples} and assume outer condition number $\alpha\asymp 1$ (e.g., linear regression or logistic regression with normalized features and constant level optimal solution). In this case, Algorithm~\ref{alg:sada} achieves an optimization complexity of order $\cO((d\tr\bs\Sigma/\mu)^{1/2})$, improving over the unaccelerated rate $\cO(\tr\bs\Sigma/\mu)$ (note that $d \leq\tr\bs\Sigma/\mu$). In the worst case, $\tr\bs\Sigma/\mu =  \kappa(\bs\Sigma)d$, so our algorithm improves the optimization convergence rate by a factor of $\sqrt{\kappa(\bs\Sigma)}$. The improvement is significant when $\bs\Sigma$ is poorly conditioned.

\paragraph{Statistical Term.}
The statistical term can be interpreted as follows. This term corresponds to the noise covariance $\bs Q$ at the minimizer $\bs x^*$ suppressed along the eigen-direction of the lower bound of the objective Hessian, which is $\bs H/\alpha$. Specifically, the noise $\bs Q$ is filtered through the local curvature: directions with larger curvature contract the noise, whereas flatter directions amplify it.

This statistical term cannot be improved without additional assumptions on the third-order smoothness of $\ell$. We construct a problem class in Appendix~\ref{sec:statistical-optimality} whose minimax risk is lower bounded by $\Omega(\alpha\tr(\bs H^{-1}\bs Q)/n)$, thereby establishing the worst-case optimality of our bound.

\paragraph{Mis-specification Term.}
The mis-specification term is a coupling effect of the approximation error $\bs a\bs a^\top\neq\bs\Sigma$ (also known as fourth moment effect~\citep{jain2018accelerating,li2024risk}) and the possibility that the noise covariance $\bs Q\neq\bs H$. Note that the complexity does not depend on $\kappa$, and  on $\alpha^{2/3}$, $\mu^{-1/3}$, and $\varepsilon^{-1/3}$. For high accuracy with $\varepsilon$, a small step size is required. Theorem~\ref{thm:glm} implies that the misspecification error vanishes asymptotically. 

For well-specified models with $\bs Q=\sigma^2\bs\Sigma$, the mis-specification term can be absorbed into the statistical term with some additional analysis. This holds directly for Gaussian and sub-Gaussian distributions in Remark~\ref{remark:examples} by applying $\kappa\asymp\tr\bs\Sigma/\mu$ and $\tilde{\kappa}\asymp d$ to the bound in Theorem~\ref{thm:glm}, so the mis-specification term $\frac{\alpha\eta\tilde{\kappa}\tr\bs Q}{L_\ell n}=\frac{\sigma^2\alpha\eta d\tr\bs\Sigma}{n}\leq\frac{\sigma^2\alpha d}{n}=\frac{\alpha\tr(\bs H^{-1}\bs Q)}{n}$, where we use $\eta\tr\bs\Sigma\leq 1$.

\paragraph{Answers to Open Problem.}
Our result answers the open problem by \citet{jain2018accelerating} in the General Linear Prediction setting with minimum assumptions. Concretely, we show that acceleration achieves the asymptotic optimal convergence rate in the following sense:
\begin{equation*}
    \limsup_{n\to\infty}\frac{\bbE F(\bs x_n^\text{SADA})-F(\bs x^*)}{\tr(\bs H^{-1}\bs Q)/n}=\cO(\alpha),
\end{equation*}
where, as discussed above, the term $\cO(\alpha)$ is optimal in our setting.

\section{Proof Sketch}\label{sec:proof-sketch}

\subsection{Part I: Analysis of Inner Loop}\label{sec:sketch-inner-loop}

The subproblem~\eqref{eq:subproblem} resembles linear regression with data pair $(\bs a, -h_k\ell'(\bs a^\top\tbs y_{k-1},b))$. We denote the minimizer of \eqref{eq:subproblem} by $\tilde{\bs x}_k^*$. Applying the standard bias-variance decomposition~\citep{dieuleveut2016nonparametric,jain2018accelerating,zou2023benign,wu2022last,li2024risk}, we can decompose the dynamics of $\bs\eta_t\defeq\begin{pmatrix}
    \bs x_t-\tbs x_k^* \\
    \bs y_t-\tbs x_k^*
\end{pmatrix}\in\bbR^{2d}$ into the sum of bias term $\bs\eta_t^\rmbias$ and variance term $\bs\eta_t^\rmvar$:

\begin{equation*}
    \begin{cases}
        \bs\eta_t^\rmbias=\hat{\bs A}_t\bs\eta_{t-1}^\rmbias,&\bs\eta_0^\rmbias=\bs\eta_0, \\
        \bs\eta_t^\rmvar=\hat{\bs A}_t\bs\eta_{t-1}^\rmvar\bs+\bs\zeta_t,&\bs\eta_0^\rmvar=\bs 0,
    \end{cases}\quad\text{where}\quad\hbs A_t=\begin{pmatrix}
        \bs O & \bs I-\eta\bs a_t\bs a_t^\top \\
        -\frac{1-\theta}{1+\theta}\bs I & \frac{2}{1+\theta}\bs I-\frac{\eta+\theta\gamma}{1+\theta}\bs a_t\bs a_t^\top
    \end{pmatrix}.
\end{equation*}
The inner loop output $\tilde{\bs x}_k$ can be decomposed as $\tilde{\bs x}_k=\tilde{\bs x}_k^*+h_k\bs r_k+h_k\bs v_k$, where $\bs r_k\in\bbR^d$ is the bias term that captures the optimization error, and $\bs v_k\in\bbR^d$ is the noise term with zero mean. The overall result in the inner loop is shown in Lemma~\ref{lemma:inner-loop}.
The bound of $\bs r_k$  follows from \citet{jain2018accelerating} using standard acceleration technique shown in Appendix~\ref{sec:bias-bound}.  The main focus is to derive a bound of $\|\bs v_k\|_{\bs H}^2$, which is non-trivial. 

\paragraph{Bound of $\bs v_k$.}
It requires to  analyze the covariance $\bs C_t=\bbE\left(\bs\eta_t^\rmvar(\bs\eta_t^\rmvar)^\top\right)$. Its dynamics of $\bs C_t$ admits
\begin{equation*}
    \bs C_t=\cB\circ\bs C_{t-1}+\begin{pmatrix}
    \eta^2\bs R & \eta q\bs R \\
    \eta q\bs R & q^2\bs R \\
\end{pmatrix},\quad\bs C_0=\bs O.
\end{equation*}
where $\bs R\in\bbR^{2d\times 2d}$ is the noise covariance and $\cB:\bs C\mapsto\bbE\hbs A_t\bs C\hbs A_t^\top$ is a linear matrix operator acts on $\bbR^{2d\times 2d}$ and $\circ$ represents the operation of a linear matrix operator on a matrix. Note that in prior work on well-specified linear regression, where $\bs R = \sigma^2 \bs\Sigma$ commutes with the data covariance $\mathbf{\Sigma}$, the stationary property of $\bs C_t$ is relatively simple to analyze. In contrast, our setting necessitates understanding the behavior of $\bs C_t$ for a general $\bs R$.

We propose the Layer-Peeled Decomposition, which decomposes $\bs C_t$ to layers of dynamics $\tbs C_t$ and $\tbs C_t^{(\ell)}$. 
Specifically, we introduce $\bs C_t=\tbs C_t+\sum_{\ell=0}^T\tbs C_t^{(\ell)}$, where
\begin{equation*}
\begin{aligned}
    &\tbs C_t=\bs A\tbs C_{t-1}\bs A+\bs R,&\quad&\tbs C_0=\bs O,&\quad&\tbs C_t^{(0)}=\tbs C_t, \\
    &\tbs C_t^{(\ell)}=\bs A\tbs C_{t-1}^{(\ell)}\bs A+(\cB-\bs A\otimes\bs A)\circ\tbs C_t^{(\ell-1)},&\quad&\tbs C_0^{(\ell)}=\bs O, &\quad&\text{for $\ell=1,2,\ldots$}
\end{aligned}
\end{equation*}
and $\bs A=\bbE\hat{\bs A}_t$. Intuitively, $\tbs C_t$ can be viewed as the covariance of dynamics by replacing $\bs a\bs a^\top$ by $\bs\Sigma$, and $\tbs C_t^{(\ell)}$ collects these approximation errors. In order to study  $\tbs C_t$ and  $\tbs C_t^{(\ell)}$,   we introduce a core auxiliary dynamics $\tbs\Pi_t(\bs M)\in\bbR^{2d\times 2d}$ for $\bs M\in\bbR^{d\times d}$  as follows:
\begin{equation}
    \tbs\Pi_t(\bs M)=\bs A\tbs\Pi_{t-1}(\bs M)\bs A+\begin{pmatrix}
        \eta^2\bs M & \eta q\bs M \\
        \eta q\bs M & q^2\bs M
    \end{pmatrix},\quad\tbs\Pi_0(\bs M)=\bs O.\label{eq:sketch-Pi}
\end{equation}
Because we have $\tbs C_t=\tbs\Pi_t(\bs R)$, and $\tbs C_{t-1}^{(\ell)}$ can be bounded by $\tbs\Pi_t(\bs R^{(\ell)})$ for some $\bs R^{(\ell)}$ which relates to the stationary covariance of $\bs C_t^{(\ell-1)}$. As a result, we can recursively have the result for $\tbs C_t^{(\ell)}$ to its limit, thereby bounding the variance.

The detailed proof for bounding the variance is provided in Appendix \ref{sec:variance-bound}. In Definition~\ref{def:L-var}, we introduce a linear mapping 
$L_\rmvar$ to calculate the loss of a general covariance matrix sequence.
In Appendix~\ref{sec:core-dynamics}, we study the dynamics of \eqref{eq:sketch-Pi}. In particular, we obtain a sharp upper bound of $L_\rmvar(\{\tbs\Pi_t(\bs M)\}_{t\in\bbN})$ in Lemma~\ref{lemma:Pi-average-bound} and stationary covariance $\tbs\Pi_\infty(\bs M)=\lim_{t\to\infty}\tbs\Pi_t(\bs M)$ in Lemma~\ref{lemma:Pi-stationary-bound}. In Appendix~\ref{sec:localize-R}, we study the noise covariance matrix $\bs R$, and decompose $\tr(\bs\Sigma^{-1}\bs R)$ and $\tr(\bs R)$. Finally, in Appendix~\ref{sec:bound-C-1}, we introduce $\tbs C_t$ and $\tbs C_t^{(\ell)}$. We first show that the variance can be further decomposed as follows: $\bbE\|\bs v_k\|_{\bs\Sigma}^2=L_\rmvar(\{\tbs C_t\}_{t\in\bbN})+\sum_{\ell=1}^T L_\rmvar(\{\tbs C_t^{(\ell)}\}_{t\in\bbN})$ in Lemma~\ref{lemma:v-decom}. Then by plugging in Lemma~\ref{lemma:L-var-bound-C-l}, we have the final bound of variance.

\subsection{Part II: Analysis of Outer Loop}\label{sec:sketch-outer-loop}

The bounds of the outer loop mainly exploit the structure of the solution $\tbs x_k$ returned by the inner loop. Our proof technique follows the analysis of AGD, but with a careful control of the noise Term (C). The main result is the following lemma.
\begin{lemma}
    Suppose we have $h_k\lesssim\min\left\{\frac{1}{L_\ell},\frac{T^2}{L_\ell\alpha^3\tilde{\kappa}^2}\right\}$ and set $\theta_k=\sqrt{\frac{h_k}{2\alpha}}$. Let the energy $L_k=\bbE F(\tbs x_k)-F(\bs x^*)+\frac{2\theta_k^2}{3 L_\ell h_k}\bbE_{k-1}\|\tbs z_k-\bs x^*\|_{\bs H}^2$, then we have
    \begin{equation*}
        L_k\leq\left(1-\frac{\theta_k}{2}\right) L_{k-1}+\frac{3 h_k\sigma^2}{2},\quad\text{where $\sigma^2\lesssim\frac{L_\ell\tr(\bs H^{-1}\bs Q)}{T}+\frac{\eta\tilde{\kappa}\tr\bs Q}{T}$}.
    \end{equation*}
\end{lemma}

We employ a two-phase step size scheme as in \eqref{eq:param-choice-outer}. For the first phase ($t\leq K/2$), we choose the maximum step size to accelerate the convergence of the initial error. For the second phase ($t>K/2$), we use $h_{k+K/2}\asymp\frac{1}{\mu_\ell(a+k)^2}$ to control the noise accumlation.

\section{Extensions and Discussions}\label{sec:extensions}

\subsection{Weakly Convex Case}
We apply the reduction technique to extend our results to the weakly convex case. Suppose $\|\bs x^*\|\leq D$ and $\|\bs x^*\|_{\bs H}\leq M$, then we consider optimizing the following sorrogate function:
\begin{equation*}
    G(\bs x)=\bbE_{\bs a,b\sim\cD}\bigg(\ell(\bs a^\top\bs x,b)+\underbrace{\frac{\varepsilon}{2 M^2}\left(\bs a^\top\bs x\right)^2}_{\text{Term (A)}}\bigg)+\underbrace{\frac{\varepsilon}{2 D^2}\|\bs x\|^2}_\text{Term (B)},
\end{equation*}
where Term (A) is needed when $\ell$ is weakly convex ($\mu_\ell\leq \varepsilon$), and Term (B) is required when $\bs\Sigma$ is poor-conditioned ($\mu\leq\varepsilon$). The algorithm SADA-WC (Algorithm~\ref{alg:sada-wc}) is deferred to Appendix~\ref{sec:sada-extension}.

\begin{corollary}[Sample Complexity]
    Under the setting of Theorem~\ref{thm:glm}. Suppose we run Algorithm~\ref{alg:sada-wc} to obtain a solution $\tilde{\bs x}_K$ with excess risk at most $\varepsilon$, the required sample size $n$ is
    \begin{equation*}
        \left(\sqrt{\frac{L_\ell R^2\tilde{\kappa}}{\mu'_\ell\mu'}}+\frac{L_\ell^2\tilde{\kappa}}{(\mu'_\ell)^2}\right)+\frac{L_\ell\tr((\bs H+\varepsilon\bs I/D^2)^{-1}\bs Q)}{\mu'_\ell\varepsilon}\vphantom{\left(\frac{\tilde{\kappa}^2\tr\bs Q}{\mu\varepsilon}\right)^{1/3}}+\left(\frac{L_\ell\tilde{\kappa}^2\tr\bs Q}{(\mu'_\ell)^2\mu'\varepsilon}\right)^{1/3},
    \end{equation*}
    where $\mu'_\ell=\max\{\mu_\ell,\varepsilon\}$ and $\mu'=\max\{\mu,\varepsilon\}$.
    
\end{corollary}

\subsection{Use of Unlabeled Data}\label{sec:extensions-unlabeled-data}
In the inner loop, we approximate $\bs{\Sigma}$ using a single data point $\bs{a}\bs{a}^\top$. This process does not require the corresponding label $b$. We can leverage available unlabeled data to refine our estimate of $\bs{\Sigma}$. Specifically, at certain steps, we use $m$ unlabeled samples to obtain a better estimate (see Algorithm~\ref{alg:sada-ud} in Appendix~\ref{sec:sada-extension}). Then we can improve $R^2$, $\kappa$ and $\tilde{\kappa}$ to $R_m^2=\frac{R^2+m\lambda_\rmmax(\bs\Sigma)}{m+1}$, $\kappa_m=\frac{R_m^2}{\mu}$ and $\tilde{\kappa}_m=\frac{\tilde{\kappa}+m}{m+1}$, respectively. In other words, to obtain a solution with excess risk at most $\varepsilon$, the required labeled sample size $n$ is
\begin{equation*}
    n\geq\tilde{\Theta}\left(\left(\sqrt{\alpha\kappa_m\tilde{\kappa}_m}+\alpha^2\tilde{\kappa}\right)+\frac{\alpha\tr(\bs H^{-1}\bs Q)}{\varepsilon}+\left(\frac{\alpha^2\tilde{\kappa}_m^2\tr\bs Q}{L_\ell\mu\varepsilon}\right)^{1/3}\right).
\end{equation*}
The total unlabeled data size is at least $\tilde{\Omega}(m\sqrt{\alpha\tilde{\kappa}_m/\mu\eta})$, where we set $\eta=\tilde{\Theta}\left(\min\left\{\frac{1}{R_m^2},\frac{n\varepsilon L_\ell}{\alpha\tilde{\kappa}_m
\tr\bs Q}\right\}\right)$ and $K=\tilde{\Theta}(\sqrt{\alpha})$ and $T=n/K$. 
The redundancy in the labeled-data complexity can be attributed to optimization and mis-specification errors, and is quantified directly by $R_m^2$, $\kappa_m$, and $\tilde{\kappa}_m$.

\subsection{Mini-batching and Parallelization}

\paragraph{Mini-batching.}The algorithm defaults to single-data updates but can be extended to a batch setting. Using a batch size of $B$, we
can redefine $R^2$, $\kappa$ and $\tilde{\kappa}$ to $R_B^2=\frac{R^2+(B-1)\lambda_\rmmax(\bs\Sigma)}{B}$, $\kappa_B=\frac{R_B^2}{\mu}$ and $\tilde{\kappa}_B=\frac{\tilde{\kappa}+(B-1)}{B}$, respectively. To obtain a solution with excess risk at most $\varepsilon$, the sample size $n$ is
\begin{equation*}
    n\geq\tilde{\Theta}\left(B\left(\sqrt{\alpha\kappa_B\tilde{\kappa}_B}+\alpha^2\tilde{\kappa}_B\right)+\frac{\alpha\tr(\bs H^{-1}\bs Q)}{\varepsilon}+\left(\frac{\alpha^2\tilde{\kappa}_B^2 B^2\tr\bs Q}{L_\ell\mu\varepsilon}\right)^{1/3}\right),
\end{equation*}
and the maximum batch-size $B_\rmmax=n\sqrt{\mu\eta/(\alpha\kappa_B)}$, where we set $\eta=\tilde{\Theta}\left(\min\left\{\frac{1}{R_B^2},\frac{n\varepsilon L_\ell}{\alpha\tilde{\kappa}_B
\tr\bs Q}\right\}\right)$.
and $K=\tilde{\Theta}(\sqrt{\alpha})$ and $T=n/K$.

\paragraph{Parallelization.} The inner loop of Algorithm~\ref{alg:sada} can be implemented in parallel when $T\geq T_0=\tilde{\Theta}(\sqrt{\tilde{\kappa}/(\mu\eta)})$. The $T$ inner iterations can be partitioned into $T/T_0$ independent inner runs, and the outer iterate $\tbs x_k$ is defined as the average of their outputs. The convergence guarantee in Theorem~\ref{thm:glm} holds under this parallel implementation.

\subsection{High-order Smoothness Assumption}
Our analysis, like the conventional gradient-based optimization and Local Rademacher Complexity frameworks it relates to,  does not assume third-order smoothness of $\ell$. 
Under this minimal regularity, the obtained statistical rate is minimax optimal. If Hessian smoothness is assumed, one may expect the sharper rate $\tr((\nabla^2 F(\bs x^*))^{-1}\bs Q)/n$, matching the asymptotic efficiency of ERM and Cram\'er-Rao lower bound. 
A direct approach would be to first run SADA until the iterate enters a neighborhood of $\bs x^*$, e.g., $\| \tilde{\bs x}  - \bs x^* \| \leq \frac{1}{\mu_\ell \mu}$. Then, when solving subproblem~\ref{eq:subproblem}, one would replace $\mathbf{a}\mathbf{a}^\top$ with $\ell''(\mathbf{a}^\top \tilde{\mathbf{x}}, b) \mathbf{a}\mathbf{a}^\top$ using a carefully chosen step size. However, the analysis is much more involved. Moreover, higher-order methods, such as cubic regularization~\citep{nesterov2006cubic}, could be more appropriate as it achiever faster rate for the outer loop when the Hessian is smooth. We leave the study as future work.

\bibliographystyle{plainnat}
\bibliography{references}

\newpage

\appendix

\tableofcontents

\section{Organization}
The appendix is organized as follows:
\begin{itemize}
    \item In Appendix~\ref{sec:generalize-assumptions}, we present two general assumptions (Assumptions~\ref{assumption:gradient-noise-I} and \ref{assumption:gradient-noise}) to characterize the stochastic gradient noise of objective $F$. Under Assumptions~\ref{assumption:l-condition}, \ref{assumption:regularity} and \ref{assumption:fourth-moment} in the main text, Assumptions~\ref{assumption:gradient-noise-I} and \ref{assumption:gradient-noise} hold with $B=\alpha R^2$ and $L=\alpha\tilde{\kappa}$ (see Remark~\ref{remark:B-L}). The results in the Appendix are stated with dependence on $B$ and $L$, and immediately transferred to the results in the main text by choosing $B=\alpha R^2$ and $L=\alpha\tilde{\kappa}$.
    \item In Appendix~\ref{sec:hyperparameter-choice}, we detail the hyperparameter choice with explicit dependence on the constant and logarithmic factors.
    \item In Appendix~\ref{sec:inner-loop}, we presents the analysis of inner loop of Algorithm~\ref{alg:sada}.
    \item In Appendix~\ref{sec:outer-loop}, we presents the analysis of outer loop of Algorithm~\ref{alg:sada}.
    \item In Appendix~\ref{sec:erm-excess-risk}, we present a proof of the non-asymptotic excess risk upper bound of ERM using the localization technique~\citep{bartlett2005local,wainwright2019high}.
    \item In Appendix~\ref{sec:statistical-optimality}, we show that the statistical term in Theorem~\ref{thm:glm} is minimax optimal by constructing a hard problem instance. For this problem, the excess risk must be lower bounded by $\Omega(\alpha\tr(\bs H^{-1}\bs Q)/n)$.
    \item In Appendix~\ref{sec:sada-extension}, we present algorithms that work for the weakly convex case (Algorithm~\ref{alg:sada-wc}) and can use unlabeled data (Algorithm~\ref{alg:sada-ud}).
\end{itemize}

\section{Generalized Assumptions}\label{sec:generalize-assumptions}
This section presents two general assumptions that provide a fine characterization of the gradient noise, and then we show that they hold under Assumptions~\ref{assumption:l-condition}, \ref{assumption:regularity}, and \ref{assumption:fourth-moment}.

\begin{assumption}[Bound of $2$-norm of Gradient Noise]\label{assumption:gradient-noise-I}
    We assume there exists $B>0$ such  that for all $\bs x$, the following holds:
    \begin{equation*}        
        \bbE_{\bs a,b\sim\cD}\left(\ell'(\bs a^\top\bs x,b)-\ell'(\bs a^\top\bs x^*,b)\right)^2\|\bs a\|^2\leq 2L_\ell B(F(\bs x)-F(\bs x^*)).
    \end{equation*}
\end{assumption}

\begin{assumption}[Bound of $\bs\Sigma^{-1}$-norm of Gradient Noise]\label{assumption:gradient-noise}
    We assume there exists $L>0$ such  that for all $\bs x$, the following holds:
    \begin{equation*}        
        \bbE_{\bs a,b\sim\cD}\left(\ell'(\bs a^\top\bs x,b)-\ell'(\bs a^\top\bs x^*,b)\right)^2\|\bs a\|_{\bs\Sigma^{-1}}^2\leq 2L_\ell L(F(\bs x)-F(\bs x^*)).
    \end{equation*}
\end{assumption}
\begin{remark}\label{remark:B-L}
    Under Assumption~\ref{assumption:l-condition}, \ref{assumption:regularity} and \ref{assumption:fourth-moment}, the above assumptions hold with $B=\alpha R^2$ and $L=\alpha\tilde{\kappa}$, as shown in Lemma~\ref{lemma:gradient-noise-I} and Lemma~\ref{lemma:gradient-noise}. In addition, if a sharper bound is available, then the convergence rate in Theorem~\ref{thm:glm} can be improved. See Theorem~\ref{thm:full-glm} for the possibly sharper convergence rate that explicitly depends on $B$ and $L$.
\end{remark}

We verify Assumptions~\ref{assumption:gradient-noise-I} and \ref{assumption:gradient-noise} under Assumptions~\ref{assumption:l-condition}, \ref{assumption:regularity} and \ref{assumption:fourth-moment}.

\begin{lemma}[Verification of Assumption~\ref{assumption:gradient-noise-I}]\label{lemma:gradient-noise-I}
    Suppose Assumptions~\ref{assumption:l-condition}, \ref{assumption:regularity} and \ref{assumption:fourth-moment} hold. Then for any $\bs x\in\bbR^d$, we have
    \begin{equation*}
        \bbE_{\bs a,b\sim\cD}\left(\ell'(\bs a^\top\bs x,b)-\ell'(\bs a^\top\bs x^*,b)\right)^2\|\bs a\|^2\leq 2L_\ell\alpha R^2(F(\bs x)-F(\bs x^*)).
    \end{equation*}
    \begin{proof}
        By Assumption~\ref{assumption:l-condition}, we have $0\leq\ell''(\cdot,\cdot)\leq L_\ell$. Thus, $|\ell'(\bs a^\top\bs x,b)-\ell'(\bs a^\top\bs x^*,b)|\leq L_\ell|\bs a^\top(\bs x-\bs x^*)|$. By Assumption~\ref{assumption:fourth-moment}, the left-hand side can be bounded as follows:
        \begin{equation*}
        \begin{aligned}
            \phrel\bbE_{\bs a,b\sim\cD}\left(\ell'(\bs a^\top\bs x,b)-\ell'(\bs a^\top\bs x^*,b)\right)^2\|\bs a\|^2&\leq L_\ell^2\bbE_{\bs a,b\sim\cD}\left(\bs a^\top(\bs x-\bs x^*)\right)^2\|\bs a\|^2 \\
            &\leq L_\ell^2 R^2\|\bs x-\bs x^*\|_{\bs\Sigma}^2\leq 2L_\ell\alpha R^2 (F(\bs x)-F(\bs x^*)).
        \end{aligned}
        \end{equation*}
        This completes the proof.
    \end{proof}
\end{lemma}

\begin{lemma}[Verification of Assumption~\ref{assumption:gradient-noise}]\label{lemma:gradient-noise}
    Suppose Assumptions~\ref{assumption:l-condition}, \ref{assumption:regularity} and \ref{assumption:fourth-moment} hold. Then for any $\bs x\in\bbR^d$, we have
    \begin{equation*}
        \bbE_{\bs a,b\sim\cD}\left(\ell'(\bs a^\top\bs x,b)-\ell'(\bs a^\top\bs x^*,b)\right)^2\|\bs a\|_{\bs\Sigma^{-1}}^2\leq 2L_\ell\alpha\tilde{\kappa}(F(\bs x)-F(\bs x^*)).
    \end{equation*}
    \begin{proof}
        By Assumptions~\ref{assumption:l-condition} and \ref{assumption:fourth-moment}, the left-hand side can be bounded as follows:
        \begin{equation*}
        \begin{aligned}
            \phrel\bbE_{\bs a,b\sim\cD}\left(\ell'(\bs a^\top\bs x,b)-\ell'(\bs a^\top\bs x^*,b)\right)^2\|\bs a\|_{\bs\Sigma^{-1}}^2&\leq L_\ell^2\bbE_{\bs a,b\sim\cD}\left(\bs a^\top(\bs x-\bs x^*)\right)^2\|\bs a\|_{\bs\Sigma^{-1}}^2 \\
            &\leq L_\ell^2\tilde{\kappa}\|\bs x-\bs x^*\|_{\bs\Sigma}^2\leq 2L_\ell\alpha\tilde{\kappa}(F(\bs x)-F(\bs x^*)).
        \end{aligned}
        \end{equation*}
        This completes the proof.
    \end{proof}
\end{lemma}

\section{Hyperparameter Choice}\label{sec:hyperparameter-choice}
In this section, we provide hyperparameter setup that depends on the constants in Assumptions~\ref{assumption:gradient-noise-I} and \ref{assumption:gradient-noise}. We also provide explicit constants and logarithmic factors for hyperparameters $T$ and $\tilde{\theta}_\rmmax$. We choose the hyperparameters $\eta$ and $T$ such that
\begin{equation*}
    \eta\leq\frac{1}{16 R^2},\quad T\geq\sqrt{\frac{\tilde{\kappa}}{\mu\eta}}\ln\kappa(\bs\Sigma)\ln\frac{4}{\tilde{\theta}_K^2}.
\end{equation*}
Then choose the hyperparameters of the inner loop as follows:
\begin{equation*}
    \gamma=\frac{1}{4}\sqrt{\frac{\eta}{\tilde{\kappa}\mu}},\quad\theta=\frac{1}{4}\sqrt{\frac{\mu\eta}{\tilde{\kappa}}}.
\end{equation*}
For the outer loop, let $L_\rmeff=160(6L+\tilde{\kappa}(7+16\eta B))$, where $B$ and $L$ are defined in Assumptions~\ref{assumption:gradient-noise-I} and \ref{assumption:gradient-noise} (see also Remark~\ref{remark:B-L}). We define
\begin{equation*}
    \tilde{\theta}_\rmmax=\min\left\{\sqrt{\frac{1}{12\alpha}},\frac{T}{12\sqrt{2}\alpha L_\rmeff}\right\},\quad\tilde{\theta}_k=\begin{cases}
        \tilde{\theta}_\rmmax,&k\leq K/2,\\
        \dfrac{4}{4/\tilde{\theta}_\rmmax+k-K/2},&k>K/2.
    \end{cases}
\end{equation*}
Then we set the hyperparameters of the outer loop as follows:
\begin{equation*}
    h_k=\frac{2\alpha\tilde{\theta}_k^2}{L_\ell},\quad \beta_k=\frac{1-\tilde{\theta}_k}{1+\tilde{\theta}_k}.
\end{equation*}

\section{Part I: Analysis of Inner Loop}\label{sec:inner-loop}
\paragraph{Organization.}
In this section, we study the inner loop corresponding to the $k$-th outer iteration of Algorithm~\ref{alg:sada}.
\begin{itemize}
    \item In Appendix~\ref{sec:preliminary}, we introduce preliminary tools for analyzing linear regression. In Appendix~\ref{sec:bias-variance-decomposition}, we introduce the bias-variance decomposition technique~\citep{dieuleveut2016nonparametric,jain2018accelerating,zou2023benign} to decompose the dynamics into a bias term, which characterizes the optimization error, and a variance term, which characterizes the stochastic noise. In Appendix~\ref{sec:covariance-dynamics}, we derive the dynamics of the bias iterate covariance $\bs B_t$ and the variance iterate covariance $\bs C_t$. In Appendix~\ref{sec:property-A}, we summarize the results concerning the momentum matrix $\bs A$, which governs the contraction of the dynamics $\bs\eta_t^\rmbias$ and $\bs\eta_t^\rmvar$.
    \item In Appendix~\ref{sec:variance-bound}, we present the complete proof for bounding $\bs C_t$, which is sketched in Section~\ref{sec:sketch-inner-loop}. In Appendix~\ref{sec:core-dynamics}, we propose the core auxiliary dynamics $\tbs\Pi_t(\cdot)$. Then, we derive an upper bound of $L_\rmvar(\{\tbs\Pi_t(\bs M)\}_{t\in\bbN})$ and analyze the stationary covariance $\tbs\Pi_\infty(\bs M)$. By a sharp bound of $\tbs\Pi_t(\cdot)$, we are able to derive the bound of the variance term. In Appendix~\ref{sec:bound-tilde-C}, we apply the tools in Appendix~\ref{sec:core-dynamics} to prove bounds on the layer 0 dynamics $\tbs C_t$. In Appendix~\ref{sec:bound-C-1}, we apply the tools in Appendix~\ref{sec:core-dynamics} to prove bounds on the dynamics $\tbs C_t^{(\ell)}$ for $\ell\geq 1$.
    \item In Appendix~\ref{sec:bias-bound}, we presents the proof for bounding the bias iterate $\bs\eta_t^\rmbias$. The proof mainly follows \citet{jain2018accelerating}.
    \item In Appendix~\ref{sec:proof-inner-loop-lemma}, we presents the proof of Lemma~\ref{lemma:inner-loop}, which is the full version of Lemma~\ref{lemma:inner-loop}. This lemma is the basis for the analysis of the outer loop.
\end{itemize}

To simplify the notation, we denote $\tbs x\defeq\tbs x_{k-1}$, $\tilde{\bs y}\defeq\tilde{\bs y}_{k-1}$ and $\tbs x_+\defeq\tbs x_k$. Let $\tbs x_+^*$ denote the minimizer of the $k$-th outer iteration objective
\begin{equation*}    
    \min_{\bs x\in\bbR^d}\;\bbE_{\bs a,b\sim \cD}\langle\ell'(\bs a^\top\tbs y_{k-1},b)\bs a,\bs x-\tbs y_{k-1}\rangle+\frac{1}{2 h_k}\|\bs x-\tbs y_{k-1}\|_{\bs \Sigma}^2.
\end{equation*}

The goal of this section is to derive the following decomposition:
\begin{equation*}
     \tbs x_+=\tbs x_+^*+h_k\bs r_k+h_k\bs v_k,
\end{equation*}
where $\bs r_k$ and $\bs v_k$ are random variables which satisfy the following lemma.
\begin{lemma}\label{lemma:inner-loop}
    Suppose Assumptions~\ref{assumption:l-condition}, \ref{assumption:regularity}, \ref{assumption:fourth-moment}, \ref{assumption:gradient-noise-I} and \ref{assumption:gradient-noise} hold. Then
    \begin{align*}
        &\bbE\|\bs r\|_{\bs H}^2\leq\frac{L_\ell h_k}{8\alpha}\|\nabla F(\tbs y_{k-1})\|_{\bs H^{-1}}^2,\quad \bbE\bs v=\bs 0, \\
        &\bbE\|\bs v\|_{\bs H}^2\leq\frac{320 L_\ell\left(3 L_\ell\tr(\bs H^{-1}\bs Q)+8\eta\tilde{\kappa}\tr\bs Q\right)}{T}+\frac{160 L_\ell^2(6L+\tilde{\kappa}(7+16\eta B))(F(\tilde{\bs y})-F(\bs x^*))}{T},
    \end{align*}
    where the expectation is taken with respect to the samples drawn in the $k$-th outer iteration.
\end{lemma}

Without loss of generality, we assume $h_k=1$ in the following of this section. From now on, we choose a basis that diagonalizes $\bs\Sigma$. Therefore, $\bs\Sigma$ and $\bs H$ are diagonal matrices.

\subsection{Preliminary}\label{sec:preliminary}
\subsubsection{Bias-variance Decomposition}\label{sec:bias-variance-decomposition}
We first provide a concise matrix form of the inner loop update.
\begin{lemma}
    The update of the inner loop has the following form:
    \begin{equation}
        \begin{pmatrix}
            \bs x_t-\tbs x_+^* \\
            \bs y_t-\tbs x_+^*        
        \end{pmatrix}=\begin{pmatrix}
            \bs O & \bs I-\eta\bs a_t\bs a_t^\top \\
            -c\bs I & (1+c)\bs I-q\bs a_t\bs a_t^\top
        \end{pmatrix}
        \begin{pmatrix}
            \bs x_{t-1}-\tbs x_+^* \\
            \bs y_{t-1}-\tbs x_+^*        
        \end{pmatrix}+\begin{pmatrix}
            \eta\varepsilon_t\bs a_t \\
            q\varepsilon_t\bs a_t
        \end{pmatrix}.\label{eq:inner-matrix-form}
    \end{equation}
    where
    \begin{equation}
        c=\frac{1-\theta}{1+\theta},\quad q=\frac{\eta+\theta\gamma}{1+\theta},\quad \varepsilon_t=-\ell'(\bs a_t^\top\tbs y,b_t)+\bs a_t^\top(\tbs y-\tbs x_+^*).\label{eq:def-c-q-epsilon}
    \end{equation}
    \begin{equation*}
    \end{equation*}
\end{lemma}
\begin{proof}
    Recall that the update rule is
    \begin{align*}
        \bs x_t&=\bs y_{t-1}-\eta\bs g_t, \\
        \bs z_t&=\theta\bs y_{t-1}+(1-\theta)\bs z_{t-1}-\gamma\bs g_t, \\
        \bs y_t&=\frac{1}{1+\theta}\bs x_t+\frac{\theta}{1+\theta}\bs z_t,
    \end{align*}
    where the gradient $\bs g_t$ is
    \begin{equation*}
        \bs g_t=\ell'(\bs a_t^\top\tbs y,b_t)\bs a_t+\bs a_t\bs a_t^\top(\bs y_{t-1}-\tbs y),
    \end{equation*}
    By algebraic transformations to eliminate $\bs z_t$ and $\bs z_{t-1}$, we obtain the following matrix form:
    \begin{equation}
        \begin{pmatrix}
            \bs x_t \\
            \bs y_t        
        \end{pmatrix}=\begin{pmatrix}
            \bs O & \bs I \\
            -c\bs I & (1+c)\bs I
        \end{pmatrix}
        \begin{pmatrix}
            \bs x_{t-1} \\
            \bs y_{t-1}        
        \end{pmatrix}-\begin{pmatrix}
            \eta\bs g_t \\
            q\bs g_t
        \end{pmatrix},\label{eq:inner-matrix-form-simple}
    \end{equation}
    where $c$ and $q$ are define in \eqref{eq:def-c-q-epsilon}. Note that by the definition of $\varepsilon_t$ in \eqref{eq:def-c-q-epsilon}, we have
    \begin{equation}
        \bs g_t=\bs a_t\bs a_t^\top(\bs y_{t-1}-\tbs x_+^*)-\varepsilon_t\bs a_t.\label{eq:gradient-epsilon-form}
    \end{equation}
    We complete the proof by substituting \eqref{eq:gradient-epsilon-form} into \eqref{eq:inner-matrix-form-simple}.
\end{proof}

Define
\begin{equation*}
    \bs\eta_t=\begin{pmatrix}
        \bs x_t-\tbs x_+^* \\
        \bs y_t-\tbs x_+^*        
    \end{pmatrix},\quad\hbs A_t=\begin{pmatrix}
        \bs O & \bs I-\eta\bs a_t\bs a_t^\top \\
        -c\bs I & (1+c)\bs I-q\bs a_t\bs a_t^\top
    \end{pmatrix},\quad \bs\zeta_t=\begin{pmatrix}
        \eta\varepsilon_t\bs a_t \\
        q\varepsilon_t\bs a_t
    \end{pmatrix}.
\end{equation*}
Then the dynamic of $\bs\eta_t$ have the following form,
\begin{equation*}
    \bs\eta_t=\hbs{A}_t\bs\eta_{t-1}+\bs\zeta_t,\quad\bs\eta_0=\begin{pmatrix}
        \tbs y-\tbs x_+^* \\
        \tbs y-\tbs x_+^* \\
    \end{pmatrix}.
\end{equation*}
Following the bias-variance decomposition technique~\citep{dieuleveut2016nonparametric,jain2018accelerating,zou2023benign}, we decompose $\bs\eta_t$ into the bias dynamic $\bs\eta_t^\rmbias$ and variance dynamic $\bs\eta_t^\rmvar$ as follows:
\begin{equation}
\begin{aligned}
    &\bs\eta^\rmbias_t=\hbs{A}_t\bs\eta^\rmbias_{t-1},&\quad&\bs\eta^\rmbias_0=\begin{pmatrix}
        \tbs y-\tbs x_+^* \\
        \tbs y-\tbs x_+^* \\
    \end{pmatrix}, \\
    &\bs\eta^\rmvar_t=\hbs{A}_t\bs\eta^\rmvar_{t-1}+\bs\zeta_t, &\quad&\bs\eta^\rmvar_0=\bs 0.
\end{aligned}\label{eq:def-eta}
\end{equation}
Note that $\bs\eta_t=\bs\eta_t^\rmbias+\bs\eta_t^\rmvar$ since the dynamics are linear. 
With the bias-variance decomposition given above, we construct $\bs r$ and $\bs v$ as follows.
\begin{definition}\label{def:r-v}
    Let
    \begin{equation*}
        \bs r=\frac{2}{h_k T}\left(\sum_{t=T/2+1}^T\bs\eta_t^\rmbias\right)_1,\quad \bs v=\frac{2}{h_k T}\left(\sum_{t=T/2+1}^T\bs\eta_t^\rmvar\right)_1,
    \end{equation*}
    where for $\bs\eta\in\bbR^{2d}$, $\bs\eta_1\in\bbR^d$ denotes the first $d$ elements of $\bs\eta$.
\end{definition}

\subsubsection{Covariance Dynamics}\label{sec:covariance-dynamics}
Define
\begin{equation*}
    \bs R=\bbE_{\bs a,b\sim\cD}\left(\varepsilon^2\bs a\bs a^\top\right).
\end{equation*}
where $\varepsilon$ is defined in \eqref{eq:def-c-q-epsilon}.  We introduce the following linear operators on $\bbR^{d\times d}$,
\begin{equation*}
    \cB=\bbE(\hbs A_t\otimes\hbs A_t),\quad \tilde{\cB}=\bs A\otimes\bs A,
\end{equation*}
where $\bs A\otimes\bs B$ defines a linear matrix operator for matrix $\bs A$ and $\bs B$ that satisfies $(\bs A\otimes\bs B)\circ\bs C=\bs A\bs C\bs B^\top$, and $\circ$ denotes the operation of a linear matrix operator on a matrix, and. Let
\begin{equation}
    \cM=\bbE(\bs a\otimes\bs a\otimes\bs a\otimes\bs a),\quad\tilde{\cM}=\bs\Sigma\otimes\bs\Sigma.\label{eq:def-M}
\end{equation}
Then for $\bs M=\begin{pmatrix}
    \bs M_{11} & \bs M_{12} \\
    \bs M_{21} & \bs M_{22}
\end{pmatrix}\in\bbR^{2d\times 2d}$,
\begin{equation*}
    (\cB-\tilde{\cB})\circ\bs M=\begin{pmatrix}
        \eta^2(\cM-\tilde{\cM})\circ\bs M_{22} & \eta q(\cM-\tilde{\cM})\circ\bs M_{22} \\
        \eta q(\cM-\tilde{\cM})\circ\bs M_{22} & q^2(\cM-\tilde{\cM})\circ\bs M_{22}
    \end{pmatrix}.
\end{equation*}

Let
\begin{equation*}
    \bs C_t=\bbE\left(\bs\eta_t^\rmvar(\bs\eta_t^\rmvar)^\top\right),
\end{equation*}
Then we can write the update of $\bs C_t$ as follows
\begin{equation*}
    \bs C_t=\cB\circ\bs C_{t-1}+\begin{pmatrix}
        \eta^2\bs R & \eta q\bs R \\
        \eta q\bs R & q^2\bs R
    \end{pmatrix},\quad \bs C_0=\bs O.
\end{equation*}

\subsubsection{Properties of Momentum Matrix}\label{sec:property-A}
We note that $\bs A$ is a block-diagonal matrix, and the $i$-th block\footnote{To be more precise, $\bs A_i$ is the $2\times 2$ submatrix of $\bs A$ formed by the entries $(i,i)$, $(d+i,i)$, $(i,d+i)$ and $(d+i,d+i)$ of $\bs A$.} is
\begin{equation*}
    \bs A_i=\begin{pmatrix}
        0 & 1-\eta\lambda_i \\        
        c & 1+c-q\lambda_i
    \end{pmatrix},
\end{equation*}
where $\lambda_i$ is the $i$-th eigenvalue of $\bs\Sigma$. We define
\begin{equation*}
    \bs A(\lambda)=\begin{pmatrix}
        0 & 1-\eta\lambda \\        
        c & 1+c-q\lambda
    \end{pmatrix}.
\end{equation*}

\begin{lemma}\label{lemma:A-inverse}
    We have
    \begin{equation*}
        (\bs I-\bs A(\lambda))^{-1}\begin{pmatrix}
            \eta \\ q
        \end{pmatrix}=\frac{1}{\lambda}\begin{pmatrix}
            1 \\ 1
        \end{pmatrix}.
    \end{equation*}
\end{lemma}
\begin{proof}The inverse of $\bs I-\bs A(\lambda)$ is given by
    \begin{equation*}
    (\bs I-\bs A(\lambda))^{-1}=\frac{1}{(q-c\eta)\lambda}\begin{pmatrix}
        q\lambda-c & 1-\eta\lambda \\ -c & 1
    \end{pmatrix}.
    \end{equation*}
    Therefore,
    \begin{equation*}
    (\bs I-\bs A(\lambda))^{-1}\begin{pmatrix}\eta\\ q\end{pmatrix}
    =\frac{1}{\lambda(q-c\eta)}\begin{pmatrix}\eta(q\lambda-c)+q(1-\eta\lambda)\\ c\eta+q\end{pmatrix}
    =\frac{1}{\lambda}\begin{pmatrix}1\\ 1\end{pmatrix}.
    \end{equation*}
    This completes the proof.
\end{proof}

\paragraph{Spectral properties of momentum matrix.}
The eigenvalues of $\bs A(\lambda)$ are
\begin{equation*}
\begin{aligned}
    x_1(\lambda)&=\frac{1+c-q\lambda}{2}-\frac{\sqrt{(1+c-q\lambda)^2-4c(1-\eta\lambda)}}{2}, \\
    x_1(\lambda)&=\frac{1+c-q\lambda}{2}+\frac{\sqrt{(1+c-q\lambda)^2-4c(1-\eta\lambda)}}{2}.
\end{aligned}
\end{equation*}
Let $(1+c-q\lambda)^2-4c(1-\eta\lambda)<0$, we obtain the region of $\lambda$ such that $x_1(\lambda),x_2(\lambda)\in\bbC$:
\begin{equation*}
    \underbrace{\frac{(1-c)^2}{\left(\sqrt{q-c\eta}+\sqrt{c(q-\eta)}\right)^2}}_{\lambda_L}<\lambda<\underbrace{\frac{(1-c)^2}{\left(\sqrt{q-c\eta}-\sqrt{c(q-\eta)}\right)^2}}_{\lambda_U}.
\end{equation*}
We adopt Lemma E.2 from \cite{li2024risk}, which bounds $x_1(\lambda)$ and $x_2(\lambda)$.
\begin{lemma}\label{lemma:radius-bound}
    Let $\lambda\geq 0$ and $\eta\lambda<1$.
    \begin{itemize}
        \item If $\lambda\leq\lambda_L$, then $x_1(\lambda)$ and $x_2(\lambda)$ are real, and $x_1(\lambda)\leq x_2(\lambda)\leq 1-(q-c\delta)\lambda/(1-c)$.
        \item If $\lambda_L<\lambda<\lambda_U$, then $x_1(\lambda)$ and $x_2(\lambda)$ are complex, and $\| x_1(\lambda)\|=\| x_2(\lambda)\|=\sqrt{c\left(1-\delta\lambda\right)}$.
        \item If $\lambda\geq\lambda_U$, then $x_1(\lambda)$ and $x_2(\lambda)$ are real, and $x_1(\lambda)\leq x_2(\lambda)\leq c\delta/q$.
    \end{itemize}
\end{lemma}

\paragraph{Bound of power of momentum matrix. } The $k$-th power of $\bs A(\lambda)$ is
\begin{equation*}
    \bs A^k(\lambda)=\begin{pmatrix}
        -c(1-\eta\lambda)a_{k-1}(\lambda) & (1-\eta\lambda)a_k(\lambda) \\
        -ca_k(\lambda) & a_{k+1}(\lambda)
    \end{pmatrix},\quad\text{where}\quad a_k(\lambda)=\frac{x_2^k(\lambda)-x_1^k(\lambda)}{x_2(\lambda)-x_1(\lambda)},
\end{equation*}
and $a_k(\lambda)\in\bbR$.

The following lemma is modified from Lemma~20 in \citep{liu2025optimal}, which bounds $\bs A^k(\lambda)$. We present the proof for completeness.
\begin{lemma}\label{lemma:A-power-bound}
    Suppose $\eta\lambda<1$ and $\eta\leq q$, then we have
    \begin{equation*}
        \left\|\left(\bs A^{k}(\lambda)\begin{pmatrix}
            1 \\ 1
        \end{pmatrix}\right)_1\right\|\leq 2.
    \end{equation*}
\end{lemma}
\begin{proof}
    For notational simplicity, we write $a_k=a_k(\lambda)$ and $x_{1,2}=x_{1,2}(\lambda)$. We have
    \begin{equation*}
        \left(\bs A^{k}(\lambda)\begin{pmatrix}
            1 \\ 1
        \end{pmatrix}\right)_1=(1-\eta\lambda)(a_k-c a_{k-1}).
    \end{equation*}
    The result holds for $k=0$, so we need to show that $\| a_{k}-c a_{k-1}\|\leq 2$ for $k\geq 1$.
    \begin{enumerate}
        \item If $\lambda\leq\lambda_L$, by Lemma~\ref{lemma:radius-bound}, and $\delta\leq q$, we have $a_k\geq 0$, and
        \begin{equation*}
            x_1\leq x_2\leq 1-\frac{q-c\delta}{1-c}\lambda\leq 1-\delta\lambda.
        \end{equation*}
        By Vieta's formula, $x_1x_2=c(1-\delta\lambda)$, so $c\leq x_1\leq x_2$. Thus, we bound $a_k-ca_{k-1}$ as follows:
        \begin{equation*}
        \begin{aligned}
            a_k-c a_{k-1}&\geq a_k-x_1 a_{k-1}=x_2^{k-1}>0,\\
            a_k-c a_{k-1}&\leq a_k-x_1 x_2 a_{k-1}=x_2^{k-1}+(1-x_2)\sum_{i=1}^{k-1}{x_1^ix_2^{k-i-1}} \\
            &\leq x_2^{k-1}\left[1+(k-1)(1-x_2)\right]\stackrel{a}{\leq} 1,
        \end{aligned}
        \end{equation*}
        where $\stackrel{a}{\leq}$ uses Lemma~\ref{lemma:aux-lemma-r} and $k\geq 1$.
        \item If $\lambda_L<\lambda<\lambda_U$, let $x_{1,2}=r(\cos\theta\pm\mathrm{i}\sin\theta)$, where $0\leq\theta\leq\pi/2$. then we have $r=\sqrt{c(1-\delta\lambda)}\leq 1$ and $2r\cos\theta=x_1+x_2=1+c-q\lambda\geq 0$. Therefore,
        \begin{equation*}
        \begin{aligned}            
            a_k-c a_{k-1}&=\frac{r^{k-1}\sin\left(k\theta\right)}{\sin\theta}-\frac{r^{k-2}\sin\left((k-1)\theta\right)}{\sin\theta} \\
            &=r^{k-2}\left(r\cos(( k-1)\theta)+\frac{r-c}{\sin\theta}\sin((k-1)\theta)-r\tan\frac{\theta}{2}\sin(( k-1)\theta)\right),
        \end{aligned}
        \end{equation*}
        Apply triangular inequality, and $\left\|\sin k\theta\right\|\leq 1$, $\left\|\cos k\theta\right\|\leq 1$, $\left\|\tan(\theta/2)\right\|\leq 1$
        \begin{equation*}    
            \left\| a_k-c a_{k-1}\right\|\stackrel{a}{\leq}r^{k-2}\left(r+(k-1)(1-r)\right)+r^{k-1}=r^{k-2}\left(1+(k-2)(1-r)\right)+r^{k-1}\stackrel{b}{\leq}2,
        \end{equation*}
        where $\stackrel{a}{\leq}$ holds since $r^2\leq c\leq 1$ implies $\left\| r-c\right\|\leq1-r$ and $\left\|\sin(k\theta)/\sin\theta\right\|\leq k$, $\stackrel{b}{\leq}$ follows from Lemma~\ref{lemma:aux-lemma-r}, $k\geq 1$ and $0\leq r\leq 1$.
        \item If $\lambda>\lambda_U$, we have $a_k\geq 0$, and $x_1\leq x_2\leq c\delta/q\leq c$ by Lemma~\ref{lemma:radius-bound}, and $\delta\leq q$. Then we have the following bounds:
        \begin{equation*}
        \begin{aligned}
            a_k-c a_{k-1}\geq& a_k-a_{k-1}=\sum_{i=0}^{k-1}{x_1^i x_2^{k-i-1}}-\sum_{i=0}^{k-2}{x_1^ix_2^{k-i-2}} \\
            =&x_1^{k-1}-(1-x_2)\sum_{i=0}^{k-2}{x_1^ix_2^{k-i-2}} \\
            \geq&x_1^{k-1}-(k-1)(1-x_2)x_2^{k-2} \\
            \geq&-x_2^{k-2}\left(1+(k-2)x_2^{k-1}\right)\stackrel{a}{\geq} -1, \\
            a_k-c a_{k-1}\leq& a_k-x_2 a_{k-1}=x_1^{k-1}\leq 1,
        \end{aligned}
        \end{equation*}
        where $\stackrel{a}{\geq}$ if from Lemma~\ref{lemma:aux-lemma-r} and $k\geq 1$.
    \end{enumerate}
    We complete the proof by combining the above cases.
\end{proof}

\subsection{Variance Upper Bound}\label{sec:variance-bound}
\paragraph{Organization.}
In the following, we introduce $L_\rmvar(\{\bs X_t\}_{t\in\bbN})$, which acts like a loss function for covariance matrices $\{\bs X_t\}_{t\in\bbN}$. In Appendix~\ref{sec:layer-peeled-decomposition}, we state the layer-peeled decomposition technique. We introduce dynamics $\tbs C_t^{(\ell)}$ and establish the relation $\bs C_t=\tbs C_t+\sum_{\ell=1}^\infty\tbs C_t^{(\ell)}$ in Lemma~\ref{lemma:layer-peeled-decomposition}. In Appendix~\ref{sec:localize-R}, we decompose the noise matrix $\bs R$ defined in \eqref{eq:def-R}, which is the key step to establish the statistical term that depends on the noise covariance matrix at minimizer $\bs x^*$ (i.e., $\bs Q$). In Appendix~\ref{sec:core-dynamics}, we construct the core auxiliary iteration $\tbs\Pi_t(\bs M)$ in \eqref{eq:def-tilde-Pi}, where the covariance matrix is a general PSD matrix $\bs M$.
In Appendix~\ref{sec:bound-tilde-C}, we set $\bs M=\bs R$ to obtain the bounds of $\tbs C_t$. In Appendix~\ref{sec:bound-C-1}, we bound $\tbs C_t^{(\ell)}$ by recursively obtaining bounds of $\tbs C_t^{(\ell)}$. Finally, by combing the bounds in Appendix~\ref{sec:bound-tilde-C} and Appendix~\ref{sec:bound-C-1}, we obtain the variance bound by proving Lemma~\ref{lemma:v-bound} in Appendix~\ref{sec:proof-lemma-v-bound}.

The goal of this section is to prove the following lemma.
\begin{lemma}[Variance Upper Bound]\label{lemma:v-bound}
    Suppose Assumptions~\ref{assumption:l-condition}, \ref{assumption:regularity}, \ref{assumption:fourth-moment}, \ref{assumption:gradient-noise-I} and \ref{assumption:gradient-noise} hold. Let $\bs v$ defined in Definition~\ref{def:r-v} Then $\bbE\bs v=\bs 0$ and
    \begin{equation*}
        \bbE\|\bs v\|_{\bs\Sigma}^2\leq\frac{320 \left(3\tr(\bs\Sigma^{-1}\bs Q)+8\eta\tilde{\kappa}L_\ell\tr\bs Q\right)}{T}+\frac{160 L_\ell(6L+\tilde{\kappa}(7+16\eta B))(F(\tilde{\bs y})-F(\bs x^*))}{T},
    \end{equation*}
    where the expectation is taken with respect to the samples drawn in the $k$-th outer iteration.
\end{lemma}
The proof is deferred to Appendix~\ref{sec:proof-lemma-v-bound}.

To simplify notation, we introduce the following mapping $L_\rmvar$ on a sequence of PSD matrices.
\begin{definition}\label{def:L-var}
    For a PSD matrices sequence $\{\bs X_t\in\bbR^{2d\times2d}\}_{t\in\bbN}$, we define
    \begin{equation}
        L_\rmvar(\{\bs X_t\}_{t\in\bbN})\defeq\frac{4}{T^2}\left\langle\begin{pmatrix}
                \bs\Sigma & \bs O \\
                \bs O & \bs O
            \end{pmatrix},\sum_{s,t=T/2+1}^{T}\bs A^{\max\{s-t,0\}}\bs X_{\min\{s,t\}}\left(\bs A^{\max\{t-s,0\}}\right)^\top\right\rangle.\label{eq:def-L-var}
    \end{equation}
\end{definition}

The mapping $L_\rmvar$ is linear and monotonic in the following sense.
\begin{lemma}\label{lemma:L-var-linearity}
    For PSD matrices sequence $\{\bs X_t\}_{t\in\bbN}$ and $\{\bs T_t\}_{t\in\bbN}$, we have
    \begin{equation}
        L_\rmvar(\{\alpha\bs X_t+\beta\bs Y_t\}_{t\in\bbN})=\alpha L_\rmvar(\{\bs X_t\}_{t\in\bbN})+\beta L_\rmvar(\{\bs Y_t\}_{t\in\bbN}).\label{eq:linear-L-var}
    \end{equation}
\end{lemma}
\begin{lemma}
    For PSD matrices sequence $\{\bs X_t\}_{t\in\bbN}$ and $\{\bs Y_t\}_{t\in\bbN}$ such that $\bs X_t\preceq\bs Y_t$ for all $t\in\bbN$, then we have
    \begin{equation}
        L_\rmvar(\{\bs X_t\}_{t\in\bbN})\leq L_\rmvar(\{\bs Y_t\}_{t\in\bbN}).\label{eq:increasing-L-var}
    \end{equation}
\end{lemma}
The proof follows directly from the definition of $L_\rmvar$. We can represent $\bbE\|\bs v\|_{\bs\Sigma}^2$ by $L_\rmvar$ as follows.

\begin{lemma}\label{lemma:v-L-var}
    We have
    \begin{equation*}
        \bbE\|\bs v\|_{\bs\Sigma}^2=L_\rmvar(\{\bs C_t\}_{t\in\bbN}).
    \end{equation*}
\end{lemma}
\begin{proof}
    By the definition of $\bs v$, we have
    \begin{equation*}
    \begin{aligned}
        \bbE\|\bs v\|_{\bs\Sigma}^2&=\left\langle\begin{pmatrix}
            \bs\Sigma & \bs O \\
            \bs O & \bs O
        \end{pmatrix},\bbE\left[\left(\frac{2}{T}\sum_{t=T/2+1}^{T}\bs\eta_t^\rmvar\right)\left(\frac{2}{T}\sum_{t=T/2+1}^{T}\bs\eta_t^\rmvar\right)^\top\right]\right\rangle \\
        &=\left\langle\begin{pmatrix}
            \bs\Sigma & \bs O \\
            \bs O & \bs O
        \end{pmatrix},\frac{4}{T^2}\sum_{s,t=T/2+1}^{T}\bbE(\bs\eta_s^\rmvar\bs\eta_t^\rmvar)\right\rangle, \\
        &\stackrel{a}{=}\left\langle\begin{pmatrix}
            \bs\Sigma & \bs O \\
            \bs O & \bs O
        \end{pmatrix},\frac{4}{T^2}\sum_{s,t=T/2+1}^{T}\bs A^{\max\{s-t,0\}}\bs C_{\min\{s,t\}}\left(\bs A^{\max\{t-s,0\}}\right)^\top\right\rangle,
    \end{aligned}
    \end{equation*}
    where $\stackrel{a}{=}$ uses
        \begin{equation*}        \bbE\bs\eta_s^\rmvar(\bs\eta_t^\rmvar)^\top=\bbE\left(\bbE\left(\bs\eta_s^\rmvar|\cF_t\right)(\bs\eta_t^\rmvar)^\top\right)=\bs A^{s-t}\bbE\bs\eta_t^\rmvar(\bs\eta_t^\rmvar)^\top=\bs A^{s-t}\bs C_t,\quad s\geq t
    \end{equation*}
    and collecting the terms.
\end{proof}

The quantity $L_\rmvar(\{\bs X_t\}_{t\in\bbN})$ can be interpreted as the loss incurred by the tail-averaged sequence with covariance $\bs X_t$. Specifically, we construct the dynamics $\bs\psi_t\in\bbR^{2d\times 2d}$ as follows:
\begin{equation*}
    \bs\psi_t=\bs A\bs\psi_{t-1}+\begin{pmatrix}
        \eta\bs\xi_t \\ q\bs\xi_t
    \end{pmatrix},\quad \bs\psi_0=\bs 0,
\end{equation*}
where $\bs\xi_t\sim\text{i.i.d.}\ \cN(\bs 0,\bs\Sigma_{\bs\xi})$ and $\bbE\bs\psi_t\bs\psi_t^\top=\bs X_t$, then we have (the proof is identical to the proof of Lemma~\ref{lemma:v-L-var})
\begin{equation*}
    L_\rmvar(\{\bs X_t\}_{t\in\bbN})=\left\langle\begin{pmatrix}
            \bs\Sigma & \bs O \\
            \bs O & \bs O
        \end{pmatrix},\bbE\left[\left(\frac{2}{T}\sum_{t=T/2+1}^{T}\bs\psi_t\right)\left(\frac{2}{T}\sum_{t=T/2+1}^{T}\bs\psi_t\right)^\top\right]\right\rangle.
\end{equation*}
However, $L_\rmvar(\{\bs X_t\}_{t\in\bbN})$ can be calculated for any sequence $\{\bs X_t\}_{t\in\bbN}$ and does not necessarily require $\bs X_t$ to be the covariance of some dynamics $\bs\psi_t$.

\subsubsection{Layer-peeled Decomposition}\label{sec:layer-peeled-decomposition}
The goal of this section is to bound $\bbE\|\bs v\|_{\bs\Sigma}^2=L_\rmvar(\{\bs C_t\}_{t\in\bbN})$. To tightly characterize $\bs C_t$, this paper proposes the \textbf{layer-peeled decomposition}, which decomposes the covariance matrix $\bs C_t$ into the sum of the following dynamics.
\begin{definition}[Layer $0$ Dynamics]
    Let
    \begin{equation}
        \tbs C_t=\bs A\tbs C_{t-1}\bs A^\top+\begin{pmatrix}
            \eta^2\bs R & \eta q\bs R \\
            \eta q\bs R & q^2\bs R
        \end{pmatrix},\quad\tbs C_0=\bs O,\label{eq:def-var-tilde-C}
    \end{equation}
\end{definition}
\begin{definition}[Layer $\ell$ Dynamics]
    Let
    \begin{equation}
        \tbs C_t^{(\ell)}=\bs A\tbs C_{t-1}^{(\ell)}\bs A^\top+(\cB-\tilde{\cB})\circ\tbs C_{t-1}^{(\ell-1)},\quad\tbs C_0^{(\ell)}=\bs O.\label{eq:def-var-tilde-C-l}
    \end{equation}
\end{definition}
We call \eqref{eq:def-var-tilde-C} the layer $0$ dynamics and call \eqref{eq:def-var-tilde-C-l} layer $\ell$ dynamics. The following lemma shows that $\bs C_t=\tbs C_t+\sum_{\ell=1}^\infty\tbs C_t^{(\ell)}$.

\begin{lemma}\label{lemma:layer-peeled-decomposition}
    We have
    \begin{equation*}
        \bs C_t=\tbs C_t+\sum_{\ell=1}^\infty\tbs C_t^{(\ell)}.
    \end{equation*}
    We also have $\tbs C_t^{(\ell)}=\bs O$ for $\ell\geq t$, so the above display is a finite sum.
\end{lemma}
\begin{proof}
    We construct a series of dynamics (indexed by $\ell$) as follows:
    \begin{equation}
        \bs C_t^{(\ell)}=\cB\circ\bs C_{t-1}^{(\ell)}+(\cB-\tilde{\cB})\circ\tbs C_{t-1}^{(\ell-1)},\quad\bs C_0^{(\ell)}=\bs O.\label{eq:def-var-C-l}
    \end{equation}    
    Recall that $\tilde{\cB}=\bs A\otimes\bs A$, so $\bs C_t=\tbs C_t+\bs C_t^{(1)}$, $\bs C_t^{(\ell)}=\tbs C_t^{(\ell)}+\bs C_t^{(\ell+1)}$, and
    \begin{equation*}
        \bs C_t^{(\ell)}=\bs O,\quad\tbs C_t^{(\ell)}=\bs O,\quad\text{for all $0\leq t\leq\ell$},
    \end{equation*}
    so we have the following decomposition:
    \begin{equation}
    \begin{aligned}
        \bs C_t^{(1)}&=\tbs C_t^{(1)}+\bs C_t^{(2)} \\
        &=\tbs C_t^{(1)}+\tbs C_t^{(2)}+\bs C_t^{(3)} \\
        &\cdots \\
        &=\sum_{\ell=1}^{\infty}\tbs C_t^{(\ell)}.\quad\quad\text{(note that $\tbs C_t^{(\ell)}=\bs O$ for all $\ell\geq t$})\label{eq:decomposition-C-1}
    \end{aligned}
    \end{equation}
    This completes the proof.
\end{proof}

\subsubsection{Core Auxilary Dynamics}\label{sec:core-dynamics}
In this section, we analyze the following dynamics $\tbs\Pi_t(\bs M)\in\bbR^{2d\times 2d}$,
\begin{equation}
    \tbs\Pi_t(\bs M)=\bs A\tbs\Pi_{t-1}(\bs M)\bs A^\top+\begin{pmatrix}
        \eta^2\bs M & \eta q\bs M \\
        \eta q\bs M & q^2\bs M
    \end{pmatrix},\quad \tbs\Pi_0=\bs O,\label{eq:def-tilde-Pi}
\end{equation}
where $\bs M$ is any PSD matrix. We also write $\tbs\Pi_t$ for notational simplicity. We focus on the update of the diagonal\footnote{To be more precise, $\tbs\Pi_{t,i}$ is the $2\times 2$ submatrix of $\tbs\Pi_{t,i}$ formed by the entries $(i,i)$, $(d+i,i)$, $(i,d+i)$ and $(d+i,d+i)$ of $\tbs\Pi_t$.} of $\tbs\Pi_t$, which is
\begin{equation*}
    \tbs\Pi_{t,i}=\bs A_i\tbs\Pi_{t-1,i}\bs A_i^\top+m_i\begin{pmatrix}
        \eta^2 & \eta q \\
        \eta q & q^2
    \end{pmatrix},
\end{equation*}
where $m_i$ denotes the $i$-th diagonal element of $\bs M$.

\paragraph{Upper Bound of $L_\rmvar(\{\tbs\Pi_t(\bs M)\}_{t\in\bbN})$.}
We bound $L_\rmvar(\{\tbs\Pi_t(\bs M)\}_{t\in\bbN})$ from above. The first step is to derive a matrix form.
\begin{lemma}\label{lemma:L-var-matrix-form}
    We have $L_\rmvar(\{\tbs\Pi_t(\bs M)\}_{t\in\bbN})=4(A_{T/2}+B_{T/2})/T^2$, where
    \begin{equation*}
    \begin{aligned}
        A_m&=\left\langle\begin{pmatrix}
            \bs\Sigma & \bs O \\
            \bs O & \bs O
        \end{pmatrix},\left(\sum_{k=0}^{T-m-1}\bs A^k\right)\tbs\Pi_{m+1}\left(\sum_{k=0}^{T-m-1}\bs A^k\right)^\top\right\rangle, \\
        B_m&=\sum_{s=1}^{T-m-1}\left\langle\begin{pmatrix}
            \bs\Sigma & \bs O \\
            \bs O & \bs O
        \end{pmatrix},\left(\sum_{k=0}^{s-1}\bs A^k\right)\begin{pmatrix}
            \eta^2\bs M & \eta q\bs M \\
            \eta q\bs M & q^2\bs M
        \end{pmatrix}\left(\sum_{k=0}^{s-1}\bs A^k\right)^\top\right\rangle.
    \end{aligned}
    \end{equation*}
\end{lemma}
\begin{proof}
    We show by induction that
    \begin{equation*}
        L_m\defeq\left\langle\begin{pmatrix}
            \bs\Sigma & \bs O \\
            \bs O & \bs O
        \end{pmatrix},\sum_{s,t=m+1}^{T}\bs A^{\max\{s-t,0\}}\tbs\Pi_{\min\{s,t\}}\left(\bs A^{\max\{t-s,0\}}\right)^\top\right\rangle=A_m+B_m.
    \end{equation*}
    The above display holds for $m=T-1$. Assuming the above display holds for $m\leq T-1$, we show that it holds for $m-1$. First, note that
    \begin{equation}
    \begin{aligned}        
        L_{m-1}-L_m&=\left\langle\begin{pmatrix}
            \bs\Sigma & \bs O \\
            \bs O & \bs O
        \end{pmatrix},\tbs\Pi_{m}\right\rangle+\left\langle\begin{pmatrix}
            \bs\Sigma & \bs O \\
            \bs O & \bs O
        \end{pmatrix},\left(\sum_{k=1}^{T-m}\bs A^k\right)\tbs\Pi_{m}\right\rangle \\
        &\phrel+\left\langle\begin{pmatrix}
            \bs\Sigma & \bs O \\
            \bs O & \bs O
        \end{pmatrix},\tbs\Pi_{m}\left(\sum_{k=1}^{T-m}\bs A^k\right)^\top\right\rangle.\label{eq:L-diff}
    \end{aligned}
    \end{equation}
    Therefore, we have
    \begin{equation*}
    \begin{aligned}
        &\phrel A_{m-1}-A_m \\
        &=\left\langle\begin{pmatrix}
            \bs\Sigma & \bs O \\
            \bs O & \bs O
        \end{pmatrix},\left(\bs A\sum_{k=0}^{T-m-1}\bs A^k+\bs I\right)\tbs\Pi_{m}\left(\bs A\sum_{k=0}^{T-m-1}\bs A^k+\bs I\right)^\top\right\rangle \\
        &\phrel-\left\langle\begin{pmatrix}
            \bs\Sigma & \bs O \\
            \bs O & \bs O
        \end{pmatrix},\left(\sum_{k=0}^{T-m-1}\bs A^k\right)\tbs\Pi_{m+1}\left(\sum_{k=0}^{T-m-1}\bs A^k\right)^\top\right\rangle \\
        &\stackrel{a}{=}(L_{m-1}-L_m)+\left\langle\begin{pmatrix}
            \bs\Sigma & \bs O \\
            \bs O & \bs O
        \end{pmatrix},\left(\sum_{k=0}^{T-m-1}\bs A^k\right)\left(\tbs\Pi_{m+1}-\bs A\tbs\Pi_m\bs A^\top\right)\left(\sum_{k=0}^{T-m-1}\bs A^k\right)^\top\right\rangle \\
        &\stackrel{b}{=}(L_{m-1}-L_m)-\left\langle\begin{pmatrix}
            \bs\Sigma & \bs O \\
            \bs O & \bs O
        \end{pmatrix},\left(\sum_{k=0}^{T-m-1}\bs A^k\right)\begin{pmatrix}
            \eta^2\bs M & \eta q\bs M \\
            \eta q\bs M & q^2\bs M
        \end{pmatrix}\left(\sum_{k=0}^{T-m-1}\bs A^k\right)^\top\right\rangle \\
        &=(L_{m-1}-L_m)-(B_{m-1}-B_{m}),
    \end{aligned}
    \end{equation*}
    where we expand the product and apply \eqref{eq:L-diff} in $\stackrel{a}{=}$, and $\stackrel{b}{=}$ applies the update rule of $\tbs\Pi_t$ in \eqref{eq:def-tilde-Pi}. Thus, $A_{m-1}+B_{m-1}=A_m+B_m+(L_{m-1}-L_m)=L_{m-1}$.
\end{proof}

We further derive a dimension-wise summation form of $L_\rmvar(\{\tbs\Pi_t(\bs M)\}_{t\in\bbN})$.
\begin{lemma}\label{lemma:L-var}
    Let $m_i$ denote the $i$-th element of the diagonal of $\bs M$, then
    \begin{equation*}
    \begin{aligned}
        L_\rmvar(\{\tbs\Pi_t(\bs M)\}_{t\in\bbN})&=\frac{4}{T^2}\sum_{i=1}^d\sum_{k=0}^{T/2}\lambda_i m_i\left(\bs A_i^k\left(\bs I-\bs A_i^{T/2}\right)(\bs I-\bs A_i)^{-1}\begin{pmatrix}
            \eta \\ q
        \end{pmatrix}\right)_1^2 \\
        &\phrel+\frac{4}{T^2}\sum_{i=1}^d\sum_{k=1}^{T/2}\lambda_i m_i\left(\left(\bs I-\bs A_i^{T/2-k}\right)(\bs I-\bs A_i)^{-1}\begin{pmatrix}
            \eta \\ q
        \end{pmatrix}\right)_1^2.        
    \end{aligned}
    \end{equation*}
\end{lemma}
\begin{proof}
    By Lemma~\ref{lemma:L-var-matrix-form}, we have $L_\rmvar(\{\tbs\Pi_t(\bs M)\}_{t\in\bbN})=4(A_{T/2}+B_{T/2})/T^2$, where
    \begin{equation*}
    \begin{aligned}
        A_{T/2}&=\left\langle\begin{pmatrix}
                \bs\Sigma & \bs O \\
                \bs O & \bs O
            \end{pmatrix},\left(\sum_{k=0}^{T/2-1}\bs A^k\right)\tbs\Pi_{T/2+1}\left(\sum_{k=0}^{T/2-1}\bs A^k\right)^\top\right\rangle, \\
        B_{T/2}&=\sum_{k=1}^{T/2}\left\langle\begin{pmatrix}
                \bs\Sigma & \bs O \\
                \bs O & \bs O
            \end{pmatrix},\left(\sum_{k=0}^{T/2-k-1}\bs A^k\right)\begin{pmatrix}
                \eta^2\bs M & \eta q\bs M \\
                \eta q\bs M & q^2\bs M
            \end{pmatrix}\left(\sum_{k=0}^{T/2-k-1}\bs A^k\right)^\top\right\rangle.
    \end{aligned}
    \end{equation*}
    Note that the matrices in $B_{T/2}$ are block-diagonal. Thus,
    \begin{equation*}
        B_{T/2}=\sum_{i=1}^d\sum_{k=1}^{T/2}\lambda_i m_i\left(\left(\bs I-\bs A_i^{T/2-k}\right)(\bs I-\bs A_i)^{-1}\begin{pmatrix}
            \eta \\ q
        \end{pmatrix}\right)_1^2.
    \end{equation*}
    By the definition of $\tbs\Pi$ in \eqref{eq:def-tilde-Pi}, we have
    \begin{equation*}
        \tbs\Pi_{T/2+1}=\sum_{k=0}^{T/2}\bs A^k\begin{pmatrix}
                \eta^2\bs M & \eta q\bs M \\
                \eta q\bs M & q^2\bs M
            \end{pmatrix}\left(\bs A^k\right)^\top.
    \end{equation*}
    Thus, $A_{T/2}$ has the following form:
    \begin{equation*}
        A_{T/2}=\sum_{i=1}^d\sum_{k=0}^{T/2}\lambda_i m_i\left(\bs A_i^k\left(\bs I-\bs A_i^{T/2}\right)(\bs I-\bs A_i)^{-1}\begin{pmatrix}
            \eta \\ q
        \end{pmatrix}\right)_1^2,
    \end{equation*}    
    where $(\cdot)_1$ denotes the first element of a $2$-dimensional vector. The result follows by applying $L_\rmvar(\{\tbs\Pi_t(\bs M)\}_{t\in\bbN})=4(A_{T/2}+B_{T/2})/T^2$.
\end{proof}

The following lemma bounds $L_\rmvar(\{\tbs\Pi_t(\bs M)\}_{t\in\bbN})$.
\begin{lemma}\label{lemma:Pi-average-bound}
    Suppose we choose the hyperparameters as specified in Appendix~\ref{sec:hyperparameter-choice}, then we have
    \begin{equation*}
        L_\rmvar(\{\tbs\Pi_t(\bs M)\}_{t\in\bbN})\leq\frac{64\tr(\bs\Sigma^{-1}\bs M)}{T}.
    \end{equation*}   
\end{lemma}
\begin{proof}
    By Lemma~\ref{lemma:A-inverse} and Lemma~\ref{lemma:A-power-bound}, we have
    \begin{equation*}
        A_{T/2}=\sum_{i=1}^d\sum_{k=0}^{T/2}\frac{m_i}{\lambda_i}\left(\bs A_i^k\left(\bs I-\bs A_i^{T/2}\right)\begin{pmatrix}
            1 \\ 1
        \end{pmatrix}\right)_1^2\leq\sum_{i=1}^d\sum_{k=0}^{T/2}\frac{16m_i}{\lambda_i}=\frac{8\tr(\bs\Sigma^{-1}\bs M)}{T},
    \end{equation*}
    \begin{equation*}        
        B_{T/2}=\sum_{i=1}^d\sum_{k=1}^{T/2}\lambda_i m_i\left(\left(\bs I-\bs A_i^{T/2-k}\right)\begin{pmatrix}
            1 \\ 1
        \end{pmatrix}\right)_1^2\leq\sum_{i=1}^d\sum_{k=0}^{T/2}\frac{16m_i}{\lambda_i}=\frac{8\tr(\bs\Sigma^{-1}\bs M)}{T}.
    \end{equation*}
    We complete the proof by $L_\rmvar(\{\tbs\Pi_t(\bs M)\}_{t\in\bbN})=4(A_{T/2}+B_{T/2})/T^2\leq 64\tr(\bs\Sigma^{-1}\bs M)/T$.
\end{proof}

\paragraph{Derivation of Stationary Covariance.}
We derive the stationary state of $\tbs\Pi_t(\bs M)$. The following lemma shows that $\tbs\Pi_t(\bs M)$ is increasing.
\begin{lemma}
    Let $\tbs\Pi_t(\bs M)$ be defined in \eqref{eq:def-tilde-Pi}. Then
    \begin{equation*}
        \tbs\Pi_0(\bs M)\preceq\tbs\Pi_1(\bs M)\preceq\cdots\preceq\tbs\Pi_t(\bs M)\preceq\tbs\Pi_{t+1}(\bs M)\preceq\cdots
    \end{equation*}
\end{lemma}
\begin{proof}
    We expand the recursion of $\tbs\Pi_t$ to obtain
    \begin{equation*}
        \tbs\Pi_t(\bs M)=\sum_{k=0}^{t-1}\bs A^k\begin{pmatrix}
            \eta^2\bs M & \eta q\bs M \\
            \eta q\bs M & q^2\bs M
        \end{pmatrix}\left(\bs A^k\right)^\top.
    \end{equation*}
    Since it is a summation of PSD matrices, the conclusion follows.
\end{proof}

We define the stationary state of $\tbs\Pi$ as follows:
\begin{equation}
    \tbs\Pi_\infty(\bs M)=\sum_{k=0}^{\infty}\bs A^k\begin{pmatrix}
        \eta^2\bs M & \eta q\bs M \\
        \eta q\bs M & q^2\bs M
    \end{pmatrix}\left(\bs A^k\right)^\top,\label{eq:stationary-tilde-Pi}
\end{equation}
so $\tbs\Pi_t(\bs M)\preceq\tbs\Pi_\infty(\bs M)$ for all $t\in\bbN$. Note that
\begin{equation}
    \tbs\Pi_\infty(\bs M)=\bs A\tbs\Pi_\infty(\bs M)\bs A^\top+\begin{pmatrix}
        \eta^2\bs M & \eta q\bs M \\
        \eta q\bs M & q^2\bs M
    \end{pmatrix}\label{eq:Pi-recursion}.
\end{equation}
The following lemma explicitly calculates $\tbs\Pi_\infty(\bs M)$.
\begin{lemma}
    Let $m_i$ denote the $i$-th  diagonal element of $\bs M$, then
    \begin{equation*}
        \tbs\Pi_{\infty,i}(\bs M)=m_i\begin{pmatrix}
            * & * \\
            * & \dfrac{(1+c)(q-c\eta)+c\eta(q+c\eta)\lambda_i}{(2+2c-(q+c\eta)\lambda_i)(1-c+c\eta\lambda_i)\lambda_i}
        \end{pmatrix}
    \end{equation*}
\end{lemma}
\begin{proof}
    We note that \eqref{eq:Pi-recursion} is a system of linear equations for $\bs M$. The desired result follows from direct calculations.
\end{proof}

The following bounds $\left\langle\begin{pmatrix}
    \bs O & \bs O \\
    \bs O & \bs\Sigma
\end{pmatrix},\tbs\Pi_\infty(\bs M)\right\rangle$.
\begin{lemma}\label{lemma:Pi-stationary-bound}
    Suppose we choose the hyperparamter as specified in Appendix~\ref{sec:hyperparameter-choice}, then we have
    \begin{equation*}
        \left\langle\begin{pmatrix}
            \bs O & \bs O \\
            \bs O & \bs\Sigma
        \end{pmatrix},\tbs\Pi_\infty(\bs M)\right\rangle\leq\frac{\tr(\bs\Sigma^{-1}\bs M)}{4\tilde{\kappa}}+4\eta\tr\bs M.
    \end{equation*}
\end{lemma}
\begin{proof}
    We have
    \begin{equation*}
    \begin{aligned}
        \left\langle\begin{pmatrix}
            \bs O & \bs O \\
            \bs O & \bs\Sigma
        \end{pmatrix},\tbs\Pi_\infty(\bs M)\right\rangle&\leq\sum_{i=1}^d\lambda_i\left(\tbs\Pi_\infty(\bs M)\right)_{22} \\
        &\leq\sum_{i=1}^d\lambda_i m_i\cdot\dfrac{(1+c)(q-c\eta)+c\eta(q+c\eta)\lambda_i}{(2+2c-(q+c\eta)\lambda_i)(1-c+c\eta\lambda_i)\lambda_i}.
    \end{aligned}
    \end{equation*}
    Since $\eta\geq\theta\gamma$, and
    \begin{equation*}
        q=\frac{\eta+\theta\gamma}{1+\theta}\leq\frac{2\eta}{1+\theta}=2c\eta,
    \end{equation*}
    so we have
    \begin{equation*}
        2+2c-(q+c\eta)\lambda_i\geq 2+2c-(2+c)\eta\lambda_i\geq 1.
    \end{equation*}
    For the first term,
    \begin{equation*}
    \begin{aligned}
        \frac{(1+c)(q-c\eta)}{(1-c+c\eta\lambda_i)\lambda_i}&=\frac{1+c}{(1-c+c\eta\lambda_i)\lambda_i}\left(q-\frac{(1+c)\eta}{2}+\frac{(1-c)\eta}{2}\right) \\
        &\leq\frac{1+c}{c\eta\lambda_i^2}\left(q-\frac{(1+c)\eta}{2}\right)+\frac{(1+c)\eta}{2\lambda_i} \\
        &=\frac{2\theta\gamma}{(1-\theta^2)\eta\lambda_i^2}+\frac{\eta}{(1+\theta)\lambda_i}\stackrel{a}{=}\frac{1}{8(1-\theta^2)\tilde{\kappa}\lambda_i^2}+\frac{\eta}{(1+\theta)\lambda_i} \\
        &\stackrel{b}{\leq}\frac{1}{4\tilde{\kappa}\lambda_i}+\frac{\eta}{\lambda_i}
    \end{aligned}
    \end{equation*}
    where $\stackrel{a}{=}$ uses $\eta=16\tilde{\kappa}\theta\gamma$, and $\stackrel{b}{\leq}$ uses $1-\theta^2\geq1/2$. For the second term,
    \begin{equation*}
    \begin{aligned}
        \frac{c\eta(q+c\eta)\lambda_i}{(1-c+c\eta\lambda_i)\lambda_i}\stackrel{a}{\leq}\frac{q+c\eta}{\lambda_i}=\frac{(3+\theta)\eta}{(1+\theta)\lambda_i}\leq\frac{3\eta}{\lambda_i},
    \end{aligned}
    \end{equation*}
    where $\stackrel{a}{\leq}$ uses $c\eta\lambda_i\leq 1-c+c\eta\lambda_i$. Combining the results, we obtain
    \begin{equation*}
        \left\langle\begin{pmatrix}
            \bs O & \bs O \\
            \bs O & \bs\Sigma
        \end{pmatrix},\tbs\Pi_\infty(\bs M)\right\rangle\leq\sum_{i=1}^d \frac{m_i}{4\tilde{\kappa}\lambda_i}+4\eta m_i=\frac{\tr(\bs\Sigma^{-1}\bs M)}{4\tilde{\kappa}}+4\eta\tr\bs M.
    \end{equation*}
    This completes the proof.
\end{proof}

\subsubsection{Decomposition of Noise Covariance}\label{sec:localize-R}

Define
\begin{equation}
    \bs R=\bbE_{\bs a,b\sim\cD}\left(\varepsilon^2\bs a\bs a^\top\right).\label{eq:def-R}
\end{equation}
where $\varepsilon$ is defined in \eqref{eq:def-c-q-epsilon}. In this subsection, we decompose the $\bs\Sigma^{-1}$-norm of the stochastic gradient, which is $\tr(\bs\Sigma^{-1}\bs R)$, into a term $\tr(\bs\Sigma^{-1}\bs Q)$ corresponding to the noise at $\bs x^*$, and an additional term proportional to excess risk $F(\tbs y)-F(\bs x^*)$ in Lemma~\ref{lemma:noise-bound-Sigma-inverse}. We derive a similar decomposition of the $2$-norm of the stochastic gradient $\tr(\bs R)$ in Lemma~\ref{lemma:noise-bound-2-norm}. These results are key to deriving statistical complexity that is localized to $\bs x^*$. i.e., obtaining statistical term that only depends on $\bs Q$, which is the covariance of the stochastic gradient at $\bs x^*$.

We need the following lemma to bound the norm of the gradient of $F$.
\begin{lemma}\label{lemma:gradient-bound}
    Suppose Assumption~\ref{assumption:l-condition} holds. Then for any $\bs x\in\bbR^d$, we have
    \begin{equation*}
        \|\nabla F(\bs x)\|_{\bs\Sigma^{-1}}^2\leq 2 L_\ell(F(\bs x)-F(\bs x^*)),
    \end{equation*}
    where $\bs x^*$ is a minimizer of $F$.
\end{lemma}
\begin{proof}
    Let $\bs u=\bs\Sigma^{-1}\nabla F(\bs x)/L_\ell$. Since $\nabla^2 F\preceq L_\ell\bs\Sigma$, so
    \begin{equation*}
        F(\bs x^*)\leq F(\bs x-\bs u)\leq F(\bs x)-\langle\nabla F(\bs x),\bs u\rangle+\frac{1}{2}\|\bs u\|_{L_\ell\bs\Sigma}^2=F(\bs x)-\frac{1}{2 L_\ell}\|\nabla F(\bs x)\|_{\bs\Sigma^{-1}}^2.
    \end{equation*}
    The result follows by rearranging the terms.
\end{proof}

With Lemma~\ref{lemma:gradient-bound}, we can bound $\tr(\bs\Sigma^{-1}\bs R)$ as follows.
\begin{lemma}[Noise Upper Bound]\label{lemma:noise-bound-Sigma-inverse}
    Suppose Assumptions~\ref{assumption:fourth-moment} and \ref{assumption:gradient-noise} hold. Then
    \begin{equation*}
        \tr(\bs\Sigma^{-1}\bs R)\leq 5\tr(\bs\Sigma^{-1}\bs Q)+5L_\ell(L+\tilde{\kappa})(F(\tbs y)-F(\bs x^*)).
    \end{equation*}
\end{lemma}
\begin{proof}
    Note that by \eqref{eq:def-c-q-epsilon}, we have
    \begin{equation}
    \begin{aligned}
        \varepsilon^2\bs a\bs a^\top&\preceq 5\left(\ell'(\bs a^\top\bs x^*,b)\right)^2\bs a\bs a^\top+\frac{5}{2}\left(\ell'(\bs a^\top\tbs y,b)-\ell'(\bs a^\top\bs x^*,b)\right)^2\bs a\bs a^\top \\
        &\phrel+\frac{5}{2}\left(\bs a^\top(\tbs y-\tbs x_+^*)\right)^2\bs a\bs a^\top,\label{eq:var-noise-bound}
    \end{aligned}
    \end{equation}
    where we use $(a+b+c)^2\leq 5a^2+5b^2/2+5c^2/2$. Taking inner product with $\bs\Sigma^{-1}$ and take expectation, the first term is bounded by $\tr(\bs\Sigma^{-1}\bs Q)$, and the second term is bounded by $2L_\ell L(F(\tbs y)-F(\bs x^*))$ by Assumption~\ref{assumption:gradient-noise}, since $\ell'(\bs a^\top\tbs y,b)\bs a$ is the stochastic gradient of $F(\tbs y)$. We bound the third term as follows
    \begin{equation*}
        \bbE\left(\bs a^\top(\tbs y-\tbs x_+^*)\right)^2\|\bs a\|_{\bs\Sigma^{-1}}^2\stackrel{a}{\leq}\tilde{\kappa}\|\tbs y-\tbs x_+^*\|_{\bs\Sigma}^2=\tilde{\kappa}\|\nabla F(\tbs y)\|_{\bs\Sigma^{-1}}^2\stackrel{b}{\leq} 2L_\ell\tilde{\kappa}(F(\tbs y)-F(\bs x^*)),
    \end{equation*}
    where $\stackrel{a}{\leq}$ uses Assumption~\ref{assumption:fourth-moment}, and $\stackrel{b}{\leq}$ uses Lemma~\ref{lemma:gradient-bound}. The result follows by combining the three bounds.
\end{proof}
We also need the upper bound of $\tr\bs R$. The proof is similar to the proof of Lemma~\ref{lemma:noise-bound-Sigma-inverse}.
\begin{lemma}\label{lemma:noise-bound-2-norm}
    Suppose Assumptions~\ref{assumption:fourth-moment} and \ref{assumption:gradient-noise-I} hold. Then
    \begin{equation*}
        \tr\bs R\leq 5\tr\bs Q+5L_\ell(B+R^2)(F(\tbs y)-F(\bs x^*)).
    \end{equation*}
\end{lemma}
\begin{proof}
    We bound the trace of \eqref{eq:var-noise-bound}. The first term is bound by $\tr\bs Q$. The second term is bounded by
    \begin{equation*}
        \bbE\left(\ell'(\bs a^\top\tbs y,b)-\ell'(\bs a^\top\bs x^*,b)\right)^2\|\bs a\|^2\leq 2L_\ell B(F(\tbs y)-F(\bs x^*)),
    \end{equation*}
    where we use Assumption~\ref{assumption:gradient-noise-I}. The third term is bounded by
    \begin{equation*}
        R^2\|\tbs y-\tbs x_+^*\|_{\bs\Sigma}^2=R^2\|\nabla F(\tbs y)\|_{\bs\Sigma^{-1}}^2\stackrel{a}{\leq}2L_\ell R^2(F(\tbs y)-F(\bs x^*)),
    \end{equation*}
    where we use Assumption~\ref{assumption:fourth-moment}. The result follows by combining the three bounds.
\end{proof}

\subsubsection{\texorpdfstring{Bound of Layer $0$ Dynamics $\tbs C_t$}{Bound of Layer 0 Dynamics}}\label{sec:bound-tilde-C}
By the definition of $\tbs\Pi$ in \eqref{eq:def-tilde-Pi}, we have
\begin{equation*}
    \tbs C_t=\tbs\Pi_t(\bs R).
\end{equation*}
Therefore, by applying Lemma~\ref{lemma:Pi-average-bound} and Lemma~\ref{lemma:Pi-stationary-bound} in Appendix~\ref{sec:core-dynamics}, we have the following lemmas.
\begin{lemma}\label{lemma:L-var-C-tilde}
    Suppose we choose the hyperparameters as specified in Appendix~\ref{sec:hyperparameter-choice}, then we have
    \begin{equation*}
        L_\rmvar\left(\left\{\tbs C_t\right\}_{t\in\bbN}\right)\leq\frac{320\tr(\bs\Sigma^{-1}\bs Q)}{T}+\frac{320(L+\tilde{\kappa})(F(\tbs y)-F(\bs x^*))}{T}.
    \end{equation*}
\end{lemma}
\begin{lemma}\label{lemma:var-tilde-C-bound}
    Suppose we choose the hyperparameters as specified in Appendix~\ref{sec:hyperparameter-choice}, then we have
    \begin{equation*}
    \begin{aligned}
        \left\langle\begin{pmatrix}
            \bs O & \bs O \\
            \bs O & \bs\Sigma
        \end{pmatrix},\tbs C_\infty\right\rangle&\leq\frac{5\tr(\bs\Sigma^{-1}\bs Q)+5(L+\tilde{\kappa})(F(\tbs y)-F(\bs x^*))}{4\tilde{\kappa}} \\
        &\phrel+20\eta\tr\bs Q+20\eta(B+R^2)(F(\tbs y)-F(\bs x^*))
    \end{aligned}
    \end{equation*}
\end{lemma}

\subsubsection{\texorpdfstring{Bound of Layer $\ell$ Dynamics $\tbs C_t^{(\ell)}$}{Bound of Layer l Dynamics}}\label{sec:bound-C-1}

The following lemma derives a decomposition of $\bbE\|\bs v_k\|_{\bs H}^2$.
\begin{lemma}\label{lemma:v-decom}
    We have
    \begin{equation*}
        \bbE\|\bs v_k\|_{\bs H}^2=L_\rmvar\left(\left\{\tbs C_t\right\}_{t\in\bbN}\right)+\sum_{\ell=1}^\infty L_\rmvar\left(\left\{\tbs C_t^{(\ell)}\right\}_{t\in\bbN}\right),
    \end{equation*}
\end{lemma}
\begin{proof}
    By Lemma~\ref{lemma:layer-peeled-decomposition}, we have
    \begin{equation*}
        \bs C_t=\tbs C_t+\sum_{\ell=1}^\infty\tbs C_t^{(\ell)}.
    \end{equation*}
    The desired result follows from the linearity of $L_\rmvar$ in Lemma~\ref{lemma:L-var-linearity}.
\end{proof}

Define the $\ell$-level stationary state $\tbs C_\infty^{(\ell)}$ as follows:

\begin{equation*}
    \tbs C_\infty^{(\ell)}=\begin{pmatrix}
        \tbs C_{\infty,11}^{(\ell)} & \tbs C_{\infty,12}^{(\ell)} \\
        \tbs C_{\infty,21}^{(\ell)} & \tbs C_{\infty,22}^{(\ell)}
    \end{pmatrix}\defeq\tbs\Pi_\infty(\bs R^{(\ell)}),\quad\text{where $\bs R^{(\ell)}=(\cM-\tilde{\cM})\circ\tbs C_{\infty,22}^{(\ell-1)}$}.
\end{equation*}
We set $\tbs C_\infty^{(0)}=\tbs C_\infty\defeq\tbs\Pi_\infty(\bs R)$ defined in \eqref{eq:stationary-tilde-Pi}. One can verify $\tbs C_t^{(\ell)}\preceq\tbs C_\infty^{(\ell)}$. Thus,
\begin{equation*}
    L_\rmvar\left(\left\{\tbs C_t^{(\ell)}\right\}_{t\in\bbN}\right)\leq L_\rmvar\left(\left\{\tbs\Pi_t(\bs R^{(\ell)})\right\}_{t\in\bbN}\right)
\end{equation*}

We apply the results in Appendix~\ref{sec:core-dynamics} for $\bs M=\bs R^{(\ell)}$ to obtain the following lemmas.
\begin{lemma}\label{lemma:L-var-l-bound}
    We have
    \begin{equation*}
        L_\rmvar\left(\left\{\tbs C_t^{(\ell)}\right\}_{t\in\bbN}\right)\leq L_\rmvar\left(\left\{\tbs\Pi_t(\bs R^{(\ell)})\right\}_{t\in\bbN}\right)\leq\frac{64\tr(\bs\Sigma^{-1}\bs R^{(\ell)})}{T}.
    \end{equation*}
\end{lemma}
\begin{lemma}\label{lemma:var-tilde-C-l-bound}
    We have
    \begin{equation*}
    \begin{aligned}
        \left\langle\begin{pmatrix}
            \bs O & \bs O \\
            \bs O & \bs\Sigma
        \end{pmatrix},\tbs C_t^{(\ell)}\right\rangle&\leq\left\langle\begin{pmatrix}
            \bs O & \bs O \\
            \bs O & \bs\Sigma
        \end{pmatrix},\tbs C_\infty^{(\ell)}\right\rangle\leq\frac{\tr(\bs\Sigma^{-1}\bs R^{(\ell)})}{4\tilde{\kappa}}+4\eta\tr\bs R^{(\ell)}.
    \end{aligned}
    \end{equation*}
\end{lemma}

We define
\begin{equation*}
    \sigma^{(\ell)}=\left\langle\begin{pmatrix}
            \bs O & \bs O \\
            \bs O & \bs\Sigma
        \end{pmatrix},\tbs C_\infty^{(\ell)}\right\rangle.
\end{equation*}
The following lemma is an implication of Assumption~\ref{assumption:fourth-moment}, which is the first step to relate $\sigma^{(\ell-1)}$ to $\sigma^{(\ell)}$.
\begin{lemma}\label{lemma:var-fourth-moment}
    Suppose Assumption~\ref{assumption:fourth-moment} holds, then we have
    \begin{equation*}
        \tr\bs R^{(\ell)}\leq R^2\sigma^{(\ell-1)},\quad\tr(\bs\Sigma^{-1}\bs R^{(\ell)})\leq\tilde{\kappa}\sigma^{(\ell-1)}.
    \end{equation*}
\end{lemma}
\begin{proof}
    By the definition of $\bs R^{(\ell)}$,
    \begin{equation*}
        \bs R^{(\ell)}=(\cM-\tilde{\cM})\circ\tbs C_{\infty,22}^{(\ell-1)}=\bbE\left(\|\bs a\|_{\tbs C_{\infty,22}^{(\ell-1)}}^2\bs a\bs a^\top\right)-\bs\Sigma\tbs C_{\infty,22}^{(\ell-1)}\bs\Sigma.
    \end{equation*}
    We drop the last term, which is a PSD matrix, so
    \begin{align*}
        \tr\bs R^{(\ell)}&\leq\bbE\left(\|\bs a\|^2\|\bs a\|_{\tbs C_{\infty,22}^{(\ell-1)}}^2\right)=\left\langle\bbE\left(\|\bs a\|^2\bs a\bs a^\top\right),\tbs C_{\infty,22}^{(\ell-1)}\right\rangle \\
        &\stackrel{a}{\leq}\left\langle R^2\bs\Sigma,\tbs C_{\infty,22}^{(\ell-1)}\right\rangle=R^2\left\langle\begin{pmatrix}
            \bs O & \bs O \\
            \bs O & \bs\Sigma
        \end{pmatrix},\tbs C_\infty^{(\ell-1)}\right\rangle\stackrel{b}{\leq}R^2\sigma^{(\ell-1)}. \\
        \tr(\bs\Sigma^{-1}\bs R^{(\ell)})&\leq\bbE\left(\|\bs a\|_{\bs\Sigma^{-1}}^2\|\bs a\|_{\tbs C_{\infty,22}^{(\ell-1)}}^2\right)=\left\langle\bbE\left(\|\bs a\|_{\bs\Sigma^{-1}}^2\bs a\bs a^\top\right),\tbs C_{\infty,22}^{(\ell-1)}\right\rangle \\
        &\stackrel{a}{\leq}\left\langle\tilde{\kappa}\bs\Sigma,\tbs C_{\infty,22}^{(\ell-1)}\right\rangle=\tilde{\kappa}\left\langle\begin{pmatrix}
            \bs O & \bs O \\
            \bs O & \bs\Sigma
        \end{pmatrix},\tbs C_\infty^{(\ell-1)}\right\rangle\stackrel{b}{\leq}\tilde{\kappa}\sigma^{(\ell-1)}.
    \end{align*}
    where $\stackrel{a}{\leq}$ is from Assumption~\ref{assumption:fourth-moment}, and $\stackrel{b}{\leq}$ uses Lemma~\ref{lemma:var-tilde-C-l-bound}.
\end{proof}

Combine Lemma~\ref{lemma:var-tilde-C-l-bound} and Lemma~\ref{lemma:var-fourth-moment}, we relate $\sigma^{(\ell-1)}$ to $\sigma^{(\ell)}$ as follows.
\begin{lemma}\label{lemma:var-sigma-l-recursion}
    We have $\sigma^{(\ell)}\leq\sigma^{(\ell-1)}/2$ for all $\ell\geq 1$.
\end{lemma}
\begin{proof}
    Note that
    \begin{equation*}
    \begin{aligned}
        \sigma^{(\ell)}\stackrel{a}{\leq}\frac{\tr(\bs\Sigma^{-1}\bs R^{(\ell)})}{4\tilde{\kappa}}+4\eta\tr\bs R^{(\ell)}\stackrel{a}{\leq}\frac{1}{4\tilde{\kappa}}\cdot\tilde{\kappa}\sigma^{(\ell-1)}+4\eta\cdot R^2\sigma^{(\ell-1)}\stackrel{b}{\leq}\frac{1}{2}\sigma^{(\ell-1)},
    \end{aligned}
    \end{equation*}
    where $\stackrel{a}{\leq}$ uses Lemma~\ref{lemma:var-tilde-C-l-bound}, $\stackrel{b}{\leq}$ applies Lemma~\ref{lemma:var-fourth-moment}, and $\stackrel{c}{\leq}$ uses $\eta\leq 1/(16R^2)$.
\end{proof}

The follow lemma bounds $L_\rmvar(\{\tbs C_t^{(\ell)}\}_{t\in\bbN})$ for $\ell\geq 1$
\begin{lemma}\label{lemma:L-var-bound-C-l}
    We have
    \begin{equation*}
    \begin{aligned}
        L_\rmvar\left(\left\{\tbs C_t^{(\ell)}\right\}_{t\in\bbN}\right)&\leq\frac{64\tilde{\kappa}}{2^{\ell-1}T}\bigg(\frac{5\tr(\bs\Sigma^{-1}\bs Q)+5L_\ell(L+\tilde{\kappa})(F(\tbs y)-F(\bs x^*))}{4\tilde{\kappa}} \\
        &\phrel+20\eta\tr\bs Q+20\eta L_\ell(B+R^2)(F(\tbs y)-F(\bs x^*))\bigg).
    \end{aligned}
    \end{equation*}
\end{lemma}
\begin{proof}
    By Lemma~\ref{lemma:var-tilde-C-bound}, we have
    \begin{equation*}
    \begin{aligned}
        \sigma^{(0)}\defeq
        \left\langle\begin{pmatrix}
            \bs O & \bs O \\
            \bs O & \bs\Sigma
        \end{pmatrix},\tbs C_\infty\right\rangle&\leq\frac{5\tr(\bs\Sigma^{-1}\bs Q)+5L_\ell(L+\tilde{\kappa})(F(\tbs y)-F(\bs x^*))}{4\tilde{\kappa}} \\
        &\phrel+20\eta L_\ell\tr\bs Q+20\eta(B+R^2)(F(\tbs y)-F(\bs x^*)).
    \end{aligned}
    \end{equation*}
    By Lemma~\ref{lemma:var-sigma-l-recursion}, we have
    \begin{equation*}
    \begin{aligned}
        L_\rmvar\left(\left\{\tbs C_t^{(\ell)}\right\}_{t\in\bbN}\right)&\leq L_\rmvar\left(\left\{\tbs\Pi_t(\bs R^{(\ell))}\right\}_{t\in\bbN}\right)\leq\frac{64\tr(\bs\Sigma^{-1}\bs R^{(\ell)})}{T}\leq\frac{64\tilde{\kappa}\sigma^{(\ell)}}{T}\leq\frac{64\tilde{\kappa}\sigma^{(0)}}{2^{\ell-1}T} \\
        &\leq\frac{64\tilde{\kappa}}{2^{\ell-1}T}\bigg(\frac{5\tr(\bs\Sigma^{-1}\bs Q)+5L_\ell(L+\tilde{\kappa})(F(\tbs y)-F(\bs x^*))}{4\tilde{\kappa}} \\
        &\phrel+20\eta\tr\bs Q+20\eta L_\ell(B+R^2)(F(\tbs y)-F(\bs x^*))\bigg),
    \end{aligned}
    \end{equation*}
    This completes the proof.
\end{proof}

Finally, we are ready to bound $\sum_{\ell=1}^{\infty}L_\rmvar(\{\tbs C_t^{(\ell)}\}_{t\in\bbN})$.
\begin{lemma}\label{lemma:C-1-bound}    
    Suppose Assumptions~\ref{assumption:fourth-moment}, \ref{assumption:gradient-noise-I} and \ref{assumption:gradient-noise} hold, then we have
    \begin{equation*}
    \begin{aligned}
        \sum_{\ell=1}^{\infty}L_\rmvar\left(\left\{\tbs C_t^{(\ell)}\right\}_{t\in\bbN}\right)&\leq\frac{640\tr(\bs\Sigma^{-1}\bs Q)+640L_\ell(L+\tilde{\kappa})(F(\tbs y)-F(\bs x^*))}{T} \\
        &\phrel+\frac{2560\eta\tilde{\kappa}\tr\bs Q+2560\eta\tilde{\kappa}L_\ell(B+R^2)(F(\tbs y)-F(\bs x^*))}{T}.        
    \end{aligned}
    \end{equation*}
\end{lemma}
\begin{proof}
    We apply Lemma~\ref{lemma:var-sigma-l-recursion},
    \begin{equation*}
    \begin{aligned}
        \sum_{\ell=1}^{\infty}L_\rmvar\left(\left\{\tbs C_t^{(\ell)}\right\}_{t\in\bbN}\right)&\leq\sum_{\ell=1}^\infty\frac{64\tilde{\kappa}}{2^{\ell-1}T}\bigg(\frac{5\tr(\bs\Sigma^{-1}\bs Q)+5L_\ell(L+\tilde{\kappa})(F(\tbs y)-F(\bs x^*))}{4\tilde{\kappa}} \\
        &\phrel+20\eta\tr\bs Q+20\eta L_\ell(B+R^2)(F(\tbs y)-F(\bs x^*))\bigg) \\
        &\leq\frac{640\tr(\bs\Sigma^{-1}\bs Q)+640 L_\ell(L+\tilde{\kappa})(F(\tbs y)-F(\bs x^*))}{T} \\
        &\phrel+\frac{2560\eta\tilde{\kappa}\tr\bs Q+2560\eta\tilde{\kappa} L_\ell(B+R^2)(F(\tbs y)-F(\bs x^*))}{T}.
    \end{aligned}
    \end{equation*}
    This completes the proof.
\end{proof}

\subsubsection{Proof of Lemma~\ref{lemma:v-bound}}\label{sec:proof-lemma-v-bound}
\begin{proof}[Proof of Lemma~\ref{lemma:v-bound}]
    By the definition of $\bs\eta_t^\rmvar$, we have $\bbE\bs\eta_t=\bs\eta_0=\bs 0$. Thus, by the definition of $\bs v$ in Definition~\ref{def:r-v}, we have $\bbE\bs v=\bs 0$. To bound $\bbE\|\bs v\|_{\bs\Sigma}^2$, we apply Lemma~\ref{lemma:v-decom},
    \begin{equation*}
        \bbE\|\bs v_k\|_{\bs H}^2=L_\rmvar\left(\left\{\tbs C_t\right\}_{t\in\bbN}\right)+\sum_{\ell=1}^\infty L_\rmvar\left(\left\{\tbs C_t^{(\ell)}\right\}_{t\in\bbN}\right).
    \end{equation*}
    
    For the bound of $\bs v_k$, we have
    \begin{equation*}
    \begin{aligned}
        \bbE\|\bs v_k\|_{\bs\Sigma}^2&=L_\rmvar\left(\left\{\tbs C_t\right\}_{t\in\bbN}\right)+\sum_{\ell=1}^\infty L_\rmvar\left(\left\{\tbs C_t^{(\ell)}\right\}_{t\in\bbN}\right) \\
        &\stackrel{a}{\leq}\frac{320\tr(\bs\Sigma^{-1}\bs Q)}{T}+\frac{320L_\ell(L+\tilde{\kappa})(F(\tbs y)-F(\bs x^*))}{T} \\
        &\phrel+\frac{640\tr(\bs\Sigma^{-1}\bs Q)+640L_\ell(L+\tilde{\kappa})(F(\tbs y)-F(\bs x^*))}{T} \\
        &\phrel+\frac{2560\eta\tilde{\kappa}\tr\bs Q+2560\eta\tilde{\kappa}L_\ell(B+R^2)(F(\tbs y)-F(\bs x^*))}{T}.        
    \end{aligned}
    \end{equation*}
    where $\stackrel{a}{\leq}$ applies Lemma~\ref{lemma:L-var-C-tilde} and Lemma~\ref{lemma:C-1-bound}. This completes the proof.
\end{proof}

\subsection{Bias Upper Bound}\label{sec:bias-bound}
The goal of this section is to prove the following lemma.
\begin{lemma}[Bias Upper Bound]\label{lemma:r-bound}
    Suppose Assumptions~\ref{assumption:l-condition}, \ref{assumption:regularity}, \ref{assumption:fourth-moment}, \ref{assumption:gradient-noise-I} and \ref{assumption:gradient-noise} hold. Let $\bs r$ defined in Definition~\ref{def:r-v}. Then we have
    \begin{equation*}
        \bbE\|\bs r\|_{\bs\Sigma}^2\leq\frac{L_\ell}{8\alpha}\|\nabla F(\tbs y_{k-1})\|_{\bs\Sigma^{-1}}^2,
    \end{equation*}
    where the expectation is taken with respect to the samples drawn in the $k$-th outer iteration.
\end{lemma}
The proof is deferred to the end of this subsection.

Recall the definition of $\bs\eta_t^\rmbias\in\bbR^{2d}$ in \eqref{eq:def-eta}:
\begin{equation*}
    \bs\eta^\rmbias_t=\hbs{A}_t\bs\eta^\rmbias_{t-1},\quad\bs\eta^\rmbias_0=\begin{pmatrix}
        \tbs y-\tbs x_+^* \\
        \tbs y-\tbs x_+^* \\
    \end{pmatrix}.
\end{equation*}
Let
\begin{equation*}
    \bs\eta_t^\rmbias=\begin{pmatrix}
        \bs x_t^\rmbias \\ \bs y_t^\rmbias
    \end{pmatrix},\quad\text{where $\bs x_t^\rmbias, \bs y_t^\rmbias\in\bbR^d$}
\end{equation*}
For simplicity, we drop the superscript of $\bs x_t^\rmbias, \bs y_t^\rmbias,\bs z_t^\rmbias$ in this subsection. The bias iteration can be written in the following form:
\begin{align*}
    \bs y_{t-1}&=\frac{1}{1+\theta}\bs x_{t-1}+\frac{\theta}{1+\theta}\bs z_{t-1}, \\
    \bs x_t&=\bs y_{t-1}-\eta\bs a_t\bs a_t^\top(\bs y_{t-1}-\tbs x_+^*), \\
    \bs z_t&=\theta\bs y_{t-1}+(1-\theta)\bs z_{t-1}-\gamma\bs a_t\bs a_t^\top(\bs y_{t-1}-\tbs x_+^*),
\end{align*}
where $\bs x_0=\bs z_0=\tbs y-\tbs x_+^*$. The following lemma shows that $\|\bs x_t-\tbs x_+^*\|$ converges exponentially fast.
\begin{lemma}\label{lemma:bias-bound}
    Suppose Assumption~\ref{assumption:fourth-moment} holds and let the hyperparameter be chosen as specified in Appendix~\ref{sec:hyperparameter-choice}, then we have
    \begin{equation}
        \bbE\|\bs x_T-\tbs x_+^*\|^2+\frac{\theta}{2\gamma}\bbE\|\bs z_T-\tbs x_+^*\|_{\bs\Sigma^{-1}}^2\leq(1-\theta)^T\left(\|\bs x_0-\tbs x_+^*\|^2+\frac{\theta}{2\gamma}\|\bs x_0-\tbs x_+^*\|_{\bs\Sigma^{-1}}^2\right).
    \end{equation}
\end{lemma}
\begin{proof}
    Let $\bbE_t$ denote the expectation conditioned on $\{(\bs a_s,b_s)\}_{s=1}^t$. We consider the energy function $\|\bs x_t-\tbs x_+^*\|^2/2$. By the update rule,
    \begin{align}
        \frac{1}{2}\bbE_{t-1}\|\bs x_t-\tbs x_+^*\|^2&=\frac{1}{2}\|\bs y_{t-1}-\tbs x_+^*\|^2+\bbE_{t-1}\langle\bs y_{t-1}-\tbs x_+^*,\bs x_t-\bs y_{t-1}\rangle+\frac{1}{2}\bbE_{t-1}\|\bs x_t-\bs y_{t-1}\|^2\notag \\
        &=\frac{1}{2}\|\bs y_{t-1}-\tbs x_+^*\|^2-\left\langle\bs y_{t-1}-\tbs x_+^*,\bs\Sigma(\bs y_{t-1}-\tbs x_+^*)\right\rangle\notag \\
        &\phrel+\frac{\eta^2}{2}\bbE_{t-1}\left\|\bs a_t\bs a_t^\top(\bs y_{t-1}-\tbs x_+^*)\right\|^2\notag \\
        &\stackrel{a}{\leq}\frac{1}{2}\|\bs y_{t-1}-\tbs x_+^*\|^2-\left(\eta-\frac{\eta^2 R^2}{2}\right)\|\bs y_{t-1}-\tbs x_+^*\|_{\bs\Sigma}^2\notag \\
        &\stackrel{b}{\leq}\frac{1}{2}\|\bs y_{t-1}-\tbs x_+^*\|^2-\frac{\eta}{2}\|\bs y_{t-1}-\tbs x_+^*\|_{\bs\Sigma}^2,\label{eq:bias-decrease}
    \end{align}
    where $\stackrel{a}{\leq}$ uses $\bbE\left(\|\bs a\|^2\bs a\bs a^\top\right)\preceq R^2\bs\Sigma$, and $\stackrel{b}{\leq}$ uses $\eta\leq 1/R^2$. Then, we have to bound $\|\bs y_{t-1}-\tbs x_+^*\|^2/2$. By the convexity of $\|\cdot\|$, we have
    \begin{align}
        \frac{1}{2}\|\bs y_{t-1}-\tbs x_+^*\|^2&\leq\langle\bs y_{t-1}-\tbs x_+^*,\bs y_{t-1}-\tbs x_+^*\rangle-\frac{1}{2}\|\bs y_{t-1}-\tbs x_+^*\|^2,\label{eq:bias-x-opt} \\
        \frac{1}{2}\|\bs y_{t-1}-\tbs x_+^*\|^2&\leq\frac{1}{2}\|\bs x_{t-1}-\tbs x_+^*\|^2+\langle\bs y_{t-1}-\tbs x_+^*,\bs y_{t-1}-\bs x_{t-1}\rangle.\label{eq:bias-x-prev}
    \end{align}
    Let $\bs x=(1-\theta)\bs x_{t-1}+\theta\bs x_+^*$, then $\theta\times\eqref{eq:bias-x-opt}+(1-\theta)\times\eqref{eq:bias-x-prev}$ yields
    \begin{equation*}
        \frac{1}{2}\|\bs y_{t-1}-\bs x_{t-1}\|^2\leq\frac{1}{2}(1-\theta)\|\bs x_{t-1}-\tbs x_+^*\|^2+\langle\bs y_{t-1}-\tbs x_+^*,\bs y_{t-1}-\bs x\rangle-\frac{\theta}{2}\|\bs y_{t-1}-\tbs x_+^*\|^2.
    \end{equation*}
    Let $\bs w_{t-1}=\theta\bs y_{t-1}+(1-\theta)\bs z_{t-1}$, and note that
    \begin{equation*}
         \bbE(\bs x_t-\bs y_{t-1})=-\eta\bs\Sigma(\bs y_{t-1}-\tbs x_+^*),
    \end{equation*}
    \begin{equation*}
         \bs y_{t-1}-\bs x=\theta(\bs w_k-\tbs x_+^*),\quad 
         \bs x_t-\bs y_{t-1}=\frac{\eta}{\gamma}(\bs z_t-\bs w_{t-1}).
    \end{equation*}
    Therefore, we have
    \begin{equation*}
    \begin{aligned}        
        &\phrel\frac{1}{2}\|\bs y_{t-1}-\bs x_{t-1}\|^2 \\
        &\leq\frac{1}{2}(1-\theta)\|\bs x_{t-1}-\tbs x_+^*\|^2+\frac{\theta}{\gamma}\bbE\langle\bs w_{t-1}-\bs z_{t-1},\bs w_{t-1}-\tbs x_+^*\rangle_{\bs\Sigma^{-1}}-\frac{\theta}{2}\|\bs y_{t-1}-\tbs x_+^*\|^2 \\
        &=\frac{1}{2}(1-\theta)\|\bs x_{t-1}-\tbs x_+^*\|^2-\frac{\theta}{2}\|\bs y_{t-1}-\tbs x_+^*\|^2 \\
        &\phrel+\frac{\theta}{2\gamma}\left(\|\bs w_{t-1}-\tbs x_+^*\|_{\bs\Sigma^{-1}}^2+\bbE_{t-1}\|\bs z_t-\bs w_{t-1}\|_{\bs\Sigma^{-1}}^2-\bbE_{t-1}\|\bs z_t-\tbs x_+^*\|_{\bs\Sigma^{-1}}^2\right) \\
        &\stackrel{a}{\leq}\frac{1}{2}(1-\theta)\|\bs x_{t-1}-\tbs x_+^*\|^2-\frac{\theta}{2}\|\bs y_{t-1}-\tbs x_+^*\|^2+\frac{\theta^2}{2\gamma}\|\bs y_{t-1}-\tbs x_+^*\|_{\bs\Sigma^{-1}}^2 \\
        &\phrel+\frac{\theta}{2\gamma}\left((1-\theta)\|\bs z_{t-1}-\tbs x_+^*\|_{\bs\Sigma^{-1}}^2+\bbE_{t-1}\|\bs z_t-\bs w_{t-1}\|_{\bs\Sigma^{-1}}^2-\bbE_{t-1}\|\bs z_t-\tbs x_+^*\|_{\bs\Sigma^{-1}}^2\right).
    \end{aligned}
    \end{equation*}
    Apply Assumption~\ref{assumption:fourth-moment}, we have
    \begin{equation*}
    \begin{aligned}        
        \frac{1}{2}\|\bs y_{t-1}-\bs x_{t-1}\|^2&\leq\frac{1}{2}(1-\theta)\|\bs x_{t-1}-\tbs x_+^*\|^2+\frac{\tilde{\kappa}\theta\gamma}{2}\|\bs y_{t-1}-\tbs x_+^*\|_{\bs\Sigma}^2-\frac{\theta(\mu\gamma-\theta)}{2\mu\gamma}\|\bs y_{t-1}-\tbs x_+^*\|^2 \\
        &\phrel+\frac{\theta}{2\gamma}\left((1-\theta)\|\bs z_{t-1}-\tbs x_+^*\|_{\bs\Sigma^{-1}}^2-\bbE_{t-1}\|\bs z_t-\tbs x_+^*\|_{\bs\Sigma^{-1}}^2\right).
    \end{aligned}
    \end{equation*}
    Combining the above inequality and \eqref{eq:bias-decrease}, we have
    \begin{equation*}
    \begin{aligned}
        &\phrel\frac{1}{2}\bbE_{t-1}\|\bs x_t-\tbs x_+^*\|^2 \\
        &\leq\frac{1}{2}(1-\theta)\|\bs x_{t-1}-\tbs x_+^*\|^2+\frac{\theta}{2\gamma}\left((1-\theta)\|\bs z_{t-1}-\tbs x_+^*\|_{\bs\Sigma^{-1}}^2-\bbE_{t-1}\|\bs z_t-\tbs x_+^*\|_{\bs\Sigma^{-1}}^2\right) \\
        &\phrel-\frac{\eta-\tilde{\kappa}\theta\gamma}{2}\|\bs y_{t-1}-\tbs x_+^*\|_{\bs\Sigma}^2-\frac{\theta(\gamma\mu-\theta)}{2\mu\gamma}\|\bs y_{t-1}-\tbs x_+^*\|^2.
    \end{aligned}
    \end{equation*}
    By the hyperparameter choice, we have $\eta-\tilde{\kappa}\theta\gamma\geq 0$ and $\gamma\mu-\theta\geq 0$, so
    \begin{equation*}
        \bbE_{t-1}\|\bs x_t-\tbs x_+^*\|^2+\frac{\theta}{2\gamma}\bbE_{t-1}\|\bs z_t-\tbs x_+^*\|_{\bs\Sigma^{-1}}^2\leq(1-\theta)\left(\|\bs x_{t-1}-\tbs x_+^*\|^2+\frac{\theta}{2\gamma}\|\bs z_{t-1}-\tbs x_+^*\|_{\bs\Sigma^{-1}}^2\right).
    \end{equation*}
    Thus, we take the expectation and obtain
    \begin{equation*}
        \bbE\|\bs x_T-\tbs x_+^*\|^2+\frac{\theta}{2\gamma}\bbE\|\bs z_T-\tbs x_+^*\|_{\bs\Sigma^{-1}}^2\leq(1-\theta)^T\left(\|\bs x_0-\tbs x_+^*\|^2+\frac{\theta}{2\gamma}\|\bs x_0-\tbs x_+^*\|_{\bs\Sigma^{-1}}^2\right).
    \end{equation*}
    This completes the proof.
\end{proof}

We are ready to prove Lemma~\ref{lemma:r-bound}.
\begin{proof}[Proof of Lemma~\ref{lemma:r-bound}]
    By Lemma~\ref{lemma:bias-bound}, we have
    \begin{equation*}
    \begin{aligned}
        \bbE\|\bs r_k\|_{\bs\Sigma}^2&\leq\lambda_\rmmax(\bs\Sigma)(1-\theta)^T\left(\|\bs x_0-\bs x_+^*\|^2+\frac{\theta}{2\gamma}\|\bs x_0-\bs x_+^*\|_{\bs\Sigma^{-1}}^2\right) \\
        &\stackrel{a}{\leq}\frac{L_\ell}{8\alpha}\|\nabla F(\tbs y_{k-1})\|_{\bs\Sigma^{-1}}^2,       
    \end{aligned}
    \end{equation*}
    where $\stackrel{a}{\leq}$ follows from the choice of $T$ and $\theta$ in Appendix~\ref{sec:hyperparameter-choice}.
\end{proof}

\subsection{Proof of Lemma~\ref{lemma:inner-loop}}\label{sec:proof-inner-loop-lemma}

\begin{proof}[Proof of Lemma~\ref{lemma:inner-loop}]
    Recall that $\bs H=L_\ell\bs\Sigma$. The desired results follow by combining Lemma~\ref{lemma:r-bound} and Lemma~\ref{lemma:v-bound}.
\end{proof}

\subsection{Auxilary Lemmas}

\begin{lemma}
    Suppose Assumption~\ref{assumption:fourth-moment} holds, then $R^2\geq\tr\bs\Sigma$ and $\tilde{\kappa}\geq d$.
\end{lemma}
\begin{proof}
    For the first statement, note that
    \begin{equation*}
        (\tr\bs\Sigma)^2=\left(\bbE\|\bs a\|^2\right)^2\leq\bbE\|\bs a\|^4\leq R^2\tr\bs\Sigma.
    \end{equation*}
    For the second statement, note that
    \begin{equation*}
        d^2=\left(\bbE\|\bs a\|_{\bs\Sigma^{-1}}^2\right)^2\leq\bbE\|\bs a\|_{\bs\Sigma^{-1}}^4\leq \tilde{\kappa}d.
    \end{equation*}
    The results follow by cancelling $\tr\bs\Sigma$ or $d$.
\end{proof}

\begin{lemma}\label{lemma:aux-lemma-r}
    For $k\geq 0$, we have $x^k\left[1+k(1-x)\right]\leq 1$ for $x\in[0,1]$.
\end{lemma}
\begin{proof}
    Let $f(x)=x^k\left[1+k(1-x)\right]$. The derivative of $f(x)$ is $f'(x)=k(k+1)x^{k-1}(1-x)\geq 0$ for $x\in[0,1]$. Thus, we have $f(x)\leq f(1)=1$.
\end{proof}

\section{Part II: Analysis of Outer Loop}\label{sec:outer-loop}
In this section, we analyze the outer loop of  Algorithm~\ref{alg:sada}. In Section~\ref{sec:excess-risk-outer}, we restate Lemma~\ref{lemma:inner-loop}, under which we analyze the outer loop.

\subsection{Excess Risk Bound of Outer Loop}\label{sec:excess-risk-outer}
Let $\cF_k$ denote the filtration generated by the samples of the first $k$ inner loops. By Lemma~\ref{lemma:inner-loop}, we write the update of the outer loop as follows:
\begin{align}
    \tilde{\bs y}_{k-1}&=\tbs x_{k-1}+\beta(\tbs x_{k-1}-\tbs x_{k-2}),\label{eq:def-tilde-y-update} \\
    \tbs x_k&=\tbs y_{k-1}-L_\ell h_k\bs H^{-1}\nabla F(\tbs y_{k-1})+h_k\bs r_k+h_k\bs v_k.\label{eq:def-tilde-x-update}
\end{align}
where $\bs r_k$ and $\bs v_k$ satisfies
\begin{align*}
    \bbE(\|\bs r_k\|_{\bs H}^2|\cF_{k-1})&\leq \frac{L_\ell h_k}{8\alpha}\|\nabla F(\tbs y_{k-1})\|_{\bs H^{-1}}^2,\quad\bbE(\bs v_k|\cF_{k-1})=0, \\
    \bbE(\|\bs v_k\|_{\bs H}^2|\cF_{k-1})&\leq\frac{320 L_\ell\left(3 L_\ell\tr(\bs H^{-1}\bs Q)+8 \eta\tilde{\kappa}\tr\bs Q\right)}{T} \\
    &\phrel+\frac{160 L_\ell^2(6L+\tilde{\kappa}(7+16\eta B))(F(\tilde{\bs y})-F(\bs x^*))}{T} \\    
    &\defeq L_\ell \sigma^2+L_\rmeff L_\ell^2(F(\tbs y_{k-1})-F(\bs x^*)),
\end{align*}
where
\begin{equation*}
    \sigma^2=\frac{320\left(3 L_\ell\tr(\bs H^{-1}\bs Q)+8 \eta\tilde{\kappa}\tr\bs Q\right)}{T},\quad L_\rmeff=\frac{160 (6L+\tilde{\kappa}(7+16\eta B))(F(\tilde{\bs y})-F(\bs x^*))}{T}.
\end{equation*}

\begin{lemma}
    Assume we have
    \begin{equation*}
        h_k\leq\min\left\{\frac{1}{6L_\ell},\frac{T^2}{288\alpha L_\ell L_\rmeff^2}\right\},\quad\theta_k=\sqrt{\frac{L_\ell h_k}{2\alpha}},\quad\beta_k=\frac{1-\theta_k}{1+\theta_k}.
    \end{equation*}
    Let $\tbs z_k=\tbs y_k+(\tbs y_k-\tbs x_k)/\theta_k$, then we have
    \begin{equation*}
    \begin{aligned}        
        &\phrel\bbE_{k-1}F(\tbs x_k)-F(\bs x^*)+\frac{2\theta_k^2}{3 L_\ell h_k}\bbE_{k-1}\|\tbs z_k-\bs x^*\|_{\bs H}^2 \\
        &\leq\left(1-\frac{\theta_k}{2}\right)\left(F(\tbs x_{k-1})-F(\bs x^*)+\frac{2\theta_k^2}{3 L_\ell h_k}\|\tbs z_{k-1}-\bs x^*\|_{\bs H}^2\right)+\frac{3}{2}h_k\sigma^2.
    \end{aligned}
    \end{equation*}
\end{lemma}
\begin{proof}
    Since $\nabla^2 F\preceq\bs H$, we have
    \begin{equation*}
        F(\tbs x_k)\leq F(\tbs y_{k-1})+\langle\nabla F(\tbs y_{k-1}),\tbs x_k-\tbs y_{k-1}\rangle+\frac{1}{2}\|\tbs x_k-\tbs y_{k-1}\|_{\bs H}^2.
    \end{equation*}
    Take expectation conditioned on $\cF_{k-1}$ and use \eqref{eq:def-tilde-x-update} twice, we have
    \begin{equation}
    \begin{aligned}
        &\phrel\bbE_{k-1}\langle\nabla F(\tbs y_{k-1}),\tbs x_k-\tbs y_{k-1}\rangle \\
        &=-\frac{1}{L_\ell h_k}\bbE_{k-1}\langle\bs H(\tbs x_k-\tbs y_{k-1}),\tbs x_k-\tbs y_{k-1}\rangle+\frac{1}{L_\ell}\bbE_{k-1}\langle\bs H(\bs r_k+\bs v_k),\tbs x_k-\tbs y_{k-1}\rangle \\
        &=-\frac{1}{L_\ell h_k}\bbE_{k-1}\|\tbs x_k-\tbs y_{k-1}\|_{\bs H}^2-h_k\bbE_{k-1}\langle\nabla F(\tbs y_{k-1}),\bs r_k\rangle+\frac{h_k}{L_\ell}\bbE_{k-1}\|\bs r_k+\bs v_k\|_{\bs H}^2. \label{eq:inner-prod-bound-1}        
    \end{aligned}
    \end{equation}
    Note that we also have
    \begin{equation}
        \bbE_{k-1}\langle\nabla F(\tbs y_{k-1}),\tbs x_k-\tbs y_{k-1}\rangle=-L_\ell h_k\|\nabla F(\tbs y_{k-1})\|_{\bs H^{-1}}^2+h_k\langle\nabla F(\tbs y_{k-1}),\bs r_k\rangle \label{eq:inner-prod-bound-2}
    \end{equation}
    Therefore, replace $\langle\nabla F(\tbs y_{k-1}),\tbs x_k-\tbs y_{k-1}\rangle$ by $(3/4)\times\eqref{eq:inner-prod-bound-1}+(1/4)\times\eqref{eq:inner-prod-bound-2}$ to obtain
    \begin{equation*}
    \begin{aligned}
        \bbE_{k-1}F(\tbs x_k)&\leq F(\tbs y_{k-1})-\frac{3-2L_\ell h_k}{4L_\ell h_k}\|\tbs x_k-\tbs y_{k-1}\|_{\bs H}^2+\frac{3h_k}{2L_\ell}\left(\bbE_{k-1}\|\bs r_k\|_{\bs H}^2+\bbE_{k-1}\|\bs v_k\|_{\bs H}^2\right) \\
        &\phrel-\frac{h_k}{2}\bbE_{k-1}\langle\nabla F(\tbs y_{k-1}),\bs r_k\rangle-\frac{L_\ell h_k}{4}\|\nabla F(\tbs y_{k-1})\|_{\bs H^{-1}}^2. \\
        &\stackrel{a}{\leq} F(\tbs y_{k-1})-\frac{3-2L_\ell h_k}{4L_\ell h_k}\|\tbs x_k-\tbs y_{k-1}\|_{\bs H}^2-\frac{L_\ell h_k}{8}\|\nabla F(\tbs y_{k-1})\|_{\bs H^{-1}}^2 \\
        &\phrel+\frac{3h_k\sigma^2}{2}+\frac{3L_\rmeff L_\ell h_k(F(\tbs y_{k-1})-F(\bs x^*))}{2}, \\
    \end{aligned}
    \end{equation*}
    where $\stackrel{a}{\leq}$ uses $\bbE_{k-1}\langle\nabla F(\tbs y_{k-1}),\bs r_k\rangle\leq L_\ell\|\nabla F(\tbs y_{k-1})\|_{\bs H^{-1}}^2/4$. Thus,
    \begin{equation}
    \begin{aligned}
        \bbE_{k-1}F(\tbs x_k)-F(\bs x^*)&\leq\left(1+\frac{3L_\rmeff L_\ell h_k}{2}\right)(F(\tbs y_{k-1})-F(\bs x^*))+\frac{3h_k\sigma^2}{2} \\
        &\phrel-\frac{3-2L_\ell h_k}{4L_\ell h_k}\|\tbs x_k-\tbs y_{k-1}\|_{\bs H}^2-\frac{L_\ell h_k}{8}\|\nabla F(\tbs y_{k-1})\|_{\bs H^{-1}}^2.\label{eq:outer-tilde-x-bound}
    \end{aligned}
    \end{equation}
    Let $\beta_k=(1-\theta_k)/(1+\theta_k)$, $\tbs z_k=\tbs y_k+(\tbs y_k-\tbs x_k)/\theta_k$, $\tbs w_k=\tbs y_k+(1-\theta_k)(\tbs y_k-\tbs x_k)$, where $\theta_k$ will be chosen later.  By Assumption~\ref{assumption:l-condition}, we have $\nabla^2 F\geq\bs H/\alpha$, so
    \begin{align}
        F(\tbs y_{k-1})&\leq F(\bs x^*)+\langle\nabla F(\tbs y_{k-1}),\tbs y_{k-1}-\bs x^*\rangle-\frac{1}{2\alpha}\|\tbs y_{k-1}-\bs x^*\|_{\bs H}^2, \label{eq:outer-x-opt}\\        
        F(\tbs y_{k-1})&\leq F(\tbs x_{k-1})+\langle\nabla F(\tbs y_{k-1}),\tbs y_{k-1}-\tbs x_{k-1}\rangle-\frac{1}{2\alpha}\|\tbs y_{k-1}-\tbs x_{k-1}\|_{\bs H}^2. \label{eq:outer-x-prev}
    \end{align}
    Let $\tbs x=(1-\theta_k)\tbs x_{k-1}+\theta_k\bs x^*$, then $\theta_k\times\eqref{eq:outer-x-opt}+(1-\theta_k)\times\eqref{eq:outer-x-prev}$ yields
    \begin{equation*}
    \begin{aligned}
        F(\tbs y_{k-1})-F(\bs x^*)&\leq(1-\theta_k)(F(\tbs x_{k-1})-F(\bs x^*))+\langle\nabla F(\tbs y_{k-1}),\tbs y_{k-1}-\tbs x\rangle \\
        &\phrel-\frac{\theta_k}{2\alpha}\|\tbs y_{k-1}-\bs x^*\|_{\bs H}^2-\frac{1-\theta_k}{2\alpha}\|\tbs y_{k-1}-\tbs x_{k-1}\|_{\bs H}^2 \\
        &\stackrel{a}{\leq}(1-\theta_k)(F(\tbs x_{k-1})-F(\bs x^*))+\langle\nabla F(\tbs y_{k-1}),\tbs y_{k-1}-\tbs x\rangle \\
        &\phrel-\frac{\theta_k}{4\alpha}\|\tbs y_{k-1}-\bs x^*\|_{\bs H}^2-\frac{1}{4\alpha}\|\tbs y_{k-1}-\tbs x\|_{\bs H}^2,        
    \end{aligned}
    \end{equation*}
    where $\stackrel{a}{\leq}$ uses the convexity of $\|\cdot\|_{\bs H}^2$ as follows
    \begin{equation*}
        \theta_k\|\tbs y_{k-1}-\bs x^*\|_{\bs H}^2+(1-\theta_k)\|\tbs y_{k-1}-\tbs x_{k-1}\|_{\bs H}^2\geq\|\tbs y_{k-1}-\tbs x\|_{\bs H}^2.
    \end{equation*}
    Note that $\nabla F(\tbs y_{k-1})=\bs H\bbE_{k-1}(\tbs y_{k-1}-\tbs x_k)/(L_\ell h_k)+\bs H\bbE_{k-1}\bs r_k/L_\ell$. Therefore,
    \begin{equation*}
    \begin{aligned}
        &\phrel\langle\nabla F(\tbs y_{k-1}),\tbs y_{k-1}-\tbs x\rangle=\frac{1}{L_\ell h_k}\bbE_{k-1}\langle\bs H(\tbs y_{k-1}-\tbs x_k),\tbs y_{k-1}-\tbs x\rangle+\frac{1}{L_\ell}\bbE_{k-1}\langle\bs H\bs r_k,\tbs y_{k-1}-\tbs x\rangle \\
        &=\frac{\theta_k^2}{L_\ell h_k}\bbE_{k-1}\langle\bs H(\tbs w_{k-1}-\tbs z_k),\tbs w_{k-1}-\bs x^*\rangle+\frac{1}{L_\ell}\bbE_{k-1}\langle\bs H\bs r_k,\tbs y_{k-1}-\tbs x\rangle \\
        &\leq\frac{\theta_k^2}{2L_\ell h_k}\left(\|\tbs w_{k-1}-\bs x^*\|_{\bs H}^2+\bbE_{k-1}\|\tbs z_k-\tbs w_{k-1}\|_{\bs H}^2-\bbE_{k-1}\|\tbs z_k-\bs x^*\|_{\bs H}^2\right) \\
        &\phrel+\frac{\alpha}{L_\ell^2}\bbE_{k-1}\|\bs r_k\|_{\bs H}^2+\frac{1}{4\alpha}\|\tbs y_{k-1}-\tbs x\|_{\bs H}^2.
    \end{aligned}
    \end{equation*}
    Thus, by applying $\alpha\bbE_{k-1}\|\bs r_k\|_{\bs H}^2\leq L_\ell^3 h_k\|\nabla F(\tbs y_{k-1})\|_{\bs H^{-1}}^2/4$, we have
    \begin{equation}
    \begin{aligned}
        &\phrel F(\tbs y_{k-1})-F(\bs x^*) \\
        &\leq(1-\theta_k)(F(\tbs x_{k-1})-F(\bs x^*))+\frac{\alpha}{L_\ell^2}\bbE_{k-1}\|\bs r_k\|_{\bs H}^2-\frac{\theta_k}{4\alpha}\|\tbs y_{k-1}-\bs x^*\|_{\bs H}^2 \\
        &\phrel+\frac{\theta_k^2}{2L_\ell h_k}\left(\|\tbs w_{k-1}-\bs x^*\|_{\bs H}^2+\bbE_{k-1}\|\tbs z_k-\tbs w_{k-1}\|_{\bs H}^2-\bbE_{k-1}\|\tbs z_k-\bs x^*\|_{\bs H}^2\right) \\
        &\leq(1-\theta_k)(F(\tbs x_{k-1})-F(\bs x^*))+\frac{\alpha}{L_\ell^2}\bbE_{k-1}\|\bs r_k\|_{\bs H}^2-\left(\frac{\theta_k}{4\alpha}-\frac{\theta_k^3}{2h_k}\right)\|\tbs y_{k-1}-\bs x^*\|_{\bs H}^2 \\
        &\phrel+\frac{\theta_k^2}{2L_\ell h_k}\left((1-\theta_k)\|\tbs z_{k-1}-\bs x^*\|_{\bs H}^2+\bbE_{k-1}\|\tbs z_k-\tbs w_{k-1}\|_{\bs H}^2-\bbE_{k-1}\|\tbs z_k-\bs x^*\|_{\bs H}^2\right). \\
    \end{aligned}
    \end{equation}
    Finally, combine above bound and \eqref{eq:outer-tilde-x-bound}, and use $3L_\rmeff L_\ell h_k\leq\theta_k/4\leq 1/3$ and Lemma~\ref{lemma:inner-loop}, we have
    \begin{equation*}
    \begin{aligned}
        \bbE_{k-1}F(\tbs x_k)-F(\bs x^*)&\leq(1-\theta_k/2)(F(\tbs x_k)-F(\bs x^*))+\frac{3}{2}h_k\sigma^2+\frac{4\alpha}{3}\bbE_{k-1}\|\bs r_k\|_{\bs H}^2 \\
        &\phrel+\frac{2\theta_k^2}{3L_\ell h_k}\left((1-\theta_k)\|\tbs z_{k-1}-\bs x^*\|_{\bs H}^2-\bbE_{k-1}\|\tbs z_k-\bs x^*\|_{\bs H}^2\right) \\
        &\phrel-\frac{1-6 L_\ell h_k}{12 L_\ell h_k}\bbE_{k-1}\|\tbs x_k-\tbs y_{k-1}\|_{\bs H^{-1}}^2-\frac{h_k}{8}\|\nabla F(\tbs y_{k-1})\|_{\bs H^{-1}}^2. 
    \end{aligned}
    \end{equation*}
    We complete the proof by arranging the terms.
\end{proof}

For convenience, we analyze the following recursion,
\begin{equation*}
    L_k\leq(1-a\theta_k)L_{k-1}+c\theta_k^2,
\end{equation*}
where $\theta_k\leq 1/d$. The following lemma solves the recursion.
\begin{lemma}\label{lemma:auxilary-recursion}
    Suppose we choose
    \begin{equation*}
        \theta_k=\begin{cases}
            \theta_\rmmax, &k\leq K/2,\\
            \dfrac{2}{a(2/(a\theta_\rmmax)+k-K/2)}, &k>K/2.
        \end{cases}
    \end{equation*}
    Then we have
    \begin{equation*}
        L_K\leq\left(1-a\theta_\rmmax\right)^{K/2} L_0+\frac{12c}{a^2K}.
    \end{equation*}
\end{lemma}
\begin{proof}
    We first consider the case $K\leq 4d/a$. Note that we have
    \begin{equation*}
        L_k\leq(1-a\theta_k)L_{k-1}+c\theta_k\theta_\rmmax.
    \end{equation*}
    By rearranging the terms, we have
    \begin{equation*}
        L_k-\frac{c\theta_\rmmax}{a}\leq\left(1-a\theta_\rmmax\right)\left(L_{k-1}-\frac{c\textbf{}}{a}\right).
    \end{equation*}
    Thus, we solve the recursion,
    \begin{equation*}
        L_K\leq\left(1-a\theta_\rmmax\right)^K L_0+\frac{c\theta_\rmmax}{a}\leq\left(1-a\theta_\rmmax\right)^K L_0+\frac{4c}{a^2K}.
    \end{equation*}

    For the case $K>4/(a\theta_\rmmax)$, we have the following initial condition:
    \begin{equation*}
        L_{K/2}\leq\left(1-a\theta_\rmmax\right)^{K/2} L_0+\frac{c\theta_\rmmax}{a}.
    \end{equation*}
    By the choice of $\theta$, we have for $t>K/2$,
    \begin{equation*}
        L_k\leq\frac{2/(a\theta_\rmmax)+k-K/2-2}{2/(a\theta_\rmmax)+k-K/2}L_{k-1}+\frac{4c}{a^2(2/(a\theta_\rmmax)+k-K/2)^2}.
    \end{equation*}
    Multiply both side by $(2/(a\theta_\rmmax)+k-K/2)^2$, we have
    \begin{equation*}
        (2/(a\theta_\rmmax)+k-K/2)^2L_k\leq(2/(a\theta_\rmmax)+k-K/2-1)^2 L_{k-1}+\frac{4c}{a^2}.
    \end{equation*}
    We solve the recursion,
    \begin{equation*}
    \begin{aligned}
        L_K&\leq\frac{16}{a^2\theta_\rmmax^2 K^2}L_{K/2}+\frac{8c}{a^2 K}\leq\frac{16}{a^2\theta_\rmmax K^2}\left[\left(1-a\theta_\rmmax\right)^{K/2}L_0+\frac{c\theta_\rmmax}{a}\right]+\frac{8c}{a^2 K}, \\
        &\leq\left(1-a\theta_\rmmax\right)^{K/2}L_0+\frac{16c}{a^3\theta_\rmmax K^2}+\frac{8c}{a^2 K}\leq\left(1-a\theta_\rmmax\right)^{K/2}L_0+\frac{12c}{a^2 K}.
    \end{aligned}
    \end{equation*}
    We complete the proof by combining the two cases.
\end{proof}

\subsection{Proof of Theorem~\ref{thm:glm}}
\begin{theorem}[Full version of Theorem~\ref{thm:glm}]\label{thm:full-glm}
    Suppose we choose the hyperparameters as specified in Appendix~\ref{sec:hyperparameter-choice}. Let $\tbs x_K$ denote the output of Algorithm~\ref{alg:sada} after $K$ outer iterations, each consisting of $T\geq\tilde{\Omega}(\sqrt{\tilde{\kappa}/(\mu\eta)})$ inner updates. Then we have
    \begin{equation*}
    \begin{aligned}
        \bbE F(\tbs x_K)-F(\bs x^*)&\leq\frac{5}{3}\exp\left(\frac{K}{8\sqrt{3\alpha}+24\sqrt2\alpha L_\rmeff/T}\right)\bigl(F(\tbs x_0)-F(\bs x^*)\bigr) \\
        &\phrel+\frac{46080\alpha\tr(\bs H^{-1}\bs Q)}{KT}+\frac{122880\alpha\eta\tilde{\kappa}\tr\bs Q}{L_\ell KT}.
    \end{aligned}
    \end{equation*}
\end{theorem}
\begin{remark}
    Suppose Assumptions~\ref{assumption:l-condition}, \ref{assumption:regularity}, and \ref{assumption:fourth-moment} holds, then we have $L_\rmeff\lesssim\alpha\tilde{\kappa}$. Therefore, Themorem~\ref{thm:full-glm} implies Theorem~\ref{thm:glm}.
\end{remark}
\begin{proof}    
    We first apply Lemma~\ref{lemma:auxilary-recursion}. Let
    \begin{equation*}
        a=\frac{1}{2},\quad c=\frac{\alpha\sigma^2}{L_\ell},\quad \frac{1}{\theta_\rmmax}=\min\left\{\sqrt{\frac{1}{12\alpha}},\frac{L_\ell T}{12\sqrt{2}\alpha L_\rmeff}\right\},
    \end{equation*}
    then we have
    \begin{equation*}
        \bbE F(\tbs x_K)-F(\bs x^*)\leq\frac{5}{3}\exp\left(\frac{K}{8\sqrt{3\alpha}+24\sqrt2\alpha L_\rmeff/T}\right)\bigl(\bbE F(\tbs x_K)-F(\bs x^*)\bigr)+\frac{48\alpha\sigma^2}{L_\ell K}.
    \end{equation*}
    Note that by Lemma~\ref{lemma:inner-loop}, we have
    \begin{equation*}
        \sigma^2=\frac{960 L_\ell\tr(\bs H^{-1}\bs Q)+2560\eta\tilde{\kappa}\tr\bs Q}{T}.
    \end{equation*}
    This completes the proof.
\end{proof}

\subsection{Proof of Corollary~\ref{corollary:sample-complexity}}\label{sec:proof-sample-complexity}
\begin{proof}[Proof of Corollary~\ref{corollary:sample-complexity}]
    Note that when
    \begin{equation*}
        n\geq\tilde{\cO}\left( \left(\sqrt{\alpha\kappa\tilde{\kappa}}+\alpha^2\tilde{\kappa}\right)+\frac{\alpha\tr(\bs H^{-1}\bs Q)}{\varepsilon}+\left(\frac{\alpha^2\tilde{\kappa}^2\tr\bs Q}{L_\ell\mu\varepsilon}\right)^{1/3}\right).
    \end{equation*}
    By Theorem~\ref{thm:glm}, the optimization error and statistical error is smaller that $\cO(\varepsilon)$,
    \begin{equation*}
        \bbE F(\tbs x_K)-F(\bs x^*)\leq\cO(\varepsilon)+\frac{\alpha\eta\tilde{\kappa}\tr\bs Q}{L_\ell n}.
    \end{equation*}
    Since we have
    \begin{equation*}
        n\geq\tilde{\Omega}\left(\left(\frac{\alpha^2\tilde{\kappa}^2\tr\bs Q}{L_\ell\mu\varepsilon}\right)^{1/3}\right)\implies\tilde{\Theta}\left(\frac{\alpha\tilde{\kappa}}{\mu n^2}\right)\leq\tilde\Theta\left(\frac{n\varepsilon L_\ell}{\alpha\tilde{\kappa}\bs Q}\right),
    \end{equation*}
    \begin{equation*}
        n\geq\tilde{\Omega}(\sqrt{\alpha\kappa\tilde{\kappa}})\implies\tilde{\Theta}\left(\frac{\alpha\tilde{\kappa}}{\mu n^2}\right)\leq\frac{1}{R^2}.
    \end{equation*}
    Thus, the choice of $\eta$ is feasible. Finally, note that the choice of $\eta\leq\tilde{\Theta}\left(\frac{n\varepsilon L_\ell}{\alpha\tilde{\kappa}\tr\bs Q}\right)$ implies that $\frac{\alpha\eta\tilde{\kappa}\tr\bs Q}{L_\ell n}\lesssim\varepsilon$. This completes the proof.
\end{proof}

\section{Excess Risk of ERM}\label{sec:erm-excess-risk}
In this section, we present the excess risk bound of ERM for the generalized linear prediction by applying the local Rademacher technique. We consider ERM of the following form:
\begin{equation*}
    \hat{\bs x}=\argmin_{\bs x\in\bbR^d}\frac{1}{N}\sum_{i=1}^N\ell(\bs a_i^\top\bs x,b_i),
\end{equation*}
where $\{(\bs a_i,b_i)\}_{i=1}^N$ are  i.i.d. random variables from the distribution $\cD$. We make the following assumptions.
\begin{assumption}\label{assumption:erm}
    We make the following assumptions on the loss function $\ell$ and data distribution $\cD$:
    \begin{enumerate}
        \item The loss function $\ell(z,y)$ is $L_0$-Lipschitz with respect to the first argument, i.e., for all $z_1,z_2\in[-1,1]$,
        \begin{equation*}
            \|\ell(z_1,y)-\ell(z_2,y)\|\leq L_0\| z_1-z_2\|\quad\text{$\cD$-a.s.}
        \end{equation*}
        \item The loss function $\ell(\cdot,\cdot)$ is $\mu_\ell$-strongly convex with respect to the first argument.
        \item The data distribution $\cD$ satisfies that $\bs\Sigma^{-1/2}\bs a$ is $\sigma_a^2$-sub-Gaussian.
    \end{enumerate}
\end{assumption}

Let the function class
\begin{equation*}
    \cF_\delta=\left\{\bs a\mapsto\frac{\bs a^\top\bs x}{AX}\bbI_{\{\|\bs a^\top\bs x\|\leq AX\}}:\frac{1}{AX}\sqrt{\bbE_{a\sim\cD}(\bs a^\top\bs x)^2\bbI_{\{\|\bs a^\top\bs x\|\leq AX\}}}\leq\delta\right\},
\end{equation*}
where we choose 
\begin{equation*}
    A=C\sigma_a d\sqrt{\ln\frac{1}{\delta}},\quad X=2\|\bs x^*\|_{\bs\Sigma},
\end{equation*}
and $C$ is a sufficiently large constant such that $\bbE_{\bs a\sim\cD}\bs a\bs a^\top\bbI_{\|\bs\Sigma^{-1/2}\bs a\|\leq A}\succeq\bs\Sigma/2$, and $\|\bs\Sigma^{-1/2}\bs a_i\|\leq A$ for all $i\in[N]$ with probability at lest $1-N\delta$.

The following theorem bounds the excess risk of ERM.

\begin{theorem}[Excess risk bound of ERM]\label{thm:erm-excess-risk}
    Suppose Assumption~\ref{assumption:erm} holds. Let $N\geq 441 L_0^2d\ln(1/\delta)/\mu_\ell^2$ and $\sigma_0^2=L_0^2d/L_\ell$ Then with probability at least $1-(N+1)\delta$, we have
    \begin{equation*}
        L(\hat{\bs x})-L(\bs x^*)\leq\frac{\alpha\sigma_0^2}{\mu_\ell N}
    \end{equation*}
\end{theorem}

Furthermore, we show that $\sigma_0$ is lower bounded by $\|\bs H^{-1/2}\bs Q\bs H^{-1/2}\| d\geq\tr(\bs H^{-1}\bs Q)$.
\begin{lemma}
    Suppose Assumption~\ref{assumption:erm} holds, then we have $\sigma_0^2\geq\|\bs H^{-1/2}\bs Q\bs H^{-1/2}\| d\geq\tr(\bs H^{-1}\bs Q)$.
\end{lemma}
\begin{proof}
    Since $\|\ell(\bs a^\top\bs x^*)\|\leq L_0$, we have
    \begin{equation*}
        \bs Q=\bbE_{\bs a,\ell\sim\cD}\left(\ell'(\bs a^\top\bs x^*)\right)^2\bs a\bs a^\top\preceq L_0\bs\Sigma=\frac{L_0\bs H}{L_\ell}.
    \end{equation*}
    Therefore, we have $L_0\geq L_\ell\|\bs H^{-1/2}\bs Q\bs H^{-1/2}\|$, which is equivalent to $\sigma_0^2\geq \|\bs H^{-1/2}\bs Q\bs H^{-1/2}\| d\geq\tr(\bs H^{-1}\bs Q)$.
\end{proof}

\subsection{Proof of Theorem~\ref{thm:erm-excess-risk}}

The following lemma bounds the Rademacher complexity of $\cF_\delta$.
\begin{lemma}
    Suppose Assumption~\ref{assumption:erm} holds, then we have
    \begin{equation*}
        \cR_N(\cF_\delta)\leq2\delta\sqrt{\frac{d}{N}}.
    \end{equation*}
\end{lemma}

\begin{proof}
    Note that $\cF_\delta\subset\bar{\cF}_\delta$, where
    \begin{equation*}
        \bar{\cF}_\delta=\left\{\bs a\mapsto\frac{\bs a^\top\bs x}{AX}\bbI_{\{\|\bs a^\top\bs x\|\leq AX\}}:\frac{\|\bs x-\bs x^*\|_{\bs\Sigma}}{2AX}\leq\delta\right\}
    \end{equation*}
    By definition, we have
    \begin{equation*}
    \begin{aligned}
        &\phrel\cR_N(\cF_\delta)\leq\cR_N(\bar{\cF}_\delta) \\
        &=\bbE_{\bs a,\bs\varepsilon}\sup_{\|\bs x-\bs x^*\|_{\bs\Sigma}\leq 2AX\delta}\frac{1}{N}\sum_{i=1}^N\frac{\varepsilon_i\bs a_i^\top\bs x}{AX}\bbI_{\{|\bs a_i^\top\bs x|\leq AX\}} \leq \frac{1}{AX}\bbE_{\bs a,\bs\varepsilon}\sup_{\|\bs h\|_{\bs\Sigma}\leq 2AX\delta}\frac{1}{N}\sum_{i=1}^N \varepsilon_i \bs a_i^\top \bs h\bbI_{\{\|\bs h\|\leq AX\}} \\
        &\leq\frac{\delta}{AX}\sqrt{\bbE_{\bs a,\bs\varepsilon}\left\|\frac{1}{N}\sum_{i=1}^N \varepsilon_i \bs a_i\right\|^2}\leq 2\delta\sqrt{\frac{d}{N}}.
    \end{aligned}
    \end{equation*}
    This completes the proof.
\end{proof}

\begin{proof}[Proof of Theorem~\ref{thm:erm-excess-risk}]
    We consider a constrained ERM as follows:
    \begin{equation*}
        \tilde{\bs x}=\argmin_{\|\bs x-\bs x^*\|_{\bs\Sigma}\leq 1}\frac{1}{N}\sum_{i=1}^N\ell_i(\bs a_i^\top\bs x).
    \end{equation*}
    Let $\tilde{\ell}(z)=\ell(AXz)$, so $\tilde{\ell}$ is $L_0AX$-Lipschitz and $\mu_\ell A^2X^2$-strongly convex, and the objective can be written as:
    \begin{equation*}
        \min_{\|\bs x-\bs x^*\|_{\bs\Sigma}\leq X}\frac{1}{N}\sum_{i=1}^N\tilde{\ell}_i(\bs a_i^\top\bs x/(AX))
    \end{equation*}
    Let $\delta_n=2\sqrt{\frac{d}{N}\ln\frac{1}{\delta}}$, then $\delta_N^2\geq\cR_N(\cF_{\delta_N})$. We apply the localization technique (Theorem~14.20 in \citet{wainwright2019high}) to obtain that
    \begin{equation*}
        \bbE\left(\frac{(\bs a^\top(\tilde{\bs x}-\bs x^*))^2\bbI_{\{\|\bs a^\top\bs x\|\leq AX\}}}{A^2X^2}\middle|\bs a_1,\ldots,\bs a_N\right)\leq\frac{441L_0^2d}{\mu_\ell^2 A^2X^2N}\ln\frac{1}{\delta},
    \end{equation*}
    \begin{equation*}
        L(\tilde{\bs x})-L(\bs x^*)\leq\frac{220 L_0^2 d}{\mu_\ell N}\ln\frac{1}{\delta},
    \end{equation*}
    with probability at least $1-\delta$.  Thus, we have $\|\tilde{\bs x}-\bs x^*\|_{\bs\Sigma}\leq 1$ by $N\geq 441 L_0^2d\ln(1/\delta)/\mu_\ell^2$. Therefore, we have $\hat{\bs x}=\tilde{\bs x}$ by the strong convexity.
\end{proof}

\section{Worst-case Optimality of Statistical Term}\label{sec:statistical-optimality}
In this section, we show that the statistical term in Theorem~\ref{thm:glm} cannot be further improved without additional assumptions on the smoothness of $\ell\sim\cD$.

\paragraph{Data Generating Process.}
We consider the following statistical model:
\begin{equation}
    \bs a\sim\cN(\bs0,\bs\Sigma),\quad b\sim\cN(\bs a^\top\bs x^*,\sigma^2),\label{eq:hard-instance}
\end{equation}
where $\bs x^*$ is the ground truth. We consider the following problem class.
\begin{definition}[Problem Class]\label{def:problem-class}
    The problem class $\cP_n$ is defined as $\cP_n=\{\mu_{\bs x^*}^{\otimes n}:\bs x^*\in\bbR^d\}$, where $\mu_{\bs x^*}$ denotes the distribution in \eqref{eq:hard-instance}, and $\mu_{\bs x^*}^{\otimes n}$ denotes the distribution of $\{(\bs a_i,b_i)\}_{i=1}^n\stackrel{\text{i.i.d.}}{\sim}\mu_{\bs x^*}$.
\end{definition}

\paragraph{Loss Function.} Let $\delta>0$, we define the following function:
\begin{equation*}
    \ell_\delta(z)=\begin{cases}
        \mu_\ell z^2/2+(L_\ell-\mu_\ell)\delta z-(L_\ell-\mu_\ell)\delta^2/2,&z<-\delta. \\
        L_\ell z^2/2,&\| z\|\leq\delta, \\
        \mu_\ell z^2/2+(L_\ell-\mu_\ell)\delta z-(L_\ell-\mu_\ell)\delta^2/2,&z>\delta.
    \end{cases}
\end{equation*}
The population loss function is defined as
\begin{equation*}
    F_\delta(\bs x)=\bbE_{\bs a,b\sim\cD}\ \ell_\delta(\bs a^\top\bs x-b).
\end{equation*}
Let $F_\delta^*=F_\delta(\bs x^*)$
We denote the minimax risk by
\begin{equation}
    \cR_n(\cP_n)=\inf_{\hat{\bs x}}\sup_{\cD_n\in\cP_n}\bbE_{\{(\bs a_i,b_i)\}_{i=1}^n\sim\cD_n}F_\delta\left(\hat{\bs x}(\{(\bs a_i,b_i)\}_{i=1}^n)\right)-F_\delta(\bs x^*(\cD)),\label{eq:minimax}
\end{equation}
where $\inf$ is taken over all estimators $\hat{\bs x}:(\bbR^d\times\bbR)^{n}\to\bbR^d$.

The following theorem establishes the worst-case optimality of the statistical term.
\begin{theorem}\label{thm:minimax-risk}
    Consider the problem class $\cP_n$ defined in Definition~\ref{def:problem-class}. For sufficiently small $\delta>0$, we have the following lower bound of the minimax risk $\cR_n(\cP_n)$ defined in \eqref{eq:minimax}:
    \begin{equation*}
        \cR_n(\cP_n)\geq\frac{\alpha\tr(\bs H^{-1}\bs Q_\delta)}{2n}.
    \end{equation*}
\end{theorem}

\subsection{Proof of Theorem~\ref{thm:minimax-risk}}

The following lemma bounds the minimax risk of $\cP_n$.
\begin{lemma}
    Consider the problem class $\cP_n$, we have the following lower bound of the minimax risk $\cR_n(\cP_n)$ defined in \eqref{eq:minimax}:
    \begin{equation*}
        \cR_n(\cP_n)\geq\frac{\mu_\ell d}{n}+\frac{\mu_\ell\sigma^2}{n}-F_\delta^*
    \end{equation*}
\end{lemma}
\begin{proof}
    We apply Theorem~1 in \citet{mourtada2022exact} to obtain:
    \begin{equation*}
        \inf_{\hat{\bs x}}\sup_{\cD_n\in\cP_n}\bbE_{\{(\bs a_i,b_i)\}_{i=1}^n\sim\cD_n}\left\|\hat{\bs x}(\{(\bs a_i,b_i)\}_{i=1}^n)-\bs x^*\right\|_{\bs\Sigma}^2=\frac{1}{n}\bbE\tr\left(\bs\Sigma\hat{\bs\Sigma}_n^{-1}\right),
    \end{equation*}
    where $\hat{\bs\Sigma}_n=\frac{1}{n}\sum_{i=1}^n\bs a_i\bs a_i^\top$ is the sample covariance matrix. By the operator convexity of the matrix inverse, we have $\bbE\tr(\bs\Sigma\hat{\bs\Sigma}_n^{-1})\geq d$. Finally, note that $\ell_\delta$ is $\mu_\ell$-strongly convex, so we have
    \begin{equation*}
        F_\delta(\bs x)\geq\bbE_{\bs a,b\sim\cD}\frac{\mu_\ell}{2}(\bs a^\top\bs x-\bs a^\top\bs x^*+\varepsilon)^2=\frac{\mu_\ell}{2}\|\bs x-\bs x^*\|_{\bs\Sigma}^2+\frac{\mu_\ell\sigma^2}{2}.
    \end{equation*}
    Therefore, we have
    \begin{equation*}
    \begin{aligned}
        \cR_n(\cP_n)&\geq\inf_{\hat{\bs x}}\sup_{\cD_n\in\cP}\bbE_{\{(\bs a_i,b_i)\}_{i=1}^n\sim\cD_n}\frac{\mu_\ell}{2}\left\|\hat{\bs x}(\{(\bs a_i,b_i)\}_{i=1}^n)-\bs x^*(\cD)\right\|_{\bs\Sigma}^2+\frac{\mu_\ell\sigma^2}{2}-F_\delta^* \\
        &\geq\frac{\mu_\ell d}{n}+\frac{\mu_\ell\sigma^2}{n}-F_\delta^*,   
    \end{aligned}
    \end{equation*}
    This completes the proof.
\end{proof}

Let $\bs H=L_\ell\bs\Sigma$, and $\bs Q_\delta=\bbE((\ell'(\bs a^\top\bs x^*))^2\bs a\bs a^\top)$ denote the second moment of the gradient covariance matrix, then the statistical term in Theorem~\ref{thm:glm} is
\begin{equation*}
    \frac{\alpha\tr(\bs H^{-1}\bs Q_\delta)}{n}.
\end{equation*}

In the following we prove Theorem~\ref{thm:minimax-risk}.
\begin{proof}[Proof of Theorem~\ref{thm:minimax-risk}]
    Let $\bs Q=\lim_{\delta\to 0}\bs Q_\delta=\mu_\ell^2\bs\Sigma$. Note that we have
    \begin{equation*}
        \lim_{\delta\to 0}F_\delta^*=\lim_{\delta\to 0}\bbE\ell_\delta(\varepsilon)=\frac{\mu_\ell\sigma^2}{2},
    \end{equation*}
    \begin{equation*}
        \lim_{\delta\to 0}\bs Q_\delta=\lim_{\delta\to 0}\bbE(\ell'_\delta(\varepsilon))^2\bs a\bs a^\top=\bbE(\mu_\ell\varepsilon)^2\bs a\bs a^\top=\mu_\ell^2\bs\Sigma.
    \end{equation*}
    Combining the above results, we have
    \begin{equation*}
        \lim_{\delta\to 0}\left\{\cR_n(\cP_n)-\frac{\alpha\tr(\bs H^{-1}\bs Q_\delta)}{2n}\right\}=\frac{\tr(\bs\Sigma^{-1}\bs Q)}{\mu_\ell n}-\frac{\alpha\tr(\bs H^{-1}\bs Q)}{2n}=\frac{\alpha\tr(\bs H^{-1}\bs Q)}{2n}>0.
    \end{equation*}
    Thus, there exists $\delta>0$ such that
    \begin{equation*}
        \inf_{\hat{\bs x}}\sup_{\cD\in\cP}F_\delta(\hat{\bs x})-F_\delta(\bs x^*)-\frac{\alpha\tr(\bs H^{-1}\bs Q_\delta)}{2n}>0.
    \end{equation*}
    Rearrange the terms to obtain the desired result.
\end{proof}

\section{Extensions of SADA}\label{sec:sada-extension}
\subsection{SADA for Weakly Convex Objectives}

\begin{algorithm}[t]
    \caption{\textbf{S}tochastic \textbf{A}ccelerated \textbf{D}ata-Dependent \textbf{A}lgorithm for \textbf{W}eakly \textbf{C}onvex Objectives (SADA-WC)}\label{alg:sada-wc}
    \begin{algorithmic}
    \Require {Initialization $\tbs x_0=\tbs x_{-1}$, regularization parameters $\{h_k\}_{k=1}^K$, step sizes $\eta$, $\gamma$, and momentum parameters $\{\beta_k\}_{k=1}^K$, $\theta$, target accuracy $\varepsilon$, radius $D$, $M$}
    \For{$k=1,2,\ldots,K$}
        \State $\tilde{\bs y}_{k-1}\gets\tbs x_{k-1}+\beta_k(\tbs x_{k-1}-\tbs x_{k-2})$ \Comment{Extraplotation}
        \State $\bs x_0\gets\tilde{\bs y}_{k-1}$, $\bs z_0\gets\tilde{\bs y}_{k-1}$
        \For{$t=1,2,\ldots,T$} \Comment{Inner loop for solving subproblem~\eqref{eq:subproblem-wc}}
            \State Sample fresh data $(\bs a_t,b_t)\sim\cD$
            \State $\bs y_{t-1}\gets\frac{1}{1+\theta}\bs x_{t-1}+\frac{\theta}{1+\theta}\bs z_{t-1}$
            \State $\bs{\hat g}_t\gets h_k\ell'(\bs a_t^\top\tbs y_{k-1},b_t)\bs a_t+\dfrac{h_k\varepsilon}{2 M^2}\bs a+\left[\bs a_t^\top(\bs y_{t-1}-\tbs y_{k-1})\right]\bs a_t+\dfrac{\varepsilon}{L_\ell D^2}(\bs y_{t-1}-\tbs y_{k-1})$
            \State $\bs x_t\gets\bs y_{t-1}-\eta\hat{\bs g}_t$
            \State $\bs z_t\gets\theta\bs y_{t-1}+(1-\theta)\bs z_{t-1}-\gamma\hat{\bs g}_t$
        \EndFor
        \State $\tbs x_k\gets\frac{2}{T}\sum_{t=T/2+1}^T\bs x_t$ \Comment{Tail-averaging scheme}
    \EndFor

    \State\Return $\tbs x_K$
    
    \end{algorithmic}
\end{algorithm}

The algorithm SADA-WC is shown in Algorithm~\ref{alg:sada-wc}. The difference is that the inner loop solves the following subproblem:
\begin{equation}    
    \min_{\bs x\in\bbR^d}\;\bbE_{\bs a,b\sim \cD} h_k\left\langle\ell'(\bs a^\top\tbs y_{k-1},b)\bs a+\frac{\varepsilon}{2 M^2}\bs a,\bs x-\tbs y_{k-1}\right\rangle+\frac{1}{2}\|\bs x-\tbs y_{k-1}\|_{\bs \Sigma+\frac{\varepsilon}{L_\ell D^2}\bs I}^2.\label{eq:subproblem-wc}
\end{equation}

The proofs presented in Appendix~\ref{sec:variance-bound} and Appendix~\ref{sec:bias-bound} can be extended to this case with slight modifications to the proof of the noise bound. For the inner loop, we modify the definition of $\cM$ and $\tilde{\cM}$ in \eqref{eq:def-M} to
\begin{equation*}
    \cM_\varepsilon=\bbE\left(\bs a\bs a^\top+\frac{\varepsilon}{L_\ell D^2}\bs I\right)\otimes\left(\bs a\bs a^\top+\frac{\varepsilon}{L_\ell D^2}\bs I\right),\quad\tilde{\cM}_\varepsilon=\left(\bs\Sigma+\frac{\varepsilon}{L_\ell D^2}\bs I\right)\otimes\left(\bs\Sigma+\frac{\varepsilon}{L_\ell D^2}\bs I\right).
\end{equation*}
Note that $\cM_\varepsilon-\tilde{\cM}_\varepsilon=\cM-\tilde{\cM}$, so the bounds in Appendix~\ref{sec:inner-loop} hold by replacing the eigenvalue $\lambda_i$ by $\lambda_i+\frac{\varepsilon}{L_\ell D^2}$, hence the claimed bound.

\subsection{SADA with Unlabeled Data}

\begin{algorithm}[t]
\caption{\textbf{S}tochastic \textbf{A}ccelerated \textbf{D}ata-Dependent \textbf{A}lgorithm with \textbf{U}nlabeled \textbf{D}ata (SADA-UD)}\label{alg:sada-ud}
    \begin{algorithmic}
    \Require {Initialization $\tbs x_0$, regularization parameters $\{h_k\}_{k=1}^K$, step sizes $\eta$, $\gamma$, and momentum parameters $\{\beta_k\}_{k=1}^K$, $\theta$, $\tbs x_{-1}=\tbs x_0$, $T_0=\tilde{\Theta}(\sqrt{\tilde{\kappa}/(\mu\eta)}$}
    \For{$k=1,2,\ldots,K$}
        \State $\tilde{\bs y}_{k-1}\gets\tbs x_{k-1}+\beta_k(\tbs x_{k-1}-\tbs x_{k-2})$ \Comment{Extraplotation}
        \State $\bs x_0\gets\tilde{\bs y}_{k-1}$, $\bs z_0\gets\tilde{\bs y}_{k-1}$
        \For{$t=1,2,\ldots,T$} \Comment{Inner loop for solving subproblem~\eqref{eq:subproblem}}
            \State $\bs y_{t-1}\gets\frac{1}{1+\theta}\bs x_{t-1}+\frac{\theta}{1+\theta}\bs z_{t-1}$
            \State Sample fresh labeled data $(\bs a_t,b_t)\sim\cD$
            \If{$t\leq T_0$}
                \State Sample fresh $m$ unlabeled data $\{\bs a_{t,i}^\text{ul}\}_{i=1}^m\sim\text{i.i.d.}\cD$\Comment{Sample $m$ unlabeled data}
                \State $\bs{\hat g}_t\gets h_k\ell'(\bs a_t^\top\tbs y_{k-1},b_t)\bs a_t+\dfrac{1}{m+1}\left(\bs a_t\bs a_t^\top+\displaystyle\sum_{i=1}^m\bs a_{t,i}^\text{ul}(\bs a_{t,i}^\text{ul})^\top\right)(\bs y_{t-1}-\tbs y_{k-1})$
            \Else            
                \State $\bs{\hat g}_t\gets h_k\ell'(\bs a_t^\top\tbs y_{k-1},b_t)\bs a_t+\left[\bs a_t^\top(\bs y_{t-1}-\tbs y_{k-1})\right] \bs a_t$ 
            \EndIf
            \State $\bs x_t\gets\bs y_{t-1}-\eta\hat{\bs g}_t$
            \State $\bs z_t\gets\theta\bs y_{t-1}+(1-\theta)\bs z_{t-1}-\gamma\hat{\bs g}_t$
        \EndFor
        \State $\tbs x_k\gets\frac{2}{T}\sum_{t=T/2+1}^T\bs x_t$ \Comment{Tail-averaging scheme}
    \EndFor

    \State\Return $\tbs x_K$
    
    \end{algorithmic}
\end{algorithm}

The algorithm SADA-UD is shown in Algorithm~\ref{alg:sada-ud}. The difference is that for each inner iteration, we sample data pair $\bs a,b$ from $\cD$ and additional $m$ independent unlabeled data $\{\bs a_i^{\text{ul}}\}_{i=1}^m$ from the same distribution $\cD$ without label. Then, we use $\frac{1}{m+1}\left(\bs a\bs a^\top+\sum_{i=1}^m \bs a_i^{\text{ul}}\bs (\bs a_i^{\text{ul}})^\top\right)$ as the estimator of $\bs\Sigma$.

We present a sketch of the analysis of Algorithm~\ref{alg:sada-ud}. Let
\begin{equation*}
    \hbs\Sigma=\frac{1}{m+1}\left(\bs a\bs a^\top+\sum_{i=1}^m \bs a_i^{\text{ul}}\bs (\bs a_i^{\text{ul}})^\top\right).
\end{equation*}
We extend Assumption~\ref{assumption:fourth-moment} to the following lemma.
\begin{lemma}\label{lemma:fourth-moment-extension}
    Suppose Assumption~\ref{assumption:fourth-moment} holds, then we have
    \begin{equation*}
        \bbE\hbs\Sigma^2\preceq R_m^2\bs\Sigma,\quad\bbE\hbs\Sigma\bs\Sigma^{-1}\hbs\Sigma\preceq\tilde{\kappa}_m\bs\Sigma,
    \end{equation*}
    where $R_m^2=\frac{R^2+m\lambda_\rmmax(\bs\Sigma)}{m+1}$ and $\tilde{\kappa}_m=\frac{\tilde{\kappa}+m}{m+1}$.
\end{lemma}
\begin{remark}
    If we set $\hbs\Sigma=\bs a\bs a^\top$, then Lemma~\ref{lemma:fourth-moment-extension} is identical to Assumption~\ref{assumption:fourth-moment}.
\end{remark}
\begin{proof}
    By definition, we have
    \begin{equation*}
    \begin{aligned}
        \bbE\hbs\Sigma^2&=\frac{1}{(m+1)^2}\bbE\left(\bs a\bs a^\top+\sum_{i=1}^m \bs a_i^{\text{ul}}\bs (\bs a_i^{\text{ul}})^\top\right)\left(\bs a\bs a^\top+\sum_{i=1}^m \bs a_i^{\text{ul}}\bs (\bs a_i^{\text{ul}})^\top\right) \\
        &\stackrel{a}{=}\frac{1}{m+1}\bbE\|\bs a\|^2\bs a\bs a^\top+\frac{m}{m+1}\bs\Sigma^2 \\
        &\stackrel{b}{\preceq}\frac{R^2+m\lambda_\rmmax(\bs\Sigma)}{m+1}\bs\Sigma=R_m^2\bs\Sigma,
    \end{aligned}
    \end{equation*}
    where for $\stackrel{a}{=}$ we use $\bs a$ and $\bs a_i^{\text{ul}}$ are i.i.d. random variables, and $\stackrel{b}{\preceq}$ uses Assumption~\ref{assumption:fourth-moment}. For the second inequality, we have
    \begin{equation*}
    \begin{aligned}
        \bbE\hbs\Sigma\bs\Sigma^{-1}\hbs\Sigma&=\frac{1}{(m+1)^2}\bbE\left(\bs a\bs a^\top+\sum_{i=1}^m \bs a_i^{\text{ul}}\bs (\bs a_i^{\text{ul}})^\top\right)\bs\Sigma^{-1}\left(\bs a\bs a^\top+\sum_{i=1}^m \bs a_i^{\text{ul}}\bs (\bs a_i^{\text{ul}})^\top\right) \\
        &\stackrel{a}{=}\frac{1}{m+1}\bbE\|\bs a\|_{\bs\Sigma^{-1}}^2\bs a\bs a^\top+\frac{m}{m+1}\bs\Sigma \\
        &\stackrel{b}{\preceq}\frac{\tilde{\kappa}+m}{m+1}\bs\Sigma=\tilde{\kappa}_m^2\bs\Sigma,
    \end{aligned}
    \end{equation*}
    where for $\stackrel{a}{=}$ we use $\bs a$ and $\bs a_i^{\text{ul}}$ are i.i.d. random variables, and $\stackrel{b}{\preceq}$ uses Assumption~\ref{assumption:fourth-moment}.
\end{proof}

For the inner loop, the stochastic gradient at $\bs x$ is
\begin{equation*}
    \ell'(\bs a^\top\tbs y,b)\bs a+\hbs\Sigma(\bs x-\tbs y),
\end{equation*}
so the noise covariance $\bs Q$ is defined as
\begin{equation*}
    \bs R=\bbE\left(\ell'(\bs a^\top\tbs y,b)\bs a-\hbs\Sigma\bs\Sigma^{-1}\nabla F(\tbs y)\right)\left(\ell'(\bs a^\top\tbs y,b)\bs a+\hbs\Sigma\bs\Sigma^{-1}\nabla F(\tbs y)\right)^\top.
\end{equation*}
Follow the proof of Lemma~\ref{lemma:noise-bound-Sigma-inverse} and Lemma~\ref{lemma:noise-bound-2-norm}, one has
\begin{align*}
    \tr(\bs\Sigma^{-1}\bs R)&\leq 5\tr(\bs\Sigma^{-1}\bs Q)+5L_\ell(L+\tilde{\kappa}_m)(F(\tbs y)-F(\bs x^*)), \\
    \tr\bs R&\leq 5\tr\bs Q+5L_\ell(B+R_m^2)(F(\tbs y)-F(\bs x^*)).
\end{align*}

The proofs can be directly extended to this setting by replacing Assumption~\ref{assumption:fourth-moment} by Lemma~\ref{lemma:fourth-moment-extension}, and accordingly the quanitiy $R^2$, $\kappa$ and $\tilde{\kappa}$ will be replaced by $R_m^2$, $\kappa_m$ and $\tilde{\kappa}_m$, yielding the claimed result. Finally, note that $L\leq\alpha\tilde{\kappa}$, we derive the results in Section~\ref{sec:extensions-unlabeled-data}.

\end{document}